\documentclass[twoside,11pt]{article}
\usepackage{jmlr2e}

\usepackage{epsf}
\usepackage{graphicx}
\usepackage{amsmath}
\usepackage{amssymb}
\usepackage{algorithmic}
\usepackage{algorithm}
\usepackage{natbib}
\usepackage{hyperref}

 \usepackage{algolyx}
 \keycomment{\#<}{>\#}
\makeatother

\def\Re{\mathbb{R}}
\def\Nat{{\rm I\kern\pIR N}}
\def\argmax{\mathop{\rm arg\,max}}
\def\argmin{\mathop{\rm arg\,min}}

\def\log{\mathop{\rm log}}

\def\Var{{\bf Var}}
\def\Cov{{\bf Cov}}
\def\exptE{{\bf E}}
\def\diag{\mathop{\rm diag}}

\def\A{{\mathcal{A}}}

\def\D{{\mathcal{D}}}
\def\E{{\mathcal{E}}}

\def\N{{\mathcal{N}}}
\def\O{{\mathcal{O}}}
\def\P{{\mathcal{P}}}

\def\X{{\mathcal{X}}}

\def\Z{{\mathcal{Z}}}

\def\vecalpha{{\boldsymbol{\alpha}}}
\def\vecbeta{{\boldsymbol{\beta}}}
\def\vectheta{{\boldsymbol{\theta}}}
\def\vecphi{{\boldsymbol{\phi}}}

\def\vecpsi{{\boldsymbol{\psi}}}

\def\vec0{{\boldsymbol{0}}}
\def\veca{{\boldsymbol{a}}}
\def\vecb{{\boldsymbol{b}}}
\def\vecf{{\boldsymbol{f}}}
\def\bvecf{{\boldsymbol{\bar{f}}}}
\def\vecg{{\boldsymbol{g}}}
\def\veck{{\boldsymbol{k}}}

\def\vecr{{\boldsymbol{r}}}

\def\vecu{{\boldsymbol{u}}}
\def\vecv{{\boldsymbol{v}}}
\def\vecw{{\boldsymbol{w}}}
\def\vecx{{\boldsymbol{x}}}
\def\vecy{{\boldsymbol{y}}}
\def\vecz{{\boldsymbol{z}}}

\def\matSigma{{\boldsymbol{\Sigma}}}
\def\matLambda{{\boldsymbol{\Lambda}}}
\def\matA{{\boldsymbol{A}}}
\def\matB{{\boldsymbol{B}}}
\def\matC{{\boldsymbol{C}}}

\def\matG{{\boldsymbol{G}}}
\def\matH{{\boldsymbol{H}}}
\def\matI{{\boldsymbol{I}}}
\def\matK{{\boldsymbol{K}}}

\def\matU{{\boldsymbol{U}}}
\def\matV{{\boldsymbol{V}}}
\def\matY{{\boldsymbol{Y}}}

\newtheorem{assumption}{Assumption}

\jmlrheading{17}{2016}{1-53}{9/10; Revised 4/16}{8/16}{Mohammad Ghavamzadeh, Yaakov Engel, and Michal Valko}

\ShortHeadings{Bayesian Policy Gradient and Actor-Critic Algorithms}{Ghavamzadeh, Engel, and Valko}
\firstpageno{1}

\begin{document}

\title{Bayesian Policy Gradient and Actor-Critic Algorithms}

\author{\name Mohammad Ghavamzadeh\thanks{Mohammad Ghavamzadeh is at Adobe Research, on leave of absence from Inria Lille - Team SequeL.} \email mohammad.ghavamzadeh@inria.fr \\
\addr Adobe Research \& Inria \\
\name Yaakov Engel \email yakiengel@gmail.com \\
\addr Rafael Advanced Defence System, Israel \\
\name Michal Valko \email michal.valko@inria.fr \\
\addr Inria Lille ---  SequeL team, France}

\editor{Jan Peters}

\maketitle


\begin{abstract} 
\vspace{-0.1in}

Policy gradient methods are reinforcement learning algorithms that adapt a parameterized policy by following a performance gradient estimate. Many conventional policy gradient methods use Monte-Carlo techniques to estimate this gradient. The policy is improved by adjusting the parameters in the direction of the gradient estimate. Since Monte-Carlo methods tend to have high variance, a large number of samples is required to attain accurate estimates, resulting in slow convergence. In this paper, we first propose a Bayesian framework for policy gradient, based on modeling the policy gradient as a Gaussian process. This reduces the number of samples needed to obtain accurate gradient estimates. Moreover, estimates of the natural gradient as well as a measure of the uncertainty in the gradient estimates, namely, the gradient covariance, are provided at little extra cost. Since the proposed Bayesian framework considers system trajectories as its basic observable unit, it does not require the dynamics within trajectories to be of any particular form, and thus, can be easily extended to partially observable problems. On the downside, it cannot take advantage of the Markov property when the system is Markovian. 

To address this issue, we proceed to supplement our Bayesian policy gradient framework with a new actor-critic learning model in which a Bayesian class of non-parametric critics, based on Gaussian process temporal difference learning, is used. Such critics model the action-value function as a Gaussian process, allowing Bayes' rule to be used in computing the posterior distribution over action-value functions, conditioned on the observed data. Appropriate choices of the policy parameterization and of the prior covariance (kernel) between action-values allow us to obtain closed-form expressions for the posterior distribution of the gradient of the expected return with respect to the policy parameters. 
We perform detailed experimental comparisons of the proposed Bayesian policy gradient and actor-critic algorithms with classic Monte-Carlo based policy gradient methods, as well as with each other, on a number of reinforcement learning problems.
\end{abstract}

\begin{keywords}
reinforcement learning, policy gradient methods, actor-critic algorithms, Bayesian inference, Gaussian processes 
\end{keywords}


\section{Introduction}
\label{sec:intro}

Policy gradient (PG) methods\footnote{The term has been coined in~\citet{Sutton00PG}, but here we use it more liberally to refer to a whole class of reinforcement learning algorithms.} are reinforcement learning (RL) algorithms that maintain a parameterized action-selection policy and update the policy parameters by moving them in the direction of an estimate of the gradient of a performance measure. Early examples of PG algorithms are the class of REINFORCE algorithms~\citep{Williams92SS},\footnote{Note that policy gradient methods have been studied in the control community (see e.g.,~\citealt{Dyer70CT,Jacobson70DD,Hasdorff76GO}) before REINFORCE. However, unlike REINFORCE that is model-free, they were all based on the exact model of the system (model-based).} which are suitable for solving problems in which the goal is to optimize the average reward. Subsequent work (e.g.,~\citealp{Kimura95RL,Marbach98SM,Baxter01IP}) extended these algorithms to the cases of infinite-horizon Markov decision processes (MDPs) and partially observable MDPs (POMDPs), while also providing much needed theoretical analysis.\footnote{It is important to note that the pioneering work of Gullapali and colleagues in the early 1990s~\citep{Gullapalli90SR,Gullapalli92LC,Gullapalli94AR} in applying policy gradient methods to robot learning problems had an important role in popularizing this class of algorithms. In fact policy gradient methods have been continuously proven to be one of the most effective class of algorithms in learning in robots.} However, both the theoretical results and empirical evaluations have highlighted a major shortcoming of these algorithms, namely, the high variance of the gradient estimates. This problem may be traced to the fact that in most interesting cases, the time-average of the observed rewards is a high-variance (although unbiased) estimator of the true average reward, resulting in the sample-inefficiency of these algorithms.

One solution proposed for this problem was to use an artificial {\em discount factor} in these algorithms~\citep{Marbach98SM,Baxter01IP}, however, this creates another problem by introducing bias into the gradient estimates. Another solution, which does not involve biasing the gradient estimate, is to subtract a {\em reinforcement baseline} from the average reward estimate in the updates of PG algorithms (e.g.,~\citealp{Williams92SS,Marbach98SM,Sutton00PG}). In~\citet{Williams92SS} an average reward baseline was used, and in~\citet{Sutton00PG} it was conjectured that an approximate value function would be a good choice for a state-dependent baseline. However, it was shown in~\citet{Weaver01OR} and~\citet{Greensmith04VR}, perhaps surprisingly, that the mean reward is in general {\em not} the optimal constant baseline, and that the true value function is generally {\em not} the optimal state-dependent baseline.
   
A different approach for speeding-up PG algorithms was proposed by~\citet{Kakade02NP} and refined and extended by~\citet{Bagnell03CP} and~\citet{Peters03RL}. The idea is to replace the policy gradient estimate with an estimate of the so-called {\em natural} policy gradient. This is motivated by the requirement that the policy updates should be invariant to bijective transformations of the parametrization. Put more simply, a change in the way the policy is parametrized should not influence the result of the policy update. In terms of the policy update rule, the move to natural-gradient amounts to linearly transforming the gradient using the inverse Fisher information matrix of the policy. In empirical evaluations, natural PG has been shown to significantly outperform conventional PG (e.g.,~\citealt{Kakade02NP,Bagnell03CP,Peters03RL,Peters08RL}).

Another approach for reducing the variance of policy gradient estimates, and as a result making the search in the policy-space more efficient and reliable, is to use an explicit representation for the value function of the policy. This class of PG algorithms are called actor-critic algorithms. Actor-critic (AC) algorithms comprise a family of RL methods that maintain two distinct algorithmic components: An {\em actor}, whose role is to maintain and update an action-selection policy; and a {\em critic}, whose role is to estimate the value function associated with the actor's policy. Thus, the critic addresses the problem of {\em prediction}, whereas the actor is concerned with {\em control}. Actor-critic methods were among the earliest to be investigated in RL~\citep{Barto83NE,Sutton84TC}. They were largely supplanted in the 1990's by methods, such as SARSA~\citep{Rummery94OQ}, that estimate action-value functions and use them directly to select actions without maintaining an explicit representation of the policy. This approach was appealing because of its simplicity, but when combined with function approximation was found to be unreliable, often failing to converge. These problems led to renewed interest in PG methods. 

Actor-critic algorithms (e.g.,~\citealt{Sutton00PG,Konda00AA,Peters05NA,Bhatnagar07IN}) borrow elements from these two families of RL algorithms. Like value-function based methods, a critic maintains a value function estimate, while an actor maintains a separately parameterized stochastic action-selection policy, as in policy based methods. While the role of the actor is to select actions, the role of the critic is to evaluate the performance of the actor. This evaluation is used to provide the actor with a feedback signal that allows it to improve its performance. The actor typically updates its policy along an estimate of the gradient (or natural gradient) of some measure of performance with respect to the policy parameters. When the representations used for the actor and the critic are {\em compatible}, in the sense explained in~\citet{Sutton00PG} and~\citet{Konda00AA}, the resulting AC algorithm is simple, elegant, and provably convergent (under appropriate conditions) to a local maximum of the performance measure used by the critic, plus a measure of the temporal difference (TD) error inherent in the function approximation scheme~\citep{Konda00AA,Bhatnagar09NA}. 

Existing AC algorithms are based on parametric critics that are updated to optimize frequentist fitness criteria. By ``frequentist'' we mean algorithms that return a point estimate of the value function, rather than a complete posterior distribution computed using Bayes' rule. A Bayesian class of critics based on Gaussian processes (GPs) has been proposed by~\citet{Engel03BM,Engel05RL}, called Gaussian process temporal difference (GPTD). By their Bayesian nature, these algorithms return a full posterior distribution over value functions. Moreover, while these algorithms may be used to learn a parametric representation for the posterior, they are generally capable of searching for value functions in an infinite-dimensional Hilbert space of functions, resulting in a non-parametric posterior. 

Both conventional and natural policy gradient and actor-critic methods rely on Monte-Carlo (MC) techniques in estimating the gradient of the performance measure. MC estimation is a frequentist procedure, and as such violates the {\em likelihood principle}~\citep{Berger84LP}.\footnote{The likelihood principle states that in a parametric statistical model, all the information about a data sample that is required for inferring the model parameters is contained in the likelihood function of that sample.} Moreover, although MC estimates are unbiased, they tend to suffer from high variance, or alternatively, require excessive sample sizes (see~\citealp{Ohagan87MC} for a discussion). In the case of policy gradient estimation this is exacerbated by the fact that consistent policy improvement requires multiple gradient estimation steps.

In~\citet{Ohagan91BQ} a Bayesian alternative to MC estimation is proposed.\footnote{\citet{Ohagan91BQ} mentions that this approach may be traced even as far back as~\citet{Poincare1896CP}.} The idea is to model integrals of the form $\int f(x) g(x) dx$ as GPs. This is done by treating the first term $f$ in the integrand as a {\em random function}, the randomness of which reflects our subjective uncertainty concerning its true identity. This allows us to incorporate our prior knowledge of $f$ into its prior distribution. Observing (possibly noisy) samples of $f$ at a set of points $\{x_1,\ldots,x_M\}$ allows us to employ Bayes' rule to compute a posterior distribution of $f$ conditioned on these samples. This, in turn, induces a posterior distribution over the value of the integral.  

In this paper, we first propose a Bayesian framework for policy gradient estimation by modeling the gradient as a GP. Our Bayesian policy gradient (BPG) algorithms use GPs to define a prior distribution over the gradient of the expected return, and compute its posterior conditioned on the observed data. This reduces the number of samples needed to obtain accurate gradient estimates. Moreover, estimates of the natural gradient as well as a measure of the uncertainty in the gradient estimates, namely the gradient covariance, are provided at little extra cost. Additional gains may be attained by learning a transition model of the environment, allowing knowledge transfer between policies. Since our BPG models and algorithms consider complete system trajectories as their basic observable unit, they do not require the dynamics within each trajectory to be of any special form. In particular, it is not necessary for the dynamics to have the Markov property, allowing the resulting algorithms to handle partially observable MDPs, Markov games, and other non-Markovian systems. On the downside, BPG algorithms cannot take advantage of the Markov property when this property is satisfied. To address this issue, we supplement our BPG framework with actor-critic methods and propose an AC algorithm that incorporates GPTD as its critic. However, rather than merely plugging-in our critic into an existing AC algorithm, we show how the posterior moments returned by the GPTD critic allow us to obtain closed-form expressions for the posterior moments of the policy gradient. This is made possible by utilizing the Fisher kernel~\citep{ShaweTaylor04KM} as our prior covariance kernel for the GPTD state-action {\em advantage} values. Unlike the BPG methods, the Bayesian actor-critic (BAC) algorithm takes advantage of the Markov property of the system trajectories and uses individual state-action-reward transitions as its basic observable unit. This helps reduce variance in the gradient estimates, resulting in steeper learning curves when compared to BPG and the classic MC approach.  

It is important to note that a short version of the two main parts of this paper, {\em Bayesian policy gradient} and {\em Bayesian actor-critic}, appeared in~\citet{Ghavamzadeh06BP} and~\citet{Ghavamzadeh07BA}, respectively. This paper extends these conference papers in the following ways:

\begin{itemize}
\item We have included a discussion on using Bayesian Quadrature (BQ) for estimating vector-valued integrals to the paper. This is totally relevant to this work because the gradient of a policy (the quantity that we are interested in estimating using BQ) is a vector-valued integral when the size of the policy parameter vector is more than 1, which is usually the case. This also helps to better see the difference between the two models we propose for BPG. In Model 1, we place a vector-valued Gaussian process (GP) over a component of the gradient integrant, while in Model 2, we put a scalar-valued GP over a different component of the gradient integrant.
\item We describe the BPG and BAC algorithms in more details and show the details of using online sparsification in these algorithms. Moreover, we show how BPG can be extended to partially observable Markov decision processes (POMDPs) along the same lines that the standard PG algorithms can be used in such problems.
\item In comparison to~\citet{Ghavamzadeh06BP}, we report more details of the experiments and more experimental results, especially in using the posterior variance (covariance) of the gradient to select the step size for updating the policy parameters.
\item We include all the proofs in this paper (almost none was reported in the two conference papers), in particular, the proofs of Propositions~\ref{prop:3},~\ref{prop:4},~\ref{prop:5}, and~\ref{prop:UV}. These proofs are important and the proof techniques are novel and definitely useful for the community. The importance of these proofs come from the fact that they show how with the right choice of GP prior (the one that uses the family of Fisher information kernels), we are able to use BQ and have a Bayesian estimate of the gradient integral, while initially everything indicates that BQ cannot be used for the estimation of this integral.
\item We apply the BAC algorithm to two new domains: ``Mountain Car", a 2-dimensional continuous state and 1-dimensional discrete action problem, and ``Ship Steering", a 4-dimensional continuous state and 1-dimensional continuous action problem.
\end{itemize}


\section{Reinforcement Learning, Policy Gradient, and Actor-Critic Methods}
\label{sec:RL}

Reinforcement learning (RL)~\citep{Bertsekas96NP,Sutton98IR} is term describing a class of learning problems in which an agent (or controller) interacts with a dynamic, stochastic, and incompletely known environment (or plant), with the goal of finding an action-selection strategy, or {\em policy}, to optimize some measure of its long-term performance. This interaction is conventionally modeled as a Markov decision process (MDP)~\citep{Puterman94MD}, or if the environmental state is not completely observable, as a partially observable MDP (POMDP)~\citep{Astrom65OC,Smallwood73OC,Kaelbling98PA}. In this work we restrict our attention to the discrete-time MDP setting.

Let $\P(\X)$, $\P(\A)$, and $\P(\Re)$ denote the set of probability distributions on (Borel) subsets of the sets $\X$, $\A$, and $\Re$, respectively. A MDP is a tuple $\left( \X,\A,q,P,P_0 \right)$ where $\X$ and $\A$ are the state and action spaces; $q(\cdot|x,a)\in\P(\Re)$ and $P(\cdot|x,a)\in\P(\X)$ are the probability distributions over rewards and next states when action $a$ is taken at state $x$ (we assume that $q$ and $P$ are stationary); and $P_0(\cdot)\in\P(\X)$ is the probability distribution according to which the initial state is selected. We denote the random variable distributed according to $q(\cdot|x,a)$ as $r(x,a)$ and its mean as $\bar{r}(x,a)$. 

In addition, we need to specify the rule according to which the agent selects an action at each possible state. We assume that this rule is {\em stationary}, i.e.,~does not depend explicitly on time. A stationary policy $\mu(\cdot|x) \in \P(\A)$ is a probability distribution over actions, conditioned on the current state. A MDP controlled by a policy $\mu$ induces a Markov chain over state-action pairs $\vecz_t =(x_t,a_t)\in\Z=\X\times\A $, with a transition probability density 
$P^\mu(\vecz_t|\vecz_{t-1})=P(x_t|x_{t-1},a_{t-1})\mu(a_t|x_t)$, and an initial state density $P_0^\mu(\vecz_0)=P_0(x_0)\mu(a_0|x_0)$. We generically denote by 
$\xi=(\vecz_0,\vecz_1,\ldots,\vecz_T)\in\Xi,\;T\in\{0,1,\ldots, \infty\}$ a path generated by this Markov chain. Note that $\Xi$ is the set of all possible trajectories that can be generated by the Markov chain induced by the current policy $\mu$. The probability (density) of such a path is given by
\begin{equation}
\Pr(\xi;\mu)=P^\mu_0(\vecz_0)\prod_{t=1}^{T}P^\mu(\vecz_t|\vecz_{t-1})=P_0(x_0)\prod_{t=0}^{T-1}\mu(a_t|x_t) P(x_{t+1}|x_t,a_t).
\label{eq:prob_path}
\end{equation}
We denote by $R(\xi)=\sum_{t=0}^{T-1}\gamma^tr(x_t,a_t)$ the (possibly discounted, $\gamma\in [0,1]$) {\em cumulative return} of the path $\xi$. $R(\xi)$ is a random variable both because the path $\xi$ itself is a random variable, and because, even for a given path, each of the rewards sampled in it may be stochastic. The expected value of $R(\xi)$ for a given path $\xi$ is denoted by $\bar{R}(\xi)$. Finally, we define the {\em expected return} of policy $\mu$ as
\begin{equation}
\eta(\mu)= \exptE\big[R(\xi)\big] = \int_\Xi\bar{R}(\xi)\Pr(\xi;\mu)d\xi.
\label{exp-ret}
\end{equation} 
The $t$-step state-action occupancy density of policy $\mu$ is given by
\begin{equation*}
{P_t^\mu}(\vecz_t)=\int_{\Z^t} d\vecz_0 \ldots d\vecz_{t-1} P_0^\mu(\vecz_0) \prod_{i=1}^{t}
P^\mu(\vecz_{i}|\vecz_{i-1}).
\end{equation*}
It can be shown that under certain regularity conditions~\citep{Sutton00PG}, the expected return of policy $\mu$ may be written in terms of state-action pairs (rather than in terms of trajectories as in Equation~\ref{exp-ret}) as 
\begin{equation}
\eta(\mu) = \int_{\Z} d\vecz \pi^\mu(\vecz) \bar{r}(\vecz),
\label{eq:exp-ret2}
\end{equation}
where $\pi^\mu(\vecz)=\sum_{t=0}^\infty \gamma^t {P_t^\mu}(\vecz)$ is a discounted weighting of state-action pairs encountered while following policy $\mu$. Integrating $a$ out of $\pi^\mu(\vecz)=\pi^\mu(x,a)$ results in the corresponding discounted weighting of states encountered by following policy $\mu$: $\nu^\mu(x)=\int_{\A}da\pi^\mu(x,a)$. Unlike $\nu^\mu$ and $\pi^\mu$, $(1-\gamma)\nu^\mu$ and $(1-\gamma)\pi^\mu$ are distributions. They are analogous to the stationary distributions over states and state-action pairs of policy $\mu$ in the undiscounted setting, respectively, since as $\gamma\rightarrow 1$, they tend to these distributions, if they exist. 

Our aim is to find a policy $\mu^*$ that maximizes the expected return, i.e., 
\begin{small}$\mu^*=\argmax_\mu\eta(\mu)$\end{small}. A policy $\mu$ is assessed according to the expected cumulative rewards associated with states $x$ or state-action pairs $\vecz$. For all states 
$x\in\X$ and actions $a\in\A$, the action-value function and the value function of policy $\mu$ are defined as
\begin{equation*}
Q^\mu(\vecz)=\exptE\left[\sum_{t=0}^\infty\gamma^tr(\vecz_t)|\vecz_0=\vecz\right]\quad\quad\text{and}\quad\quad V^\mu(x)=\int_\A da\mu(a|x)Q^\mu(x,a).
\end{equation*}

In policy gradient (PG) methods, we define a class of smoothly parameterized stochastic policies $\big\{\mu(\cdot|x;\vectheta), x \in \X, \vectheta \in \Theta\big\}$. We estimate the gradient of the expected return, defined by Equation~\ref{exp-ret} (or Equation~\ref{eq:exp-ret2}), with respect to the policy parameters $\vectheta$, from the observed system trajectories. We then improve the policy by adjusting the parameters in the direction of the gradient (e.g.,~\citealt{Williams92SS,Marbach98SM,Baxter01IP}). Since in this setting a policy $\mu$ is represented by its parameters $\vectheta$, policy dependent functions such as $\eta(\mu)$, $\Pr\big(\xi;\mu)$, $\pi^\mu(\vecz)$, $\nu^\mu(x)$, $V^\mu(x)$, and $Q^\mu(\vecz)$ may be written as $\eta(\vectheta)$, $\Pr(\xi;\vectheta)$, $\pi(\vecz;\vectheta)$, $\nu(x;\vectheta)$, $V(x;\vectheta)$, and $Q(\vecz;\vectheta)$, respectively. We assume
\begin{assumption}[Differentiability] \label{ass:1}
For any state-action pair $(x,a)$ and any policy parameter $\vectheta\in\Theta$, the policy $\mu(a|x;\vectheta)$ is continuously differentiable in the parameters $\vectheta$.
\end{assumption}
The {\em score function} or {\em likelihood ratio} method has become the most prominent technique for gradient estimation from simulation. It has been first proposed in the 1960's~\citep{Aleksandrov68SO,Rubinstein69SP} for computing performance gradients in i.i.d.~(independently and identically distributed) processes, and was then extended to {\em regenerative} processes including MDPs by~\citet{Glynn86SA,Glynn90LR},~\citet{Reiman86SA,Reiman89SA},~\citet{Glynn95LR}, and to episodic MDPs by~\citet{Williams92SS}. This method estimates the gradient of the expected return with respect to the policy parameters $\vectheta$, defined by Equation~\ref{exp-ret}, 
%
%
using the following equation:\footnote{Throughout the paper, we use the notation $\nabla$ to denote $\nabla_{\vectheta}$ -- the gradient w.r.t.~the policy parameters.}
\begin{equation}
\label{eq:grad}
\nabla\eta(\vectheta)=\int\bar{R}(\xi)\frac{\nabla \Pr(\xi;\vectheta)}{\Pr(\xi;\vectheta)}\Pr(\xi;\vectheta)d\xi.
\end{equation}
%
In Equation~\ref{eq:grad}, the quantity 
$\frac{\nabla \Pr(\xi;\vectheta)}{\Pr(\xi;\vectheta)}=\nabla\log \Pr(\xi;\vectheta)$ is called the (Fisher) score function or likelihood ratio. Since the initial-state distribution $P_0$ and the state-transition distribution $P$ are independent of the policy parameters $\vectheta$, we may write the score function for a path $\xi$ using Equation~\ref{eq:prob_path} as\footnote{To simplify notation, we omit $\vecu$'s dependence on the policy parameters $\vectheta$, and denote $\vecu(\xi;\vectheta)$ as $\vecu(\xi)$ in the sequel.}
\begin{equation}
\label{eq:score}
\vecu(\xi;\vectheta)=\nabla\log \Pr(\xi;\vectheta)=\frac{\nabla \Pr(\xi;\vectheta)}{\Pr(\xi;\vectheta)}=\sum_{t=0}^{T-1}\frac{\nabla\mu(a_t|x_t;\vectheta)}{\mu(a_t|x_t;\vectheta)}=\sum_{t=0}^{T-1}\nabla\log\mu(a_t|x_t;\vectheta).
\end{equation}

Previous work on policy gradient used classical MC to estimate the gradient in Equation~\ref{eq:grad}. These methods generate i.i.d.~sample paths $\xi_1,\ldots,\xi_M$ according to $\Pr(\xi;\vectheta)$, and estimate the gradient $\nabla\eta(\vectheta)$ using the MC estimator
\begin{equation}
\label{eq:grad_MC}
\widehat{\nabla\eta}(\vectheta)=\frac{1}{M}\sum_{i=1}^MR(\xi_i)\nabla\log \Pr(\xi_i;\vectheta)=\frac{1}{M}\sum_{i=1}^MR(\xi_i)\sum_{t=0}^{T_i-1}\nabla\log\mu(a_{t,i}|x_{t,i};\vectheta).
\end{equation}
This is an unbiased estimate, and therefore, by the law of large numbers, 
$\widehat{\nabla\eta}(\vectheta)\rightarrow\nabla\eta(\vectheta)$ as $M$ goes to infinity, with probability one. 

The policy gradient theorem (\citealp[Proposition~1]{Marbach98SM}; \citealp[Theorem~1]{Sutton00PG}; \citealp[Theorem~1]{Konda00AA}) states that the gradient of the expected return, defined by Equation~\ref{eq:exp-ret2}, for parameterized policies satisfying Assumption~\ref{ass:1} is given by
\begin{equation}
\nabla\eta(\vectheta)=\int dxda\;\nu(x;\vectheta)\nabla\mu(a|x;\vectheta)Q(x,a;\vectheta).
\label{eq:policy_grad}
\end{equation}
Observe that if $b:\X\rightarrow\Re$ is an arbitrary function of $x$ (also called a {\em baseline}), then

\vspace{-0.1in}
\begin{small}
\begin{equation*}
\int_\Z dxda\;\nu(x;\vectheta)\nabla\mu(a|x;\vectheta)b(x)=\int_\X dx\;\nu(x;\vectheta)b(x)\nabla\Big(\int_\A da\mu(a|x;\vectheta)\Big)=\int_\X dx\;\nu(x;\vectheta)b(x)\nabla(1)=0,
\end{equation*}
\end{small}
\vspace{-0.1in}

\noindent
and thus, for any baseline $b(x)$, the gradient of the expected return can be written as

\vspace{-0.1in}
\begin{small}
\begin{equation}
\nabla\eta(\vectheta)=\int dxda\;\nu(x;\vectheta)\nabla\mu(a|x;\vectheta)\big(Q(x,a;\vectheta)+b(x)\big)= \int_\Z d\vecz \pi(\vecz;\vectheta)\nabla\log\mu(a|x;\vectheta)\big(Q(\vecz;\vectheta)+b(x)\big).
\label{eq:policy_grad_baseline}
\end{equation}
\end{small}
\vspace{-0.1in}

\noindent
The baseline may be chosen in such a way so as to minimize the variance of the gradient estimates~\citep{Greensmith04VR}. 

Now consider the actor-critic (AC) framework in which the action-value function for a fixed policy $\mu$, $Q^\mu$, is approximated by a learned function approximator. If the approximation is sufficiently good, we may hope to use it in place of $Q^\mu$ in Equations~\ref{eq:policy_grad} and~\ref{eq:policy_grad_baseline}, and still point roughly in the direction of the true gradient.~\citet{Sutton00PG} and~\citet{Konda00AA} showed that if the approximation $\hat{Q}^\mu(\cdot;\vecw)$ with parameter 
$\vecw$ is {\em compatible}, i.e., $\nabla_\vecw\hat{Q}^\mu(x,a;\vecw)=\nabla\log\mu(a|x;\vectheta)$, and if it minimizes the mean squared error
\begin{equation}
\E^\mu(\vecw)=\int_\Z d\vecz\pi^\mu(\vecz)\big[Q^\mu(\vecz)-\hat{Q}^\mu(\vecz;\vecw)\big]^2 
\label{eq:mse1}
\end{equation}
for parameter value $\vecw^*$, then we may replace $Q^\mu$ with $\hat{Q}^\mu(\cdot;\vecw^*)$ in Equations~\ref{eq:policy_grad} and~\ref{eq:policy_grad_baseline}. The second condition means that $\hat{Q}^\mu(\cdot;\vecw^*)$ is the projection of $Q^\mu$ onto the space $\{\hat{Q}^\mu(\cdot;\vecw)|\vecw\in\Re^n\}$, with respect to a $\ell_2$-norm weighted by $\pi^\mu$. 

An approximation for the action-value function, in terms of a linear combination of basis functions, may be written as $\hat{Q}^\mu(\vecz;\vecw)=\vecw^\top\vecpsi(\vecz)$. This approximation is compatible if the $\vecpsi$'s are compatible with the policy, i.e., $\vecpsi(\vecz;\vectheta)=\nabla\log\mu(a|x;\vectheta)$. Note that compatibility is well defined under Assumption~\ref{ass:1}. Let $\E^\mu(\vecw)$ denote the mean squared error
\begin{equation}
\E^\mu(\vecw)=\int_\Z d\vecz\pi^\mu(\vecz)\big[Q^\mu(\vecz)-\vecw^\top\vecpsi(\vecz)-b(x)\big]^2
\label{eq:mse2}
\end{equation} 
of our compatible linear approximation $\vecw^\top\vecpsi(\vecz)$ and an arbitrary baseline $b(x)$. Let $\vecw^*=\argmin_\vecw\E^\mu(\vecw)$ denote the optimal parameter. It can be shown that the value of $\vecw^*$ does not depend on the baseline $b(x)$. As a result, the mean squared-error problems of Equations~\ref{eq:mse1} and~\ref{eq:mse2} have the same solutions (see e.g.,~\citealt{Bhatnagar07IN,Bhatnagar09NA}). It can also be shown that if the parameter $\vecw$ is set to be equal to $\vecw^*$, then the resulting mean squared error $\E^\mu(\vecw^*)$, now treated as a function of the baseline $b(x)$, is further minimized by setting $b(x)=V^\mu(x)$~\citep{Bhatnagar07IN,Bhatnagar09NA}. In other words, the variance in the action-value function estimator is minimized if the baseline is chosen to be the value function itself. 

A convenient and rather flexible choice for a space of policies that ensures compatibility between the policy and the action-value representation is a parametric exponential family 
\begin{equation*}
\mu(a|x;\vectheta)=\frac{1}{Z_\vectheta(x)}\exp\left(\vectheta^\top\vecphi(x,a) \right),
\end{equation*}
where $Z_\vectheta(x)=\int_\A da\exp\big(\vectheta^\top\vecphi(x,a)\big)$ is a normalizing factor, referred to as the {\em partition function}. It is easy to show that $\vecpsi(\vecz)=\vecphi(\vecz)-\exptE_{a|x}\left[\vecphi(\vecz)\right]$, where $\exptE_{a|x}[\cdot]=\int_\A da\mu(a|x;\vectheta)[\cdot]$, and as a result, $\hat{Q}^\mu(\vecz;\vecw^*)=\vecw^{*\top}\big(\vecphi(\vecz)-\exptE_{a|x}[\vecphi(\vecz)]\big)+b(x)$ is a compatible action-value function for this family of policies. Note that $\exptE_{a|x}[\hat{Q}(\vecz;\vecw^*)]=b(x)$, since 
$\exptE_{a|x}\big[\vecphi(\vecz)-\exptE_{a|x}[\vecphi(\vecz)]\big]=0$. This means that if $\hat{Q}^\mu(\vecz;\vecw^*)$ approximates $Q^\mu(\vecz)$, then $b(x)$ must approximate the value function $V^\mu(x)$. The term $\hat{A}^\mu(\vecz;\vecw^*)=\hat{Q}^\mu(\vecz;\vecw^*)-b(x)=\vecw^{*\top}\big(\vecphi(\vecz)-\exptE_{a|x}[\vecphi(\vecz)]\big)$ approximates the {\em advantage function} $A^\mu(\vecz)=Q^\mu(\vecz)-V^\mu(x)$~\citep{Baird93AU}.



\section{Bayesian Quadrature}
\label{sec:BQ}

Bayesian quadrature (BQ)~\citep{Ohagan91BQ} is, as its name suggests, a Bayesian method for evaluating an integral using samples of its integrand. We consider the problem of evaluating the integral
\begin{equation}
\label{eq:integral}
\rho=\int f(x)g(x)dx .
\end{equation}
If $g(x)$ is a probability density function, i.e., $g(x)=p(x)$, this becomes the problem of evaluating the expected value of $f(x)$. A well known frequentist approach to evaluating such expectations is the Monte-Carlo (MC) method. For MC estimation of such expectations, it is typically required that samples $x_1,x_2,\ldots,x_M$ are drawn from $p(x)$.\footnote{If samples are drawn from some other distribution, importance sampling variants of MC may be used.} The integral in Equation~\ref{eq:integral} is then estimated as
\begin{equation}
\label{eq:mc}
\hat\rho_{MC} = \frac{1}{M} \sum_{i=1}^M f(x_i) .
\end{equation}
It is easy to show that $\hat\rho_{MC}$ is an unbiased estimate of $\rho$, with variance that diminishes to zero as $M \rightarrow \infty$. However, as~\citet{Ohagan87MC} points out, MC estimation is fundamentally unsound, as it violates the likelihood principle, and moreover, does not make full use of the data at hand. The alternative proposed in~\citet{Ohagan91BQ} is based on the following reasoning: In the Bayesian approach, $f(\cdot)$ is random simply because it is unknown. We are therefore uncertain about the value of $f(x)$ until we actually evaluate it. In fact, even then, our uncertainty is not always completely removed, since measured samples of $f(x)$ may be corrupted by noise. Modeling $f$ as a Gaussian process (GP) means that our uncertainty is completely accounted for by specifying a Normal prior distribution over functions. This prior distribution is specified by its mean and covariance, and is denoted by $f(\cdot) \sim \N\big(\bar{f}(\cdot),k(\cdot,\cdot)\big)$. This is shorthand for the statement that $f$ is a GP with prior mean and covariance 
\begin{equation}
\exptE\big[f(x)\big] = \bar{f}(x) \quad \text{and} \quad \Cov\big[f(x),f(x')\big]=k(x,x'), \quad \forall x,x'\in\X,
\end{equation} 
respectively. The choice of kernel function $k$ allows us to incorporate prior knowledge on the smoothness properties of the integrand into the estimation procedure. When we are provided with a set of samples $\D_M = \big\{\big(x_i,y(x_i)\big)\big\}_{i=1}^M$, where $y(x_i)$ is a (possibly noisy) sample of $f(x_i)$, we apply Bayes' rule to condition the prior on these sampled values. If the measurement noise is normally distributed, the result is a Normal posterior distribution of $f|\D_M$. The expressions for the posterior mean and covariance are standard:
\begin{align}
\label{eq:f_post} 
\exptE\big[f(x)|\D_M\big] = \bar{f}(x) &+ \veck(x)^\top\matC (\vecy - \bvecf), \nonumber\\ 
\Cov\big[f(x),f(x')|\D_M\big] &= k(x,x') - \veck(x)^\top\matC \veck(x').
\end{align} 
Here and in the sequel, we make use of the definitions:
\begin{table}[!h]
\begin{center}
\begin{tabular}{lll}
$\bvecf=\big(\bar{f}(x_1),\ldots,\bar{f}(x_M)\big)^\top,\hspace{0.2in}$ &  & $\veck(x)=\big(k(x_1,x),\ldots,k(x_M,x)\big)^\top,$ \\
$\vecy=\big( y(x_1),\ldots,y(x_M) \big)^\top,$ & $[\matK]_{i,j}=k(x_i,x_j),\hspace{0.2in}$ & $\matC=\left(\matK+\matSigma \right)^{-1},$
\end{tabular}
\end{center}
\vspace{-0.25in}
\end{table}

\noindent
where $\matK$ is the kernel (or Gram) matrix, and $[\matSigma]_{i,j}$ is the measurement noise covariance between the $i$th and $j$th samples. It is typically assumed that the measurement noise is i.i.d.,~in which case $\matSigma = \sigma^2 \matI$, where $\sigma^2$ is the noise variance and $\matI$ is the (appropriately sized - here $M \times M$) identity matrix. 

Since integration is a linear operation, the posterior distribution of the integral in Equation~\ref{eq:integral} is also Gaussian, and the posterior moments are given by~\citep{Ohagan91BQ}
\begin{align}
\label{eq:rho_post}
\exptE[\rho|\D_M] &= \int \exptE\big[f(x)|\D_M\big] g(x)dx , \nonumber\\ 
\Var[\rho|\D_M] &= \iint \Cov\big[f(x),f(x')|\D_M\big] g(x)g(x')dxdx'.
\end{align} 
Substituting Equation~\ref{eq:f_post} into Equation~\ref{eq:rho_post}, we obtain
\begin{equation*}
\exptE[\rho|\D_M] = \rho_0 + \vecb^\top\matC (\vecy - \bvecf)\hspace{0.25in}\text{and}\hspace{0.25in} \Var[\rho|\D_M] = b_0 - \vecb^\top\matC \vecb, 
\end{equation*}
where we made use of the definitions:
\begin{equation}
\label{eq:q}
\rho_0 = \int \bar{f}(x) g(x)dx\;,\quad\quad\vecb = \int \veck(x) g(x)dx\;,\quad\quad b_0 = \iint k(x,x') g(x)g(x')dxdx'. 
\end{equation}  
Note that $\rho_0$ and $b_0$ are the prior mean and variance of $\rho$, respectively.

\citet{Rasmussen03BM} experimentally demonstrated how this approach, when applied to the evaluation of an expectation, can outperform MC estimation by orders of magnitude in terms of the mean-squared error.
 
In order to prevent the problem from ``degenerating into infinite regress'', as phrased by~\citet{Ohagan91BQ},\footnote{What O'Hagan means by ``degenerating into infinite regress" is simply that if we cannot compute the posterior integrals of Equation~\ref{eq:q} analytically, then we have started with estimating one integral (Equation~\ref{eq:integral}) and ended up with three (Equation~\ref{eq:q}), and if we repeat this process, this can go forever and leave us with infinite integrals to evaluate. Therefore, for Bayesian MC to work, it is crucial to be able to analytically calculate the posterior integrals, and this can be achieved through the way we divide the integrant into two parts and the way we select the kernel function.} we should choose the functions $g$, $k$, and $\bar{f}$ so as to allow us to solve the integrals in Equation~\ref{eq:q} analytically. For example, O'Hagan provides the analysis required for the case where the integrands in Equation~\ref{eq:q} are products of multivariate Gaussians and polynomials, referred to as Bayes-Hermite quadrature. One of the contributions of our work is in providing analogous analysis for kernel functions that are based on the {\em Fisher kernel}~\citep{Jaakkola98EG,ShaweTaylor04KM}.

It is important to note that in MC estimation, samples must be drawn from the distribution $p(x)=g(x)$, whereas in the Bayesian approach, samples may be drawn from arbitrary distributions. This affords us with flexibility in the choice of sample points, allowing us, for instance, to actively design the samples $x_1,\ldots,x_M$ so as to maximize information gain.


\subsection{Vector-Valued Integrals}
\label{subsec:non-scalar-int}

\citet{Ohagan91BQ} treated the case where the integral to be estimated is a scalar-valued integral.  However, in the context of our PG method, it is useful to consider vector-valued integrals, such as the gradient of the expected return with respect to the policy parameters, which we shall study in Section~\ref{sec:BPG}. In the BQ framework, an integral of the form in Equation~\ref{eq:integral} may be vector-valued for one of two possible reasons: either $f$ is a vector-valued GP and $g$ is a scalar-valued function, or $f$ is a scalar-valued GP and $g$ is a vector-valued function. These two possibilities correspond to two very different data-generation models. In the first of these, an $n$-valued function $f(\cdot) = \big(f_1(\cdot), \ldots, f_n(\cdot) \big)^\top$ is sampled from the GP distribution of $f$. This distribution may include correlations between different components of $f$. Hence, in general, to specify the GP prior distribution, one needs to specify not only the covariance kernel of each component $j$ of $f$, $k_{j,j}(x,x') = \Cov\big[f_j(x),f_j(x')\big]$, but also cross-covariance kernels for pairs of different components, $k_{j,\ell}(x,x') = \Cov\big[f_j(x),f_\ell(x')\big]$. Thus, instead of a single kernel function, we now need to specify a matrix of kernel functions.\footnote{Note that to satisfy the symmetry property of the covariance operator, we require that $k_{j,\ell}(x,x') = k_{\ell,j}(x',x) =  k_{\ell,j}(x,x')$, for all $x, x' \in \X$ and $j, \ell \in \{ 1,\ldots, n \} $.}  Similarly, we also need to specify a vector of prior means, consisting of a function for each component: $\bar{f}_j(x) = \exptE\big[f_j(x)\big]$. The distribution of the measurement noise added to $f(x)$ to produce $y(x)$ may also include correlations, requiring us to specify an array of noise covariance matrices $\matSigma_{j,\ell}$. As we show below, the GP posterior distribution is also specified in similar terms. 

In the second model, a scalar-valued function is sampled from the GP prior distribution, which is specified by a single prior mean function and a single prior covariance-kernel function. Gaussian noise may be added, and the result is then multiplied by each of the components of the $n$-valued function $g$ to produce the integrand. This model is significantly simpler, both conceptually and in terms of the number of parameters required to specify it.
To see how a model of the first kind may arise, consider the following example.
\begin{example}
Let $\rho(\vectheta)=\int f(x;\vectheta) g(x)dx$, where $f$ is a scalar GP, parameterized by a vector of parameters $\vectheta$. Its prior mean and covariance functions must therefore depend on $\vectheta$. We denote these dependencies by writing: 
\begin{equation*}
\exptE\big[f(x;\vectheta)\big] = \bar{f}(x;\vectheta), \quad \Cov\big[f(x;\vectheta),f(x';\vectheta)\big]=k(x,x';\vectheta),\quad\forall x,x'\in\X .
\end{equation*} 
We choose $\bar{f}(x;\vectheta)$ and $ k(x,x';\vectheta)$ so as to be once and twice differentiable in $\vectheta$, respectively. Suppose now that we are not interested in estimating $\rho(\vectheta)$, but rather in its gradient with respect to the parameters $\vectheta$: $\nabla_\vectheta \rho(\vectheta) = \int \nabla_\vectheta f(x;\vectheta) g(x)dx$. It may be easily verified that the mean functions and covariance kernels of the vector-valued GP 
$\nabla_\vectheta f(x;\vectheta)$ are given by
\begin{equation*}
\exptE\big[\nabla_\vectheta f(x;\vectheta)\big] = \nabla_\vectheta  \bar{f}(x;\vectheta)\quad\text{\em and}\quad\Cov\big[\partial_{\theta_j} f(x;\vectheta), \partial_{\theta_\ell} f(x';\vectheta)\big] = \partial_{\theta_j} \partial_{\theta_\ell} k(x,x';\vectheta) ,
\end{equation*}
where $\partial_{\theta_j}$ denotes the $j$th component of $\nabla_\vectheta$. \hfill$\Box$
\end{example}

Propositions~\ref{prop:1} and~\ref{prop:2} specify the form taken by the mean and covariance functions of the integral GP under the two models discussed above.

\begin{proposition}[Vector-valued GP] \label{prop:1}
Let $f$ be an $n$-valued GP with mean functions $\bar{f}_j(x)=\exptE\big[f_j(x)\big]$ and covariance functions $k_{j,\ell}(x,x')=\Cov\big[f_j(x),f_\ell(x')\big],\;\forall j,\ell\in\{1,\ldots, n\}$, and let $g$ be a scalar-valued function. Then, the mean and covariance of $\rho$ defined by Equation~\ref{eq:integral} are of the following form:
\begin{equation*}
\exptE\big[\rho_j\big] = \int  \bar{f}_j(x) g(x) dx,\quad\Cov\big[\rho_j, \rho_\ell\big]=\iint k_{j,\ell}(x,x') g(x)g(x')dxdx',\quad\forall j,\ell\in\{1,\ldots, n\}.
\end{equation*}
\end{proposition}

\begin{proposition}[Scalar-valued GP]  \label{prop:2}
Let $f$ be a scalar-valued GP with mean function $\bar{f}(x)=\exptE\big[f(x)\big]$ and covariance function $k(x,x')=\Cov\big[f(x), f(x')\big]$, and let $g$ be an $n$-valued function. Then, the mean and covariance of $\rho$ defined by Equation~\ref{eq:integral} are of the following form:
\begin{equation*}
\exptE\big[\rho_j\big]=\int\bar{f}(x)g_j(x)dx,\quad\Cov\big[\rho_j,\rho_\ell\big]=\iint k(x,x')g_j(x)g_\ell(x')dxdx',\quad\forall j,\ell\in\{1,\ldots, n\}.
\end{equation*}
\end{proposition}
The proofs of these two propositions follow straightforwardly from the definition of the covariance operator in terms of expectations, and the order-exchangeability of GP expectations and integration with respect to $x$.

To wrap things up, we need to describe the form taken by the posterior moments of $f$ in the vector-valued GP case.
Using the standard Gaussian conditioning formulas, it is straightforward to show that
\begin{align}
\label{eq:fvec_post} 
\exptE\big[f_j(x)|\D_M\big] &= \bar{f}_j(x) + \sum_{m,m'} \veck_{j,m}(x)^\top \matC_{m,m'} (\vecy_{m'} - \bvecf_{m'}), \nonumber\\ 
\Cov\big[f_j(x),f_\ell(x')|\D_M\big] &= k_{j,\ell}(x,x') - \sum_{m,m'} \veck_{j,m}(x)^\top \matC_{m,m'} \veck_{m',\ell}(x'),
\end{align}
where
\begin{align*}
& \bvecf_j = \big(\bar{f}_j(x_1),\ldots,\bar{f}_j(x_M)\big)^\top,
& \veck_{j,\ell}(x) = \big( k_{j,\ell}(x_1,x),\ldots,k_{j,\ell}(x_M,x) \big)^\top,
\\
& \vecy_\ell = \big( y_\ell(x_1),\ldots,y_\ell(x_M) \big)^\top, 
&[\matK_{j,\ell}]_{i,i'} = k_{j,\ell}(x_i,x_{i'}),
\end{align*}
\[
\matK =
\left[
\begin{array}{lll}
\matK_{1,1} & \ldots & \matK_{1,n} \\
\vdots & & \vdots \\
\matK_{n,1} & \ldots & \matK_{n,n} \\
\end{array}
\right],
\quad
\matSigma =
\left[
\begin{array}{lll}
\matSigma_{1,1} & \ldots & \matSigma_{1,n} \\
\vdots & & \vdots \\
\matSigma_{n,1} & \ldots & \matSigma_{n,n} \\
\end{array}
\right],
\quad
\matC = \left(\matK+\matSigma \right)^{-1},
\]
and $\matC_{j,\ell}$ is the $(j,\ell)$th $M \times M$ block of $\matC$. The posterior moments of $f$, given in Equation~\ref{eq:fvec_post}, may now be substituted into the expressions for the moments of the integral $\rho$, given in Proposition~\ref{prop:1}, to produce the posterior moments of $\rho$ in the vector-valued GP case. 

Clearly, models of the first type (vector-valued GP) are potentially richer and more complex than models of the second type (scalar-valued GP), as the latter only requires us to define a single prior mean function, a single prior covariance kernel function, and a single noise covariance matrix; while the former requires us to define $n$ prior mean functions, and $n(n+1)/2$ prior covariance kernel functions and noise covariance matrices. One way to simplify the first type of models is to define a single prior mean function $\bar{f}$, a single prior covariance kernel function $k$,  and to postulate that $\bar{f}_j(x) = \bar{f}(x)$, $k_{j,\ell}(x,x') = \delta_{j,\ell} \, k(x,x')$, and $\matSigma_{j,\ell} = \delta_{j,\ell} \, \matSigma,\;\forall j,\ell\in\{1,\ldots,n\}$, where $\delta$ denotes the Kronecker delta function. Applying these simplifying assumptions to the expressions for the posterior moments (Equation~\ref{eq:fvec_post}) results in a complete decoupling between the posterior moments for the different components of $f$, and consequently a decoupling between the components of the integral $\rho$ as well, since Equation~\ref{eq:fvec_post} becomes
\begin{align}
\label{eq:fvec_post_simplified} 
\exptE\big[f_j(x)|\D_M\big] &= \bar{f}_j(x) +  \veck_{j,j}(x)^\top \matC_{j,j} (\vecy_{j} - \bvecf_{j}), \nonumber\\ 
\Cov\big[f_j(x),f_\ell(x')|\D_M\big] &= \delta_{j,\ell} \left( k_{j,j}(x,x') - \veck_{j,j}(x)^\top \matC_{j,\ell} \veck_{\ell,\ell}(x') \right),
\end{align}
where $\matC_{j,\ell} = \delta_{j,\ell}\left(\matK_{j,\ell}+\matSigma_{j,\ell} \right)^{-1}$. Note that all the terms in Equation~\ref{eq:fvec_post_simplified}, except $\vecy_{j}$, do not depend on the indices $j$ and $\ell$. In other words, these simplifying assumptions amount to assuming that $\rho$ is a vector of $n$ independent integrals, each of which may be estimated individually as 
\begin{align*}
\exptE\big[f_j(x)|\D_M\big] &= \bar{f}(x) + \veck(x)^\top \matC (\vecy_{j} - \bvecf), \nonumber\\ 
\Cov\big[f_j(x),f_\ell(x')|\D_M\big] &= \delta_{j,\ell} \left( k(x,x') - \veck(x)^\top \matC \veck(x') \right),
\end{align*}
where $\matC=(\matK+\matSigma)^{-1}$. It should, however, be kept in mind that ignoring correlations between the components of $f$, when such correlations exist, may result in suboptimal use of the available data (\citealp[see][Chapter 9]{Rasmussen06GP} for references on GP regression with multiple outputs).

\section{Bayesian Policy Gradient}
\label{sec:BPG}

In this section, we use vector-valued Bayesian quadrature to estimate the gradient of the expected return with respect to the policy parameters, allowing us to propose new {\em Bayesian policy gradient} (BPG) algorithms. In the frequentist approach to policy gradient, the performance measure used is $\eta(\vectheta)=\int \bar{R}(\xi) \Pr(\xi;\vectheta) d\xi$ (Equation~\ref{exp-ret}). In order to serve as a useful performance measure, it has to be a deterministic function of the policy parameters $\vectheta$. This is achieved by averaging the cumulative return $R(\xi)$ over all possible paths $\xi$ and all possible returns accumulated in each path. In the Bayesian approach we have an additional source of randomness, namely, our subjective Bayesian uncertainty concerning the process generating the cumulative return. Let us denote 
\begin{equation}
\eta_B(\vectheta) =
\int\bar{R}(\xi) \Pr(\xi;\vectheta) d\xi,
\label{eq:eta_bayes}
\end{equation}
where $\eta_B(\vectheta)$ is a random variable because of the Bayesian uncertainty. Under the quadratic loss, the optimal Bayesian performance measure is the posterior expected value of $\eta_B(\vectheta)$, $\exptE\big[\eta_B(\vectheta)|\D_M\big]$. However, since we are interested in optimizing the performance rather than evaluating it,\footnote{Although evaluating the posterior distribution of performance is an interesting question in its own right.} we would rather evaluate the posterior distribution of the {\em gradient} of $\eta_B(\vectheta)$ with respect to the policy parameters $\vectheta$. The posterior mean of the gradient is \footnote{We may interchange the order of the gradient and the expectation operators for the mean, $\nabla \exptE\big[\eta_B(\vectheta)\big]=\exptE\big[\nabla \eta_B(\vectheta)\big]$, but the same is not true for the variance, namely,
$\nabla\Var\big[\eta_B(\vectheta)\big]\neq\Cov\big[\nabla \eta_B(\vectheta)\big]$.}     
\begin{equation}
\label{eq:grad:Bayes}
\nabla\exptE\big[\eta_B(\vectheta)|\D_M\big]=\exptE\big[\nabla\eta_B(\vectheta)|\D_M\big]=\exptE\left[\int R(\xi)\frac{\nabla \Pr(\xi;\vectheta)}{\Pr(\xi;\vectheta)}\Pr(\xi;\vectheta)d\xi \Big| \D_M\right].
\end{equation}
Consequently, in BPG we cast the problem of estimating the gradient of the expected return (Equation~\ref{eq:grad:Bayes}) in the form of Equation~\ref{eq:integral}. As described in Section~\ref{sec:BQ}, we need to partition the integrand into two parts, $f(\xi;\vectheta)$ and $g(\xi;\vectheta)$. We will model $f$ as a GP and assume that $g$ is a function known to us. We will then proceed by calculating the posterior moments of the gradient  $\nabla\eta_B(\vectheta)$ conditioned on the observed data. 
Because in general, $R(\xi)$ cannot be known exactly, even for a given $\xi$ (due to the stochasticity of the rewards), $R(\xi)$ should always belong to the GP part of the model, i.e., $f(\xi;\vectheta)$. Interestingly, in certain cases it is sufficient to know the Fisher information matrix corresponding to $\Pr(\xi;\vectheta)$, rather than having exact knowledge of $\Pr(\xi;\vectheta)$ itself. We make use of this fact in the sequel. In the next two sections, we investigate two different ways of partitioning the integrand in Equation~\ref{eq:grad:Bayes}, resulting in two distinct Bayesian policy gradient models. 


\subsection{Model 1 -- Vector-Valued GP}
\label{subsec:M1}

In our first model, we define $g$ and $f$ as follows:
\begin{equation*}
g(\xi;\vectheta)=\Pr(\xi;\vectheta) \quad,\quad f(\xi;\vectheta)=\bar{R}(\xi)\frac{\nabla \Pr(\xi;\vectheta)}{\Pr(\xi;\vectheta)}=\bar{R}(\xi)\nabla\log \Pr(\xi;\vectheta).
\end{equation*}  
We place a vector-valued GP prior over $f(\xi;\vectheta)$ which induces a GP prior over the corresponding noisy measurement $y(\xi;\vectheta)=R(\xi) \nabla\log\Pr(\xi;\vectheta)$. We adopt the simplifying assumptions discussed at the end of Section~\ref{subsec:non-scalar-int}: We assume that each component of $f(\xi;\vectheta)$ may be evaluated independently of all other components, and use the same kernel function and noise covariance for all components of $f(\xi;\vectheta)$. We therefore omit the component index $j$ from $\matK_{j,j}$, $\matSigma_{j,j}$ and $\matC_{j,j}$, denoting them simply as $\matK$, $\matSigma$ and $\matC$, respectively.
Hence, for the $j$th component of $f$ and $y$ we have, a priori
\begin{align*}
\vecf_j &= \big( f_j(\xi_1;\vectheta),\ldots,f_j(\xi_M;\vectheta) \big)^\top \sim \N(\vec0,\matK), \\
\vecy_j &= \big( y_j(\xi_1;\vectheta),\ldots,y_j(\xi_M;\vectheta) \big)^\top \sim \N(\vec0,\matK+\matSigma).
\end{align*}
In this vector-valued GP model, the posterior mean and covariance of $\nabla\eta_B(\vectheta)$ are
\begin{equation}
\label{model1_posteriors}
\exptE\big[\nabla\eta_B(\vectheta)|\D_M\big] = \matY \matC \vecb \quad\quad \text{and} \quad\quad \Cov\big[\nabla\eta_B(\vectheta)|\D_M\big] = \left(b_0 - \vecb^\top \matC \vecb \right) \matI,
\end{equation} 
respectively, where
\begin{equation} \label{integrals1}
\matY = \left[
\begin{array}{l}
	\vecy_1^\top \\
	\vdots \\
	\vecy_n^\top \\
\end{array}
\right], \quad
\vecb=\int\veck(\xi)\Pr(\xi;\vectheta)d\xi\;, \quad \text{and} \quad b_0=\iint k(\xi,\xi')\Pr(\xi;\vectheta)\Pr(\xi';\vectheta)d\xi d{\xi'}.
\end{equation} 

Our choice of kernel, which allows us to derive closed-form expressions for $\vecb$ and $b_0$, and as a result for the posterior moments of the gradient, is the quadratic Fisher kernel~\citep{Jaakkola98EG,ShaweTaylor04KM}
\begin{equation}
\label{kernel1}
k(\xi_i,\xi_j)=\left(1+\vecu(\xi_i)^\top\matG(\vectheta)^{-1}\vecu(\xi_j)\right)^2,
\end{equation} 
where $\vecu(\xi)=\nabla\log\Pr(\xi;\vectheta)$ is the Fisher score function of the path $\xi$ defined by Equation~\ref{eq:score}, and $\matG(\vectheta)$ is the corresponding Fisher information matrix defined as\footnote{To simplify notation, we omit  $\matG$'s dependence on the policy parameters $\vectheta$, and denote $\matG(\vectheta)$ as $\matG$ in the sequel.}
\begin{equation}
\label{Fisher:Matrix}
\matG(\vectheta)=\exptE\big[\vecu(\xi)\vecu(\xi)^\top\big]=\int\vecu(\xi)\vecu(\xi)^\top\Pr(\xi;\vectheta)d\xi.
\end{equation}
\begin{proposition}\label{prop:3}
Using the quadratic Fisher kernel from Equation~\ref{kernel1}, the integrals $\vecb$ and $b_0$ in Equation~\ref{integrals1} have the following closed form expressions
\begin{equation*}
(\vecb)_i=1+\vecu(\xi_i)^\top\matG^{-1}\vecu(\xi_i)\hspace{0.5in}\text{and}\hspace{0.5in}b_0=1+n.
\end{equation*}
\end{proposition}
\begin{proof}
See Appendix~\ref{sec:proofP3}.
\end{proof}


\subsection{Model 2 -- Scalar-Valued GP}
\label{subsec:M2}

In our second model, we define $g$ and $f$ as follows:
\begin{equation*}
g(\xi;\vectheta)=\nabla\Pr(\xi;\vectheta) \quad , \quad f(\xi)=\bar{R}(\xi).
\end{equation*} 
Now $g$ is a vector-valued function, while $f$ is a scalar valued GP representing the expected return of the path given as its argument. The noisy measurement corresponding to $f(\xi_i)$ is $y(\xi_i) = R(\xi_i)$, namely, the {\em actual} return accrued while following the path $\xi_i$. In this model, the posterior mean and covariance of the gradient $\nabla\eta_B(\vectheta)$ are 
\begin{equation}
\label{model2_posteriors}
\exptE\big[\nabla\eta_B(\vectheta)|\D_M\big] = \matB \matC \vecy \quad\quad \text{and} \quad\quad \Cov\big[\nabla\eta_B(\vectheta)|\D_M\big] = \matB_0-\matB \matC \matB^\top,
\end{equation}
respectively, where
\begin{align} \label{integrals2}
\vecy &= \big( R(\xi_1), \ldots, R(\xi_M) \big)^\top, \quad\quad\quad\quad\quad\quad\quad\quad\quad\quad \matB = \int \nabla\Pr(\xi;\vectheta) \veck(\xi)^\top d\xi\;, \nonumber\\
\matB_0 &= \iint k(\xi,\xi')\nabla \Pr(\xi;\vectheta)\nabla\Pr(\xi';\vectheta)^\top d\xi d{\xi'}.
\end{align}

Here, our choice of kernel function, which again allows us to derive closed-form expressions for $\matB$ and $\matB_0$, is the Fisher kernel~\citep{Jaakkola98EG,ShaweTaylor04KM}
\begin{equation}
\label{kernel2}
k(\xi,\xi') = \vecu(\xi)^\top\matG^{-1}\vecu(\xi').
\end{equation} 
\begin{proposition}\label{prop:4}
Using the Fisher kernel from Equation~\ref{kernel2}, the integrals $\matB$ and $\matB_0$ in Equation~\ref{integrals2} have the following closed-form expressions
\begin{equation*}
\matB=\matU \quad\quad \text{and} \quad\quad \matB_0=\matG,
\end{equation*} 
{\em where} $\matU=\big[\vecu(\xi_1), \ldots ,\vecu(\xi_M) \big]$.
\end{proposition}  
\begin{proof}
See Appendix~\ref{sec:proofP4}.
\end{proof}

Table~\ref{tab:models} summarizes the two BPG models presented in Sections~\ref{subsec:M1} and~\ref{subsec:M2}. Our choice of Fisher-type kernels was motivated by the notion that a good representation should depend on the process generating the data (see~\citealp{Jaakkola98EG,ShaweTaylor04KM}, for a thorough discussion). Our particular selection of linear and quadratic Fisher kernels were guided by the desideratum that the posterior moments of the gradient be analytically tractable as discussed in Section~\ref{sec:BQ}. 

\begin{table}
\begin{center}
\begin{tabular}{|l|l|l|}
\hline
& {\bf Model 1} & {\bf Model 2} \\
\hline
Deter. factor ($g$)
& $g(\xi;\vectheta)=\Pr(\xi;\vectheta)$ 
& $g(\xi;\vectheta)=\nabla\Pr(\xi;\vectheta)$ \\
GP factor ($f$)
& $f(\xi;\vectheta)=\bar{R}(\xi)\nabla\log\Pr(\xi;\vectheta)$ 
& $f(\xi)=\bar{R}(\xi)$ \\
Measurement ($y$)
& $y(\xi;\vectheta) = R(\xi)\nabla\log\Pr(\xi;\vectheta)$ 
& $y(\xi) = R(\xi)$\\
Prior mean of $f$ 
& $\exptE\big[f_j(\xi;\vectheta)\big] = 0$ 
& $\exptE\big[f(\xi)\big] = 0$ \\
Prior Cov. of $f$ 
& $\Cov\big[f_j(\xi;\vectheta),f_\ell(\xi';\vectheta)\big] = \delta_{j,\ell} \, k(\xi,\xi')$ 
& $\Cov\big[f(\xi),f(\xi')\big] = k(\xi,\xi')$\\
Kernel function 
& $k(\xi,\xi') = \big( 1+\vecu(\xi)^\top \matG^{-1} \vecu(\xi') \big)^2$
& $k(\xi,\xi') = \vecu(\xi)^\top \matG^{-1} \vecu(\xi')$\\
$\exptE\big[\nabla\eta_B(\vectheta)|\D_M\big]=$ 
& $\matY\matC\vecb$
& $\matB\matC\vecy$\\
$\Cov\big[\nabla\eta_B(\vectheta)|\D_M\big]=$ 
& $(b_0-\vecb^\top \matC \vecb) \matI$
& $\matB_0 - \matB \matC \matB^\top$\\
$\vecb$ or $\matB$ 
& $(\vecb)_i=1+\vecu(\xi_i)^\top \matG^{-1} \vecu(\xi_i)$ 
& $\matB=\matU$  \\
$b_0$ or $\matB_0$
& $b_0 = 1+n$ 
& $\matB_0 = \matG$\\
\hline
\end{tabular}
\end{center}
\caption{Summary of the Bayesian policy gradient Models 1 and 2.}
\label{tab:models}
\end{table}

As described above, in either model, we are restricted in the choice of kernel (quadratic Fisher kernel in Model~1~and Fisher kernel in Model~2) in order to be able to derive closed-form expressions for the posterior mean and covariance of the gradient integral. The loss due to this restriction depends on the problem at hand and is hard to quantify. This loss is exactly the loss of selecting an inappropriate prior in any Bayesian algorithm or, more generally, of choosing a wrong representation (function space) in a machine learning algorithm (referred to as {\em approximation error} in approximation theory). However, the experimental results of Section~\ref{sec:BPG-experiments} indicate that this restriction did not cause a significant error (especially for Model~1) in our gradient estimates, as those estimated by BPG were more accurate than the ones estimated by the MC-based method, given the same number of samples.


\subsection{A Bayesian Policy Gradient Evaluation Algorithm}
\label{subsec:BPG_eval_alg}

We can now use our two BPG models to define corresponding algorithms for evaluating the gradient of the expected return with respect to the policy parameters. Pseudo-code for these algorithms is shown in Algorithm~\ref{alg:BPG-Eval}. The generic algorithm (for either model) takes a set of policy parameters $\vectheta$ and a sample size $M$ as input, and returns an estimate of the posterior moments of the gradient of the expected return with respect to the policy parameters. This algorithm generates $M$ sample paths to evaluate the gradient. For each path $\xi_i$, the algorithm first computes its score function $\vecu(\xi_i)$ (Line 6). The score function is needed for computing the kernel function $k$, the measurement $\vecy$ in Model 1, and $\vecb$ or $\matB$. The algorithm then computes the return $R$ and the measurement $y(\xi_i)$ for the observed path $\xi_i$ (Lines 7 and 9), and updates the kernel matrix $\matK$ (Line 8) using 
\begin{equation}
\label{K_dict}
\matK := 
\begin{bmatrix}
\matK & \veck(\xi_i) \\
\veck^\top(\xi_i) & k(\xi_i,\xi_i) 
\end{bmatrix} .
\end{equation}
Finally, the algorithm adds the measurement error $\matSigma$ to the covariance matrix $\matK$ (Line 12) and computes the posterior moments of the policy gradient (Line 14). $\matB(:,i)$ on Line 10 denotes the $i$th column of the matrix $\matB$.

\begin{algorithm}
\begin{algor}[1]
\item [{*}]{\bf BPG\_Eval}$(\vectheta,M)$ \\
\hspace*{0.25in}$\bullet\;\;$sample size $M>0$\\
\hspace*{0.25in}$\bullet\;\;$a vector of policy parameters $\vectheta\in\Re^n$ \\ 
\item [{*}]Set $\matG=\matG(\vectheta)\;\;\;,\;\;\;\D=\emptyset$
\item [for]$i=1$ to $M$
\item [{*}]Sample a path $\xi_i$ using the policy $\mu(\vectheta)$
\item [{*}]$\D := \D \bigcup $\{$\xi_i$\}
\item [{*}]$\vecu(\xi_i)=\sum_{t=0}^{T_i-1}\nabla\log\mu(a_{t,i}|x_{t,i};\vectheta)$ \\
\item [{*}]$R(\xi_i)=\sum_{t=0}^{T_i-1}r(x_{t,i},a_{t,i})$
\item [{*}]Update $\matK$ using Equation~\ref{K_dict}
\item [{*}]$y(\xi_i)=R(\xi_i)\vecu(\xi_i)$\hspace{0.85in}(Model 1)\hspace{0.275in}or\hspace*{0.275in}$y(\xi_i)=R(\xi_i)$\hfill(Model 2) \\
\item [{*}]$(\vecb)_i=1+\vecu(\xi_i)^\top\matG^{-1}\vecu(\xi_i)$\hspace{0.275in}(Model 1)\hspace*{0.275in}or\hspace{0.275in}$\matB(:,i)=\vecu(\xi_i)$\hfill(Model 2) \\
\item [endfor]
\item [{*}]$\matC=(\matK+\matSigma)^{-1}$
\item [{*}]$b_0=1+ n$ \hspace{1.45in}(Model 1)\hspace*{0.275in}or\hspace{0.275in}$\matB_0=\matG$ \hfill(Model 2) \\
\item [{*}]Compute the posterior mean and covariance of the policy gradient \\
\hspace*{0.1in}$\exptE\big(\nabla\eta_B(\vectheta)|\D\big)=\matY\matC\vecb\;$\hspace{0.25in},\hspace{0.25in}$\Cov\big(\nabla\eta_B(\vectheta)|\D\big)=(b_0-\vecb^\top\matC\vecb)\matI$\hfill(Model 1)\\ 
or \\
\hspace*{0.1in}$\exptE\big(\nabla\eta_B(\vectheta)|\D\big)=\matB\matC\vecy\;$\hspace{0.25in},\hspace{0.25in}$\Cov\big(\nabla\eta_B(\vectheta)|\D\big)=\matB_0-\matB\matC\matB^\top$\hfill(Model 2) \\
\item [{*}]{\bf return}\hspace{0.12in} $\exptE\big(\nabla\eta_B(\vectheta)|\D\big)\;\;\;$ and $\;\;\;\Cov\big(\nabla\eta_B(\vectheta)|\D\big)$
\end{algor}
\caption{\label{alg:BPG-Eval}A Bayesian Policy Gradient Evaluation Algorithm}
\end{algorithm}

The kernel functions used in Models~1 and~2 (Equations~\ref{kernel1} and~\ref{kernel2}) are both based on the Fisher kernel. Computing the Fisher kernel requires calculating the Fisher information matrix $\matG(\vectheta)$ (Equation~\ref{Fisher:Matrix}). Consequently, every time we update the policy parameters, we need to recompute $\matG$. In Algorithm~\ref{alg:BPG-Eval} we assume that the Fisher information matrix is known. However, in most practical situations this will not be the case, and consequently the Fisher information matrix must be estimated. Let us briefly outline two possible approaches for estimating the Fisher information matrix in an online manner.

\begin{paragraph}{1) Monte-Carlo Estimation:} 
The BPG algorithm generates a number of sample paths using the current policy parameterized by $\vectheta$ in order to estimate the gradient $\nabla\eta_B(\vectheta)$. We can use these generated sample paths to estimate the Fisher information matrix $\matG(\vectheta)$ in an online manner, by replacing the expectation in $\matG$ with empirical averaging as $\hat{\matG}_M(\vectheta)=\frac{1}{M}\sum_{i=1}^M\vecu(\xi_i)\vecu(\xi_i)^\top$. It can be shown that $\hat{\matG}_M$ is an unbiased estimator of $\matG$. One may obtain this estimate recursively $\hat{\matG}_{i+1}=(1-\frac{1}{i})\hat{\matG}_i+\frac{1}{i}\vecu(\xi_i)\vecu(\xi_i)^\top$, or more generally $\hat{\matG}_{i+1}=(1-\zeta_i)\hat{\matG}_i+\zeta_i\vecu(\xi_i)\vecu(\xi_i)^\top$, where $\zeta_i$ is a step-size with $\sum_i\zeta_i=\infty$ and $\sum_i\zeta_i^2<\infty$. Using the Sherman-Morrison matrix inversion lemma, it is possible to directly estimate the inverse of the Fisher information matrix as
\begin{equation*}
\hat{\matG}_{i+1}^{-1}=\frac{1}{1-\zeta_i}\left[\hat{\matG}_i^{-1}-\zeta_i\frac{\hat{\matG}_i^{-1}\vecu(\xi_i)\big(\hat{\matG}_i^{-1}\vecu(\xi_i)\big)^\top}{1-\zeta_i+\zeta_i\vecu(\xi_i)^\top\hat{\matG}_i^{-1}\vecu(\xi_i)}\right].
\end{equation*}
\end{paragraph}

\begin{paragraph}{2) Maximum Likelihood Estimation:} 
The Fisher information matrix defined by Equation~\ref{Fisher:Matrix} depends on the probability distribution over paths. This distribution is a product of two factors, one corresponding to the current policy and the other corresponding to the MDP's state-transition probability $P$ (see Equation~\ref{eq:prob_path}). Thus if $P$ is known, the Fisher information matrix may be evaluated offline. We can model $P$ using a parameterized model and then estimate the maximum likelihood (ML) model parameters. This approach may lead to a model-based treatment of policy gradients, which could allow us to transfer information between different policies. Current policy gradient algorithms, including the algorithms described in this paper, are extremely wasteful of training data, since they do not have any {\em disciplined} way to use data collected for previous policy updates in computing the update of the current policy. Model-based policy gradient may help solve this problem. 
\end{paragraph}


\subsection{BPG Online Sparsification}
\label{subsec:BPG-sparsification}

Algorithm~\ref{alg:BPG-Eval} can be made more efficient, both in time and memory, by sparsifying the solution. Such sparsification may be performed incrementally and helps to numerically stabilize the algorithm when the kernel matrix is singular, or nearly so. Sparsification may, in some cases, reduce the accuracy of the solution (the posterior moments of the policy gradient), but  it often makes the algorithms significantly faster, especially for large sample sizes. 
Here we use an online sparsification method proposed by~\citet{Engel02SO} (see also~\citealp{Csato02SO}) to selectively add a new observed path to a set of {\em dictionary} paths $\tilde{\D}$ used as a basis for representing or approximating the full solution. We only add a new path $\xi_i$ to $\tilde{\D}$, if $k(\xi_i,\xi_i)-\tilde{\veck}(\xi_i)^\top\tilde{\matK}^{-1}\tilde{\veck}(\xi_i)>\tau$, where $\tilde{\veck}$ and 
$\tilde{\matK}$ are the dictionary kernel vector and kernel matrix before observing 
$\xi_i$, respectively, and $\tau$ is a positive threshold parameter that determines the level of accuracy in the approximation as well as the level of sparsity attained. If the new path is added to $\tilde{\D}$ the dictionary kernel matrix $\tilde{\matK}$ is expanded as shown in Equation~\ref{K_dict}.

\begin{proposition}\label{prop:5}
Let $\tilde{\matK}$ be the $m \times m$ sparse kernel matrix, where $m\leq M$ is the cardinality of $\tilde{\D}_M$. Let $\matA$ be the $M\times m$ matrix, whose $i$th row is $[\matA]_{i,|\tilde{\D}_i|}=1$ and $[\matA]_{i,j}=0\;;\;\forall j\neq |\tilde{\D}_i|$, if we add the sample path $\xi_i$ to the set of sample paths, and be $\tilde{\veck}(\xi_i)^\top\tilde{\matK}^{-1}$ followed by zeros, otherwise. Finally, let $(\tilde{\vecb})_i=1+\vecu(\xi_i)^\top\matG^{-1}\vecu(\xi_i)$ and $\tilde{\matB}=[\vecu(\xi_1), \ldots, \vecu(\xi_m)]$ with $\xi_i\in\tilde{\D}$. Then, using the sparsification method described above, the posterior moments of the gradient are given by 
\begin{align*}
&\;\;\exptE\big[\nabla\eta_B(\vectheta)|\D_M\big]=\matY\matSigma^{-1}\matA\big(\tilde{\matK}\matA^\top\matSigma^{-1}\matA+\matI\big)^{-1}\tilde{\vecb} \\
&\hspace{4.75in}\text{{\bf for Model 1}} \\
&\;\;\Cov\big[\nabla\eta_B(\vectheta)|\D_M\big]=\Big[(1+n)-\tilde{\vecb}^\top\matA^\top\matSigma^{-1}\matA\big(\tilde{\matK}\matA^\top\matSigma^{-1}\matA+\matI\big)^{-1}\tilde{\vecb}\Big]\matI \\
&\hspace{-0.175in}\text{and} \\
&\;\;\exptE\big[\nabla\eta_B(\vectheta)|\D_M\big]=\tilde{\matB}\big(\matA^\top\matSigma^{-1}\matA\tilde{\matK}+\matI\big)^{-1}\matA^\top\matSigma^{-1}\vecy \\
&\hspace{4.75in}\text{{\bf for Model 2}} \\
&\;\;\Cov\big[\nabla\eta_B(\vectheta)|\D_M\big]=\matG-\tilde{\matB}\big(\matA^\top\matSigma^{-1}\matA\tilde{\matK}+\matI\big)^{-1}\matA^\top\matSigma^{-1}\matA\tilde{\matB}^\top\;.
\end{align*}
%
%
\end{proposition}

\begin{proof}
See Appendix~\ref{sec:proofP5}.
\end{proof}


\subsection{A Bayesian Policy Gradient Algorithm}
\label{subsec:BPG_alg}

So far we were concerned with estimating the gradient of the expected return with respect to the policy parameters. In this section, we present a Bayesian policy gradient (BPG) algorithm that employs the Bayesian gradient estimation methods proposed in Section~\ref{subsec:BPG_eval_alg} to update the policy parameters. The pseudo-code of this algorithm is shown in Algorithm~\ref{alg:BPG}. The algorithm starts with an initial vector of policy parameters $\vectheta_0$, and updates the parameters in the direction of the posterior mean of the gradient of the expected return estimated by Algorithm~\ref{alg:BPG-Eval}. This is repeated $N$ times, or alternatively, until the 
gradient estimate is sufficiently close to zero. 

\begin{algorithm}
\begin{algor}[1]
\item [{*}]{\bf BPG}$(\vectheta_0,\vecbeta,N,M)$ \\
\hspace*{0.25in}$\bullet\;\;$initial policy parameters $\vectheta_0\in\Re^n$ \\
\hspace*{0.25in}$\bullet\;\;$learning rates $\beta_j\;,\;j=0,\ldots,N-1$ \\
\hspace*{0.25in}$\bullet\;\;$number of policy updates $N>0$ \\ 
\hspace*{0.25in}$\bullet\;\;$sample size $M>0$ for the gradient evaluation algorithm ({\bf BPG\_Eval}) \\
\item [for]$j=0$ to $N-1$
\item [{*}]$\Delta\vectheta_j=\exptE\big(\nabla\eta_B(\vectheta_j)|\D_M\big)\;\;$ from $\;\;${\bf BPG\_Eval}$(\vectheta_j,M)$
\item [{*}]$\vectheta_{j+1}=\vectheta_j+\beta_j\Delta\vectheta_j$\hfill(Conventional Gradient) \\
or \\
$\vectheta_{j+1}=\vectheta_j+\beta_j\matG(\vectheta_j)^{-1}\Delta\vectheta_j$\hfill(Natural Gradient) \\
\item [endfor]
\item [{*}]{\bf return} $\vectheta_N$
\end{algor}
\caption{\label{alg:BPG}A Bayesian Policy Gradient Algorithm}
\end{algorithm}


\section{Extension to Partially Observable Markov Decision Processes}
\label{sec:pomdp}

The Bayesian policy gradient models and algorithms of Section~\ref{sec:BPG} can be extended to partially observable Markov decision processes (POMDPs) along the same lines as in Section 6 of~\citet{Baxter01IP}. In the partially observable case, the stochastic parameterized policy $\mu(\cdot|\cdot;\vectheta)$ controls a POMDP, i.e., the policy has access to an observation process that depends on the state, but it may not observe the state itself directly.

Specifically, for each state $x\in\X$, an observation $o\in\O$ is generated independently according to a probability distribution $P_o$ over observations in $\O$. We denote the probability of observation $o$ at state $x$ by $P_o(o|x)$. A stationary stochastic parameterized policy $\mu(\cdot|\cdot;\vectheta)$ is a function mapping observations $o\in\O$ into probability distributions over the actions $\mu(\cdot|o;\vectheta)\in\P(\A)$. In this case, the probability of a path $\xi=(x_0,a_0,x_1,a_1,\ldots,x_{T-1},a_{T-1},x_T)$, $T\in\{0,1,\ldots,\infty\}$ generated by the Markov chain induced by policy $\mu(\cdot|\cdot;\vectheta)$ is given by
\begin{equation*}
\Pr(\xi;\mu)=\Pr(\xi;\vectheta)=P_0(x_0)\int\prod_{t=0}^{T-1}P_o(o_t|x_t)\mu(a_t|o_t;\vectheta)P(x_{t+1}|x_t,a_t)do_0da_0\ldots do_{T-1}da_{T-1}.
\end{equation*}
The Fisher score of this path may be written as 
\begin{equation*}
\vecu(\xi;\vectheta) = \nabla\log\Pr(\xi;\vectheta) = \frac{\nabla\Pr(\xi;\vectheta)}{\Pr(\xi;\vectheta)} = \int\left(\sum_{t=0}^{T-1}\nabla\log\mu(a_t|o_t;\vectheta)\right)do_0da_0\ldots do_{T-1}da_{T-1},
\end{equation*}
which is the same as in the observable case (Equation~\ref{eq:score}), except here the policy is defined over observations instead of states. As a result, the models and algorithms of Section~\ref{sec:BPG} may be used in the partially observable case with no change, substituting observations for states. 

Moreover, similarly to the gradient estimated by the GPOMDP algorithm in~\citet{Baxter01IP}, the gradient estimated by Algorithm~\ref{alg:BPG-Eval}, $\nabla\eta_B(\vectheta)$, may be employed with the conjugate-gradients and line-search methods of~\citet{Baxter01EI} for making better use of gradient information. This allows us to exploit the information contained in the gradient estimate more aggressively than by simply adjusting the parameters by a small amount in the direction of $\nabla\eta_B(\vectheta)$. Conjugate-gradients and line-search are two widely used techniques in non-stochastic optimization that allow us to find better gradient directions than the pure gradient direction, and to obtain better step sizes, respectively.  

Note that in this section, we followed~\citet{Baxter01IP} (the GPOMDP algorithm) and considered stochastic policies that map observations to actions. However, as mentioned by~\citet{Baxter01IP}, it is immediate that the same algorithm works for any finite history of observations. Moreover, along the same way that~\citet{Aberdeen01PG} showed that GPOMDP can be extended to apply to policies with internal state, our BPG POMDP algorithm can also be extended to handle such policies.

\section{BPG Experimental Results}
\label{sec:BPG-experiments}

In this section, we compare the Bayesian quadrature (BQ) and the plain MC gradient estimates on a simple bandit problem as well as on a continuous state and action linear quadratic regulator (LQR). We also evaluate the performance of the Bayesian policy gradient (BPG) algorithm described in Algorithm~\ref{alg:BPG} on the LQR, and compare it with a 
Monte-Carlo based policy gradient (MCPG) algorithm.


\subsection{A Simple Bandit Problem}
\label{subsec:bandit}

The goal of this example is to compare the BQ and MC estimates of the gradient (for some fixed set of policy parameters) using the same sample. Our bandit problem has a single state and a continuous action space $\A=\Re$, thus, each path $\xi_i$ consists of a single action $a_i$. The policy, and therefore also the distribution over the paths is given by $a\sim\N(\theta_1=0,\theta_2^2=1)$. The parameters $\theta_1$ and $\theta_2$ are the mean and the standard deviation of this distribution. The score function of the path $\xi=a$ and the Fisher information matrix for the policy are $\vecu(\xi)=[a,a^2-1]^\top$ and 
$\matG=\text{diag}(1,2)$, respectively.

Table~\ref{tab:bandit} shows the exact gradient of the expected return and its MC and BQ estimates using $10$ and $100$ samples for two instances of the bandit problem corresponding to two different deterministic reward functions $r(a)=a$ and $r(a)=a^2$. The average over $10^4$ runs of the MC and BQ estimates and their standard deviations are reported in Table~\ref{tab:bandit}. The true gradient is analytically tractable and is reported as ``Exact" in Table~\ref{tab:bandit} for reference.

\begin{table}
\centering
\begin{footnotesize}
\begin{tabular}{|c|c|c|c|c|c|} \hline
& Exact & MC (10) & BQ (10) & MC (100) & BQ (100)\\ \hline
$r(a)=a$ & $\begin{pmatrix} 1 \\ 0 \end{pmatrix}$ 
                 & $\begin{pmatrix}0.995\pm0.438 \\ -0.001\pm0.977\end{pmatrix}$ 
                 & $\begin{pmatrix}0.986\pm0.050 \\ 0.001\pm0.060\end{pmatrix}$ 
                 & $\begin{pmatrix}1.000\pm0.140 \\ 0.004\pm0.317\end{pmatrix}$ 
                 & $\begin{pmatrix}1.000\pm0.000001 \\ 0.000\pm 0.000004\end{pmatrix}$ \\ \hline
$r(a)=a^2$ & $\begin{pmatrix} 0 \\ 2 \end{pmatrix}$ 
                     & $\begin{pmatrix}0.014\pm1.246 \\ 2.034\pm2.831\end{pmatrix}$
                     & $\begin{pmatrix}0.001\pm0.082 \\ 1.925\pm0.226\end{pmatrix}$
                     & $\begin{pmatrix}0.005\pm0.390 \\ 1.987\pm0.857\end{pmatrix}$
                     & $\begin{pmatrix}0.000\pm0.000003 \\ 2.000\pm0.000011\end{pmatrix}$ \\ \hline
\end{tabular}
\end{footnotesize} 
\caption{The true gradient of the expected return and its MC and BQ estimates for two instances of the bandit problem corresponding to two different reward functions.}
\label{tab:bandit}
\end{table}

As shown in Table~\ref{tab:bandit}, the variance of the BQ estimates are lower than the variance of the MC estimates by an order of magnitude for the small sample size ($M=10$), and by 6 orders of magnitude for the large sample size ($M=100$). The BQ estimate is also more accurate than the MC estimate for the large sample size, and is roughly the same for the small sample size.
 

\subsection{Linear Quadratic Regulator}
\label{subsec:LQR}

In this section, we consider the following linear system in which the goal is to minimize the expected return over $20$ steps.\footnote{What we mean by reward and return in this section is in fact cost and loss, and this is why we are dealing with a minimization, and not a maximization, problem here. The reason for this is to maintain consistency in notations and definitions throughout the paper.} Thus, it is an episodic problem with paths of length $20$.

\begin{table}
\centering
\begin{tabular}{ll}
{\bf System:} & {\bf Policy:} \\
Initial State: $x_0\sim\N(0.3,0.001)$\hspace{1.1in} & Actions: $a_t\sim\mu(\cdot|x_t;\vectheta)=\N(\lambda x_t,\sigma^2)$ \\
Reward: $r_t=x_t^2+0.1a_t^2$ & Parameters: $\vectheta=(\lambda\;,\;\sigma)^\top$ \\
Transition: $x_{t+1}=x_t+a_t+n_x\;;\;n_x\sim\N(0,0.01)$ & \\
\end{tabular}
\end{table}

We run two sets of experiments on this system. We first fix the set of policy parameters and compare the BQ and MC estimates of the gradient of the expected return using the same sample. We then proceed to solving the complete policy gradient problem and compare the performance of the BPG algorithm (with both conventional and natural gradients) with a Monte-Carlo based policy gradient (MCPG) algorithm.


\subsubsection{Gradient Estimation}
\label{subsec:LQR_grad_est}

In this section, we compare the BQ and MC estimates of the gradient of the expected return for the policy induced by parameters $\lambda=-0.2$ and $\sigma=1$. We use several different sample sizes (number of paths used for gradient estimation) $M=5j\;,\;j=1,\ldots,20$ for the BQ and MC estimates. For each sample size, we compute the MC and BQ estimators using the same sample, repeat this process $10^4$ times, and then compute the average. The true gradient is analytically tractable and is used for comparison purposes. 

Figure~\ref{Fig1-LQR} shows the mean squared error (MSE) (left column) and the mean absolute angular error (right column) of the MC and BQ estimates of the gradient for several different sample sizes. The absolute angular error is the absolute value of the angle between the true and estimated gradients. In this figure, the BQ gradient estimates were calculated using Model 1 (top row) and Model 2 (bottom row) with sparsification. The error bars in the figures on the right column are the standard errors of the mean absolute angular errors.

\begin{figure}
\begin{center}
\includegraphics[height=5cm,width=1\textwidth]{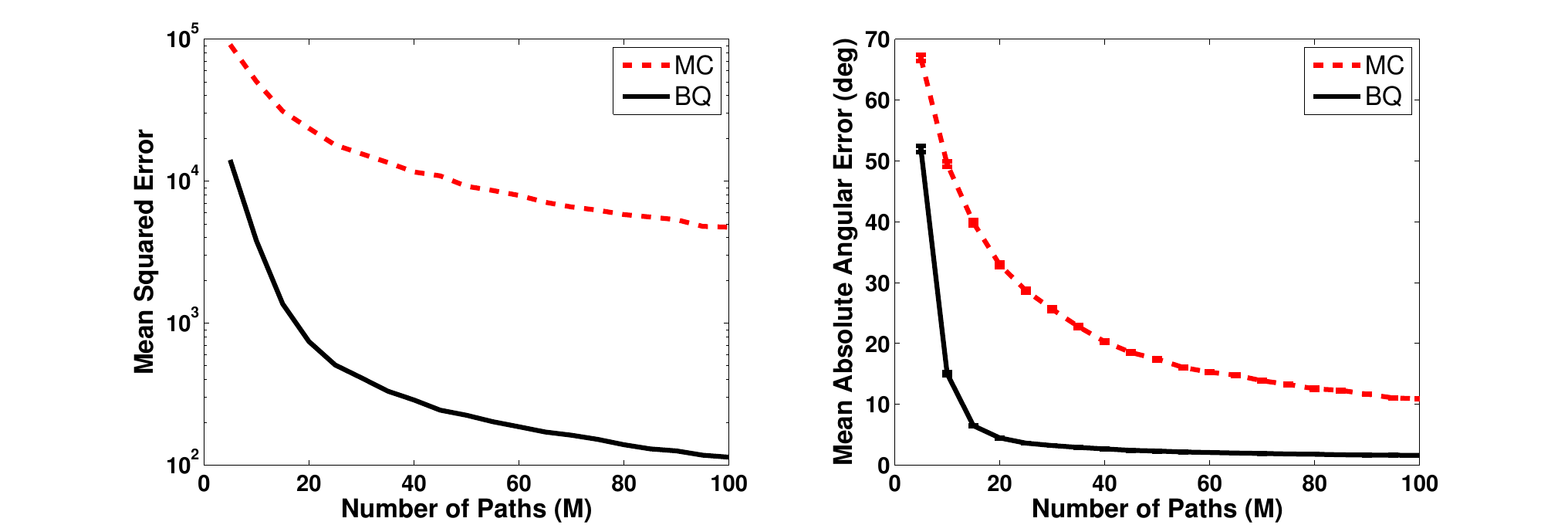} 
\includegraphics[height=5cm,width=1\textwidth]{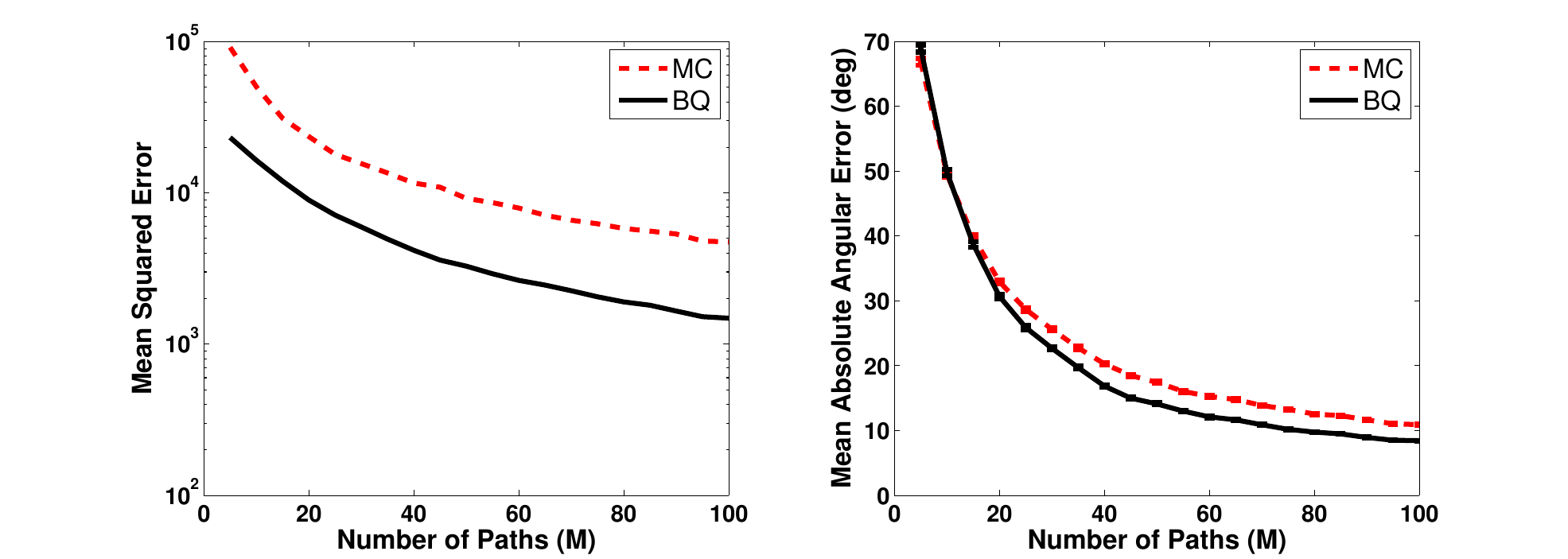} 
\caption{Results for the LQR problem using Model 1 (top row) and Model 2 (bottom row) with sparsification. Shown are the MSE (left column) and the mean absolute angular error (right column) of the MC and BQ estimates as a function of the number of sample paths $M$. All results are averaged over $10^4$ runs.} 
\label{Fig1-LQR}
\end{center}
\end{figure}

We ran another set of experiments in which we added i.i.d.~Gaussian noise to the rewards: \begin{small}$r_t=x_t^2+0.1a_t^2+n_r\;;\;n_r\sim\N(0,\sigma_r^2)$\end{small}. Note that in Models~1 and~2, $y(\xi)$, the noisy sample of $f(\xi)$, is of the form $R(\xi)\nabla\log\Pr(\xi;\vectheta)$ and $R(\xi)$, respectively (see Sections~\ref{subsec:M1} and~\ref{subsec:M2}). Moreover, since each reward $r_t$ is a Gaussian random variable with variance $\sigma_r^2$, the return $R(\xi)=\sum_{t=0}^{T-1}r_t$ is also a Gaussian random variable with variance $T\sigma_r^2$. Therefore in this case, the measurement noise covariance matrices for Models~1 and~2 may be written as \begin{small}$\matSigma=T\sigma_r^2\diag\Big(\big(\frac{\partial}{\partial\theta_i}\log p(\xi_1;\vectheta)\big)^2,\ldots,\big(\frac{\partial}{\partial\theta_i}\log p(\xi_M;\vectheta)\big)^2\Big)$\end{small} and $\matSigma=T\sigma_r^2\matI$, respectively, where $T=20$ is the path length.\footnote{In Model~1, $\matSigma$ is the measurement noise covariance matrix for the $i$th component of the gradient $\frac{\partial}{\partial\theta_i}\eta_B(\vectheta)$. Note that $\frac{\partial}{\partial\theta_i}\log p(\xi_j;\vectheta)$ depends only on the policy and can be calculated using Equation~\ref{eq:score}.} We tried two different Gaussian reward noise standard deviations: $\sigma_r=0.1\;\text{and}\;1$ in our experiments. Adding noise to the rewards slightly increased the error of the BQ and MC estimates of the gradient. However, the graphs comparing these estimates remained quite similar to those shown in Figure~\ref{Fig1-LQR}. Hence in Figure~\ref{Fig2-LQR}, we compare the MSE (left column) and the mean absolute angular error (right column) of the BQ estimates with and without noise in the rewards as a function of the number of sample paths $M$. In this figure, the noise in the rewards has variance $\sigma_r^2=1$, and the BQ gradient estimates were calculated using Model 1 (top row) and Model 2 (bottom row) with sparsification.

\begin{figure}[!h]
\begin{center}
\includegraphics[height=5cm,width=1\textwidth]{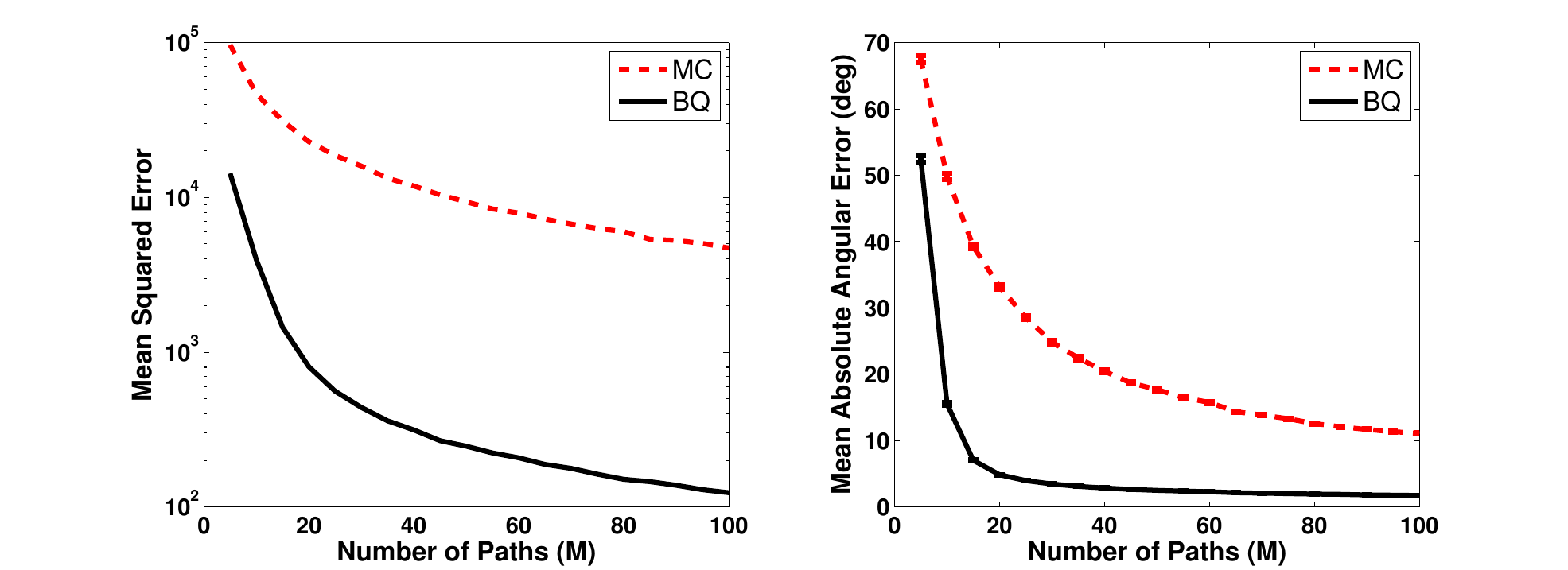} 
\includegraphics[height=5cm,width=1\textwidth]{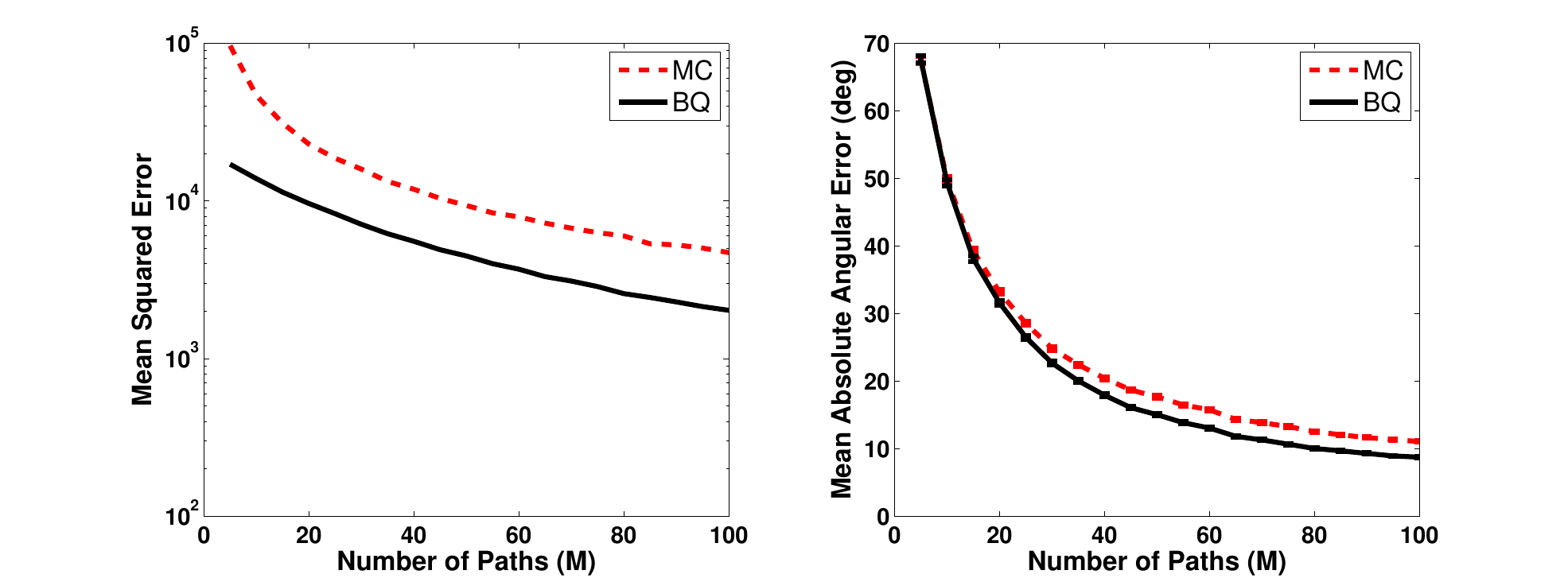} 
\caption{Results for the LQR problem in which the rewards are corrupted by i.i.d. Gaussian noise with $\sigma_r^2=1$. Shown are the MSE (left column) and the mean absolute angular error (right column) of the BQ estimates with and without noise in the rewards as a function of the number of sample paths $M$. The BQ gradient estimates were calculated using Model 1 (top row) and Model 2 (bottom row) with sparsification. All results are averaged over $10^4$ runs.} 
\label{Fig2-LQR}
\end{center}
\end{figure}


\subsubsection{Policy Optimization}
\label{subsec:LQR_policy_opt}

In this section, we use Bayesian policy gradient (BPG) to optimize the policy parameters in the LQR problem. Figure~\ref{Fig3-LQR} shows the performance of the BPG algorithm with the conventional 
(BPG) and natural (BPNG) gradient estimates, versus a MC-based policy gradient (MCPG) algorithm, for sample sizes (number of sample paths used to estimate the gradient of each policy) $M=5$, $10$, $20$, and $40$. We use Algorithm~\ref{alg:BPG} with the number of updates set to $N=100$, and Model 1 with sparsification for the BPG and BPNG methods. Since Algorithm~\ref{alg:BPG} computes the Fisher information matrix for each set of policy parameters, the estimate of the natural gradient is provided at little extra cost at each step. The returns obtained by these methods are averaged over $10^4$ runs. The policy parameters are initialized randomly at each run. In order to ensure that the learned parameters do not exceed an acceptable range, the policy parameters are defined as 
$\lambda=-1.999+1.998/(1+e^{\kappa_1})$ and $\sigma=0.001+1/(1+e^{\kappa_2})$. The optimal solution is $\lambda^*\approx-0.92,\;\sigma^*=0.001,\;\eta_B(\lambda^*,\sigma^*)=0.3067$, corresponding to $\kappa^*_1\approx-0.16$ and $\kappa^*_2\rightarrow\infty$.  

\begin{figure}[!h]
\begin{center}
\includegraphics[width=0.49\textwidth]{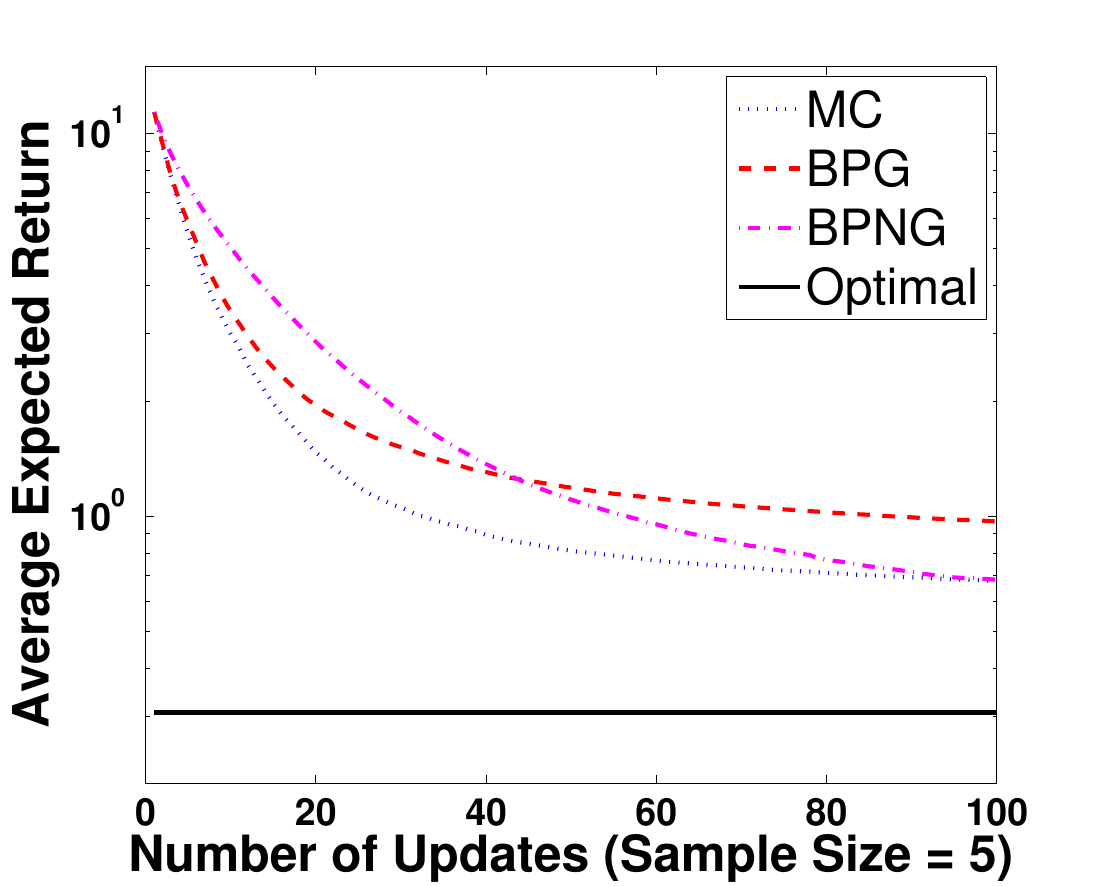}
\hfill
\includegraphics[width=0.49\textwidth]{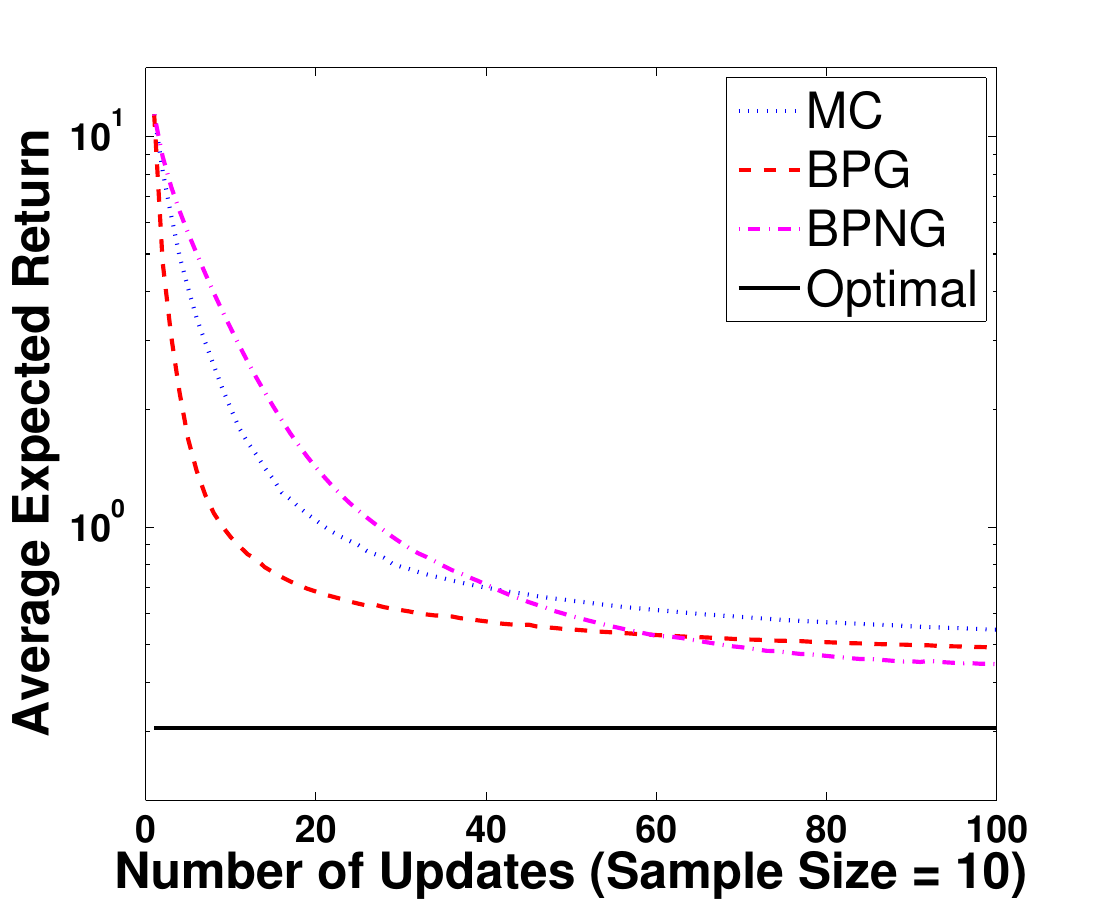} \\
\includegraphics[width=0.49\textwidth]{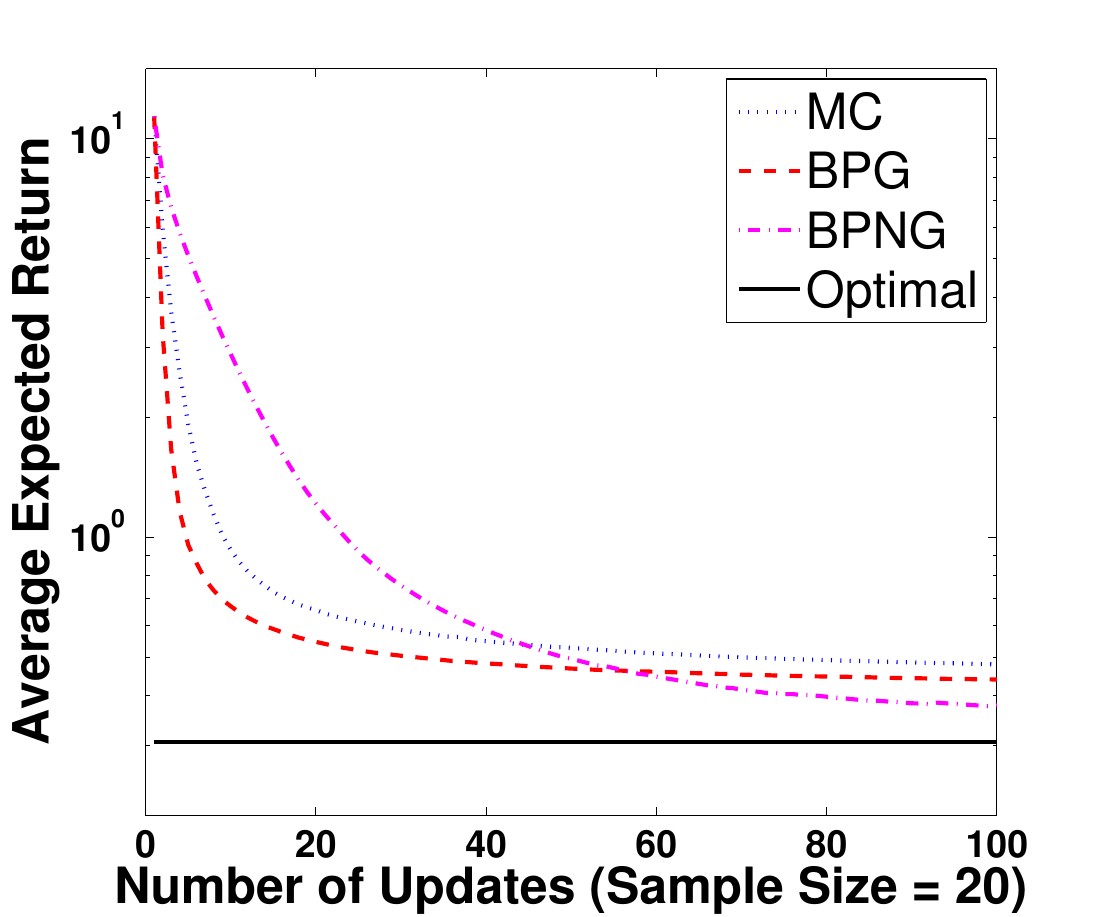}
\hfill
\includegraphics[width=0.49\textwidth]{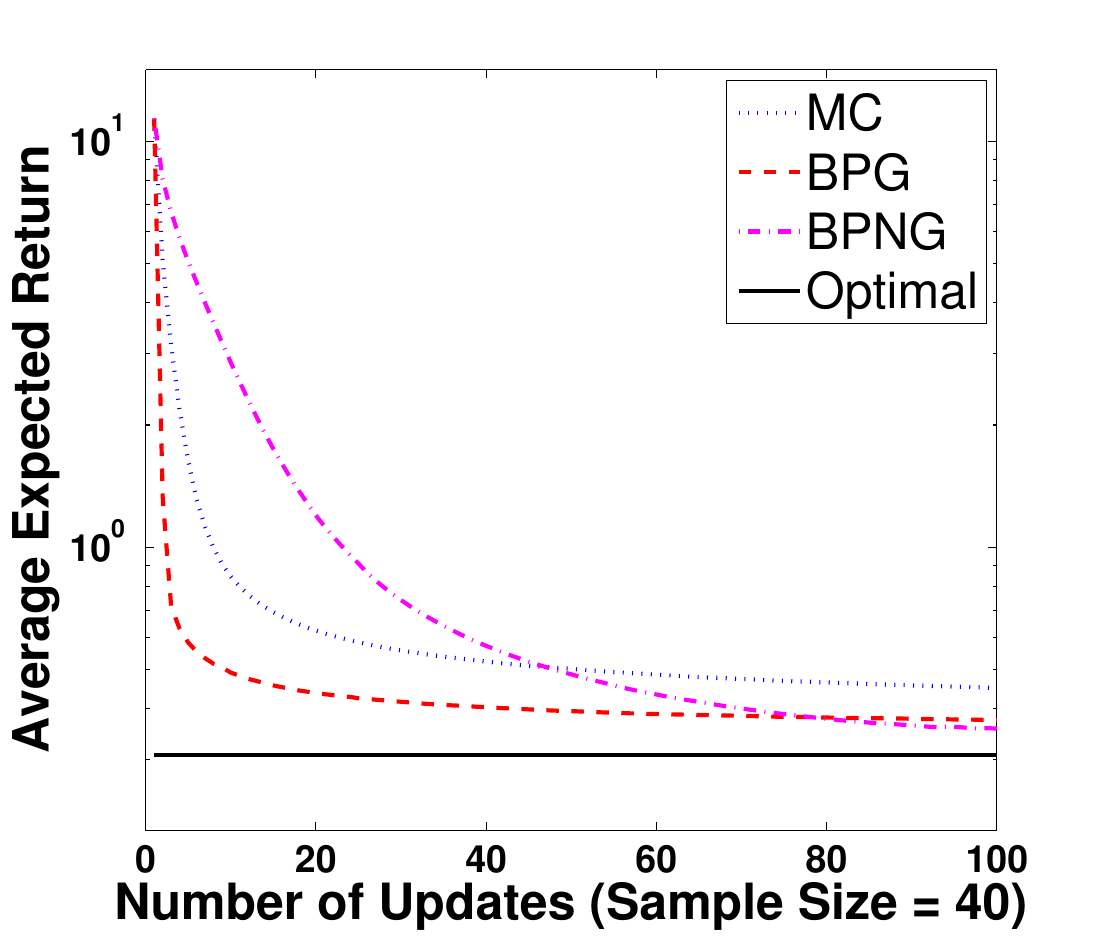}
\caption{A comparison of the average expected returns of the Bayesian policy gradient algorithm using conventional (BPG) and natural (BPNG) gradient estimates, with the average expected return of a MC-based policy gradient algorithm (MCPG) for sample sizes $M=5$, $10$, $20$, and $40$. All results are averaged over $10^4$ runs.}        
\label{Fig3-LQR}
\end{center}
\end{figure}

Figure~\ref{Fig3-LQR} shows that the MCPG algorithm performs better than BPG and BPNG only for the smallest sample size ($M=5$), whereas for larger samples BPG and BPNG dominate MCPG. The better performance of MCPG for very small sample size is due to the fact that in this case, the Bayesian estimators, BPG and BPNG, like any other Bayesian estimator or posterior in such case, rely more on the prior, and thus, are not accurate if the prior is not very informative. A similar phenomenon was also reported by~\citet{Rasmussen03BM}. We used two different learning rates for the two components of the gradient. For a fixed sample size, BPG and MCPG methods start with an initial learning rate and decrease it according to the schedule $\beta_j=\beta_0\big(20/(20+j)\big)$. The BPNG algorithm uses a fixed learning rate multiplied by the determinant of the Fisher information matrix. We tried many values for the initial learning rates used by these algorithms and those in Table~\ref{tab:LQR} yielded the best performance of those we tried. 

\begin{table}[!h]
\centering
\begin{tabular}{|c|c|c|c|c|} \hline
$\beta_0$ & $M=5$ & $M=10$ & $M=20$ & $M=40$ \\ \hline
MCPG & $0.01,0.05$ & $0.05,0.05$ & $0.05,0.10$ & $0.05,0.10$ \\ \hline
BPG & $0.01,0.05$ & $0.07,0.10$ & $0.15,0.15$ & $0.10,0.30$ \\ \hline
BPNG & $0.010,0.005$ & $0.010,0.005$ & $0.015,0.005$ & $0.015,0.005$ \\ \hline
BPG-var & $0.05,0.05$ & $0.10,0.10$ & $0.10,0.15$ & $0.15,0.30$ \\ \hline
\end{tabular}
\caption{Initial learning rates $\beta_0$ used by the policy gradient algorithms for the two components of the gradient.}
\label{tab:LQR}
\end{table}

So far we have assumed that the Fisher information matrix is known. In the next experiment, we estimate it using both MC and a model-based maximum likelihood (ML) method, as discussed in Section~\ref{subsec:BPG_eval_alg}. In the ML approach, we model the transition probability function as $P(x_{t+1}|x_t,a_t)=\N(c_1x_t+c_2a_t+c_3,c_4^2)$, and then estimate its parameters $(c_1,c_2,c_3,c_4)$ using observing state transitions. Figure~\ref{Fig4-LQR} shows that the BPG algorithm, when the Fisher information matrix is estimated using ML and MC, still performs better than MCPG. Top and bottom rows contain the results for the BPG algorithm with conventional (BPG-ML and BPG-MC) and natural (BPNG-ML and BPNG-MC) gradient estimates, respectively. Although the BPG-ML (BPNG-ML) outperforms BPG-MC (BPNG-MC) for small sample sizes, the difference in their performance disappears as we increase the sample size. One reason for the good performance of BPG-ML is that the form of the state transition function $P(x_{t+1}|x_t,a_t)$ has been selected correctly. Here we used the same initial learning rates and learning rate schedules as in the experiments of Figure~\ref{Fig3-LQR} (see Table~\ref{tab:LQR}).   

\begin{figure}[!h]
\begin{center}
\includegraphics[height=5.3cm,width=0.49\textwidth]{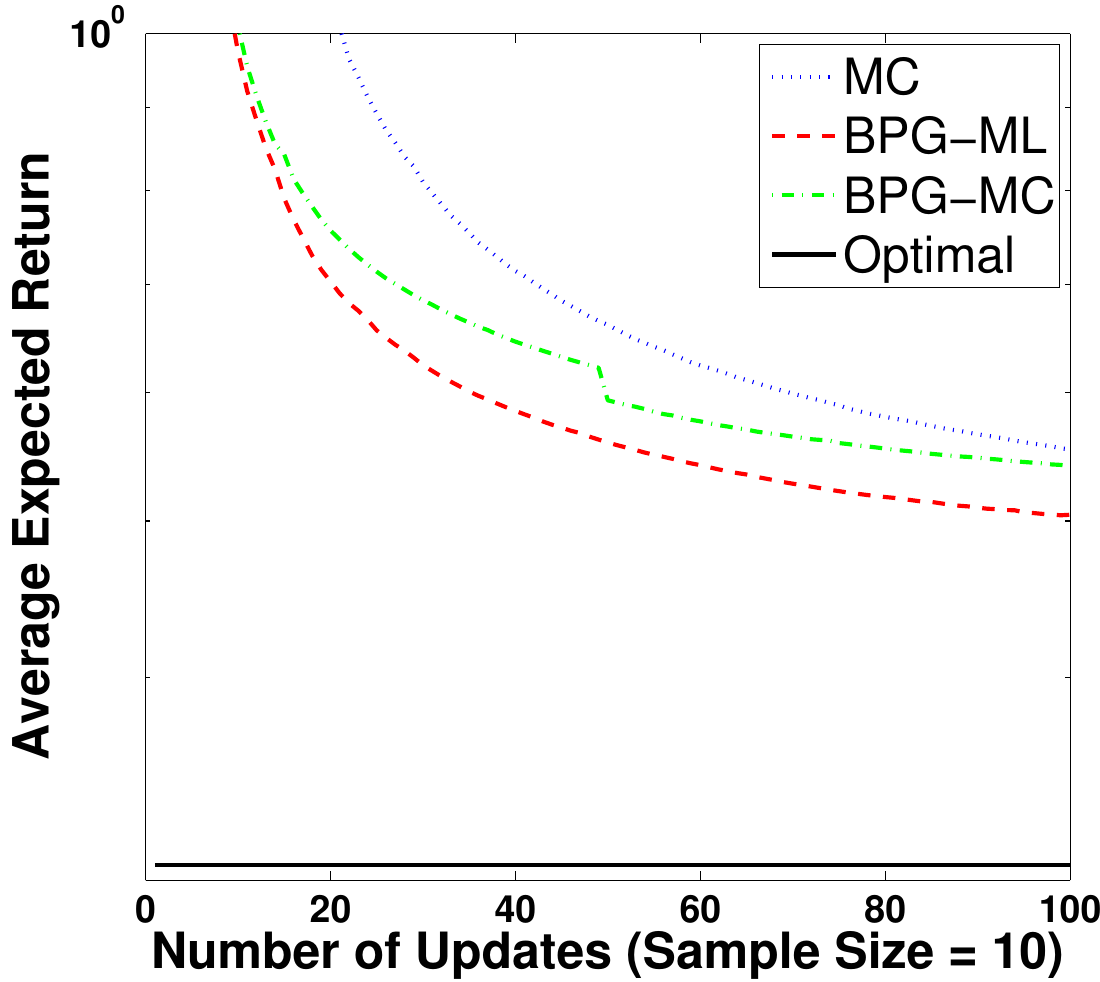}
\hfill
\includegraphics[height=5.3cm,width=0.49\textwidth]{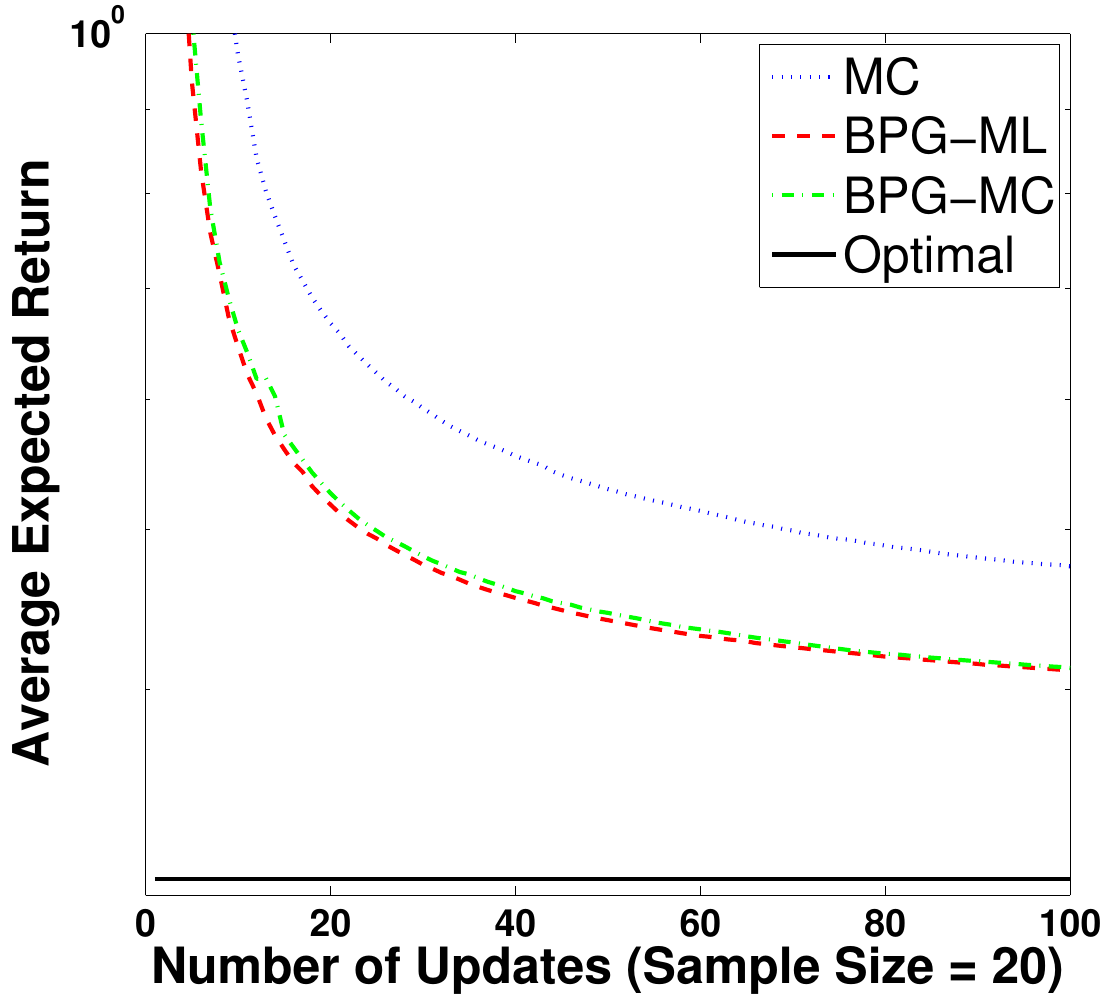} \\
\includegraphics[height=4.7cm,width=0.325\textwidth]{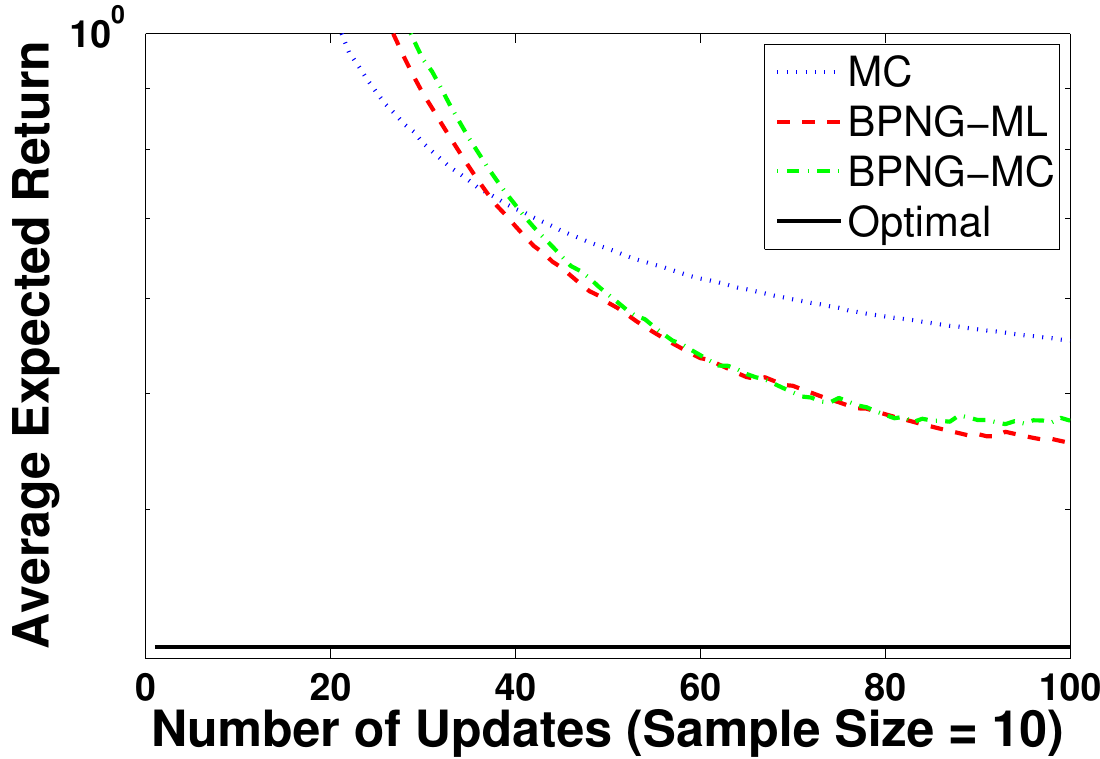}
\includegraphics[height=4.7cm,width=0.325\textwidth]{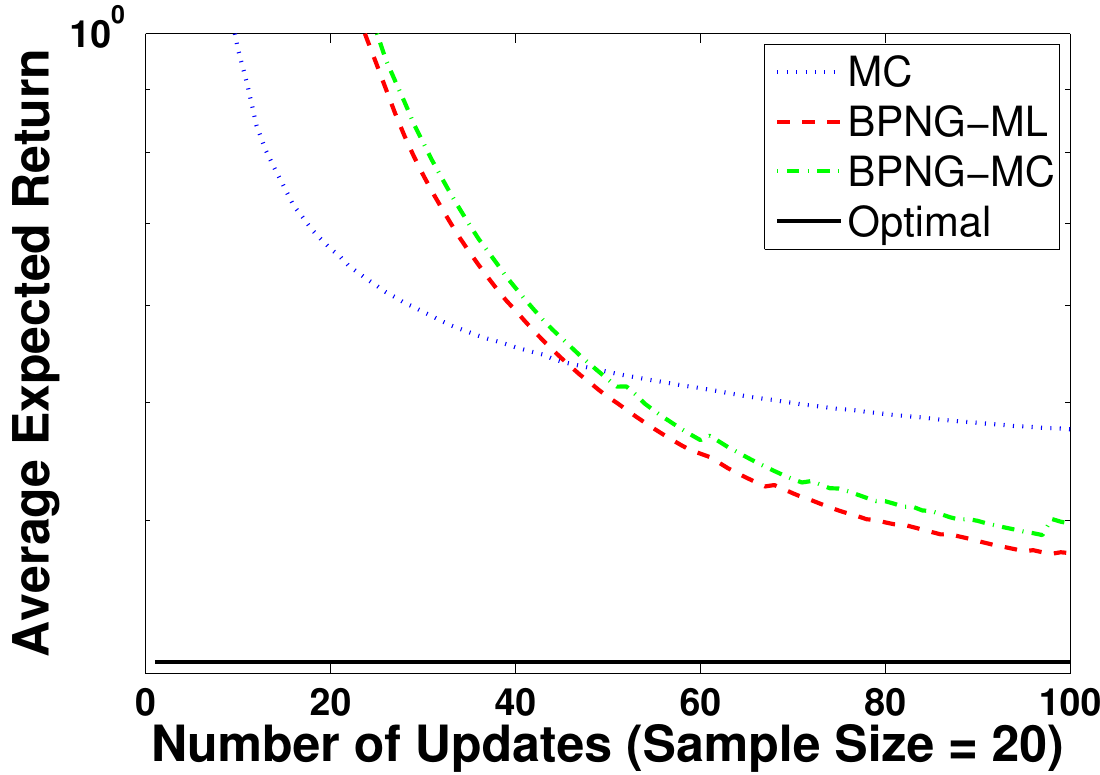}
\includegraphics[height=4.7cm,width=0.325\textwidth]{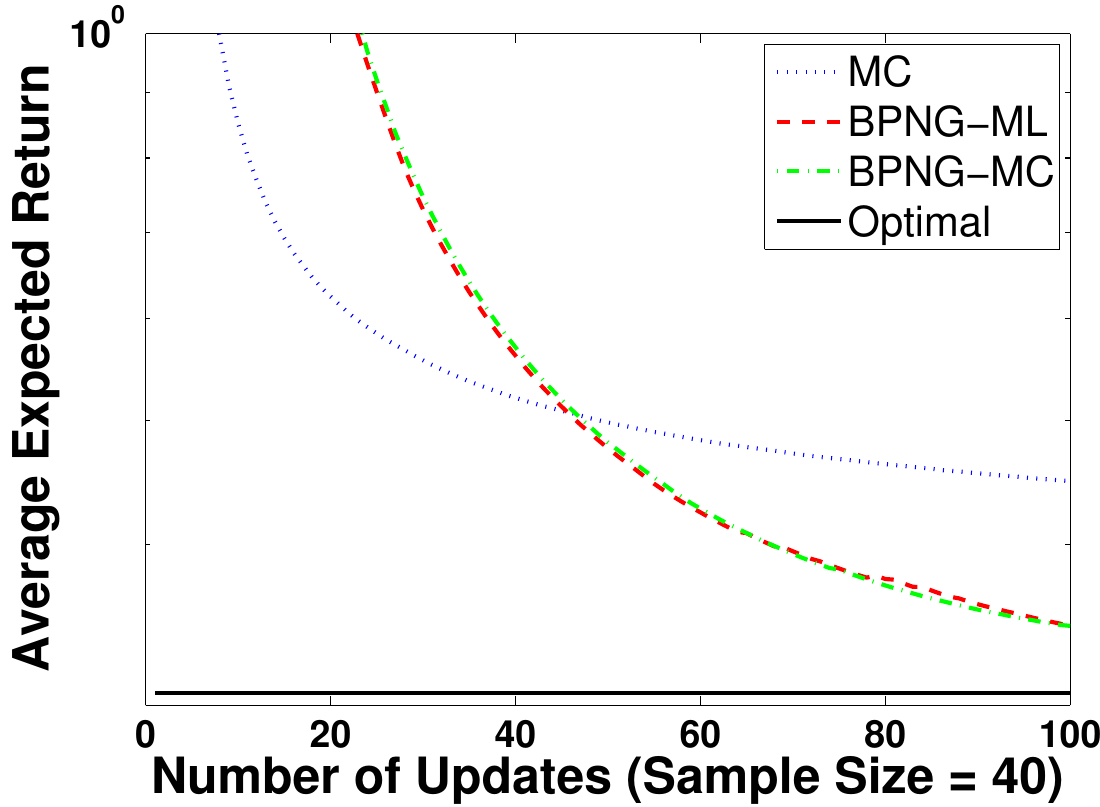}
\caption{A comparison of the average expected returns of the BPG algorithm, when the Fisher information matrix is estimated using ML and MC, with the average expected return of a MC-based policy gradient algorithm (MCPG). The top and bottom rows contain the results for the BPG algorithm with conventional (BPG-ML and BPG-MC) and natural (BPNG-ML and BPNG-MC) gradient estimates, respectively. All results are averaged over $10^4$ runs.}        
\label{Fig4-LQR}
\end{center}
\end{figure}

Although the proposed Bayesian policy gradient algorithm (Algorithm~\ref{alg:BPG}) uses only the posterior mean of the gradient in its updates, it can be extended to make judicious use of the second moment information provided by the Bayesian policy gradient estimation algorithm (Algorithm~\ref{alg:BPG-Eval}). In the last experiment of this section, we use the posterior covariance of the gradient, provided by Algorithm~\ref{alg:BPG-Eval}, to select the learning rate and the direction of the updates in Algorithm~\ref{alg:BPG}. The idea is to use a small learning rate when the variance of the gradient estimate is large, and to have a large update when it is small. We refer to the resulting algorithm by the name BPG-var. This algorithm uses a fixed learning rate parameter (see Table~\ref{tab:LQR}) multiplied by $\Big[\big(1+n\big)\matI - \Cov\big(\nabla\eta_B(\vectheta)|\D_M\big)\Big]/(1+n)$ in its updates. Note that $n+1$ is $b_0$ in the calculation of the posterior covariance of the gradient in Model 1 (see Proposition~\ref{prop:3}), and is used here as an upper bound for the posterior covariance of the gradient. Figure~\ref{Fig5-LQR} compares the average expected return of BPG-var with BPG and MCPG for sample sizes $M=5$, $10$, $20$, and $40$. The figure shows that BPG-var performs better than BPG and MCPG for all the sample sizes. It even has a better performance than MCPG for the smallest sample size ($M=5$). Comparing Figures~\ref{Fig3-LQR} and~\ref{Fig5-LQR} shows that BPG-var converges faster than BPNG and has similar final performance. As we expected, BPG-var and BPG perform more and more alike as we increase the sample size. This is because by increasing the sample size the estimated gradient (the posterior mean of the gradient), and as a result, the update direction used by BPG becomes more reliable.  

\begin{figure}
\begin{center}
\includegraphics[height=5.3cm,width=0.49\textwidth]{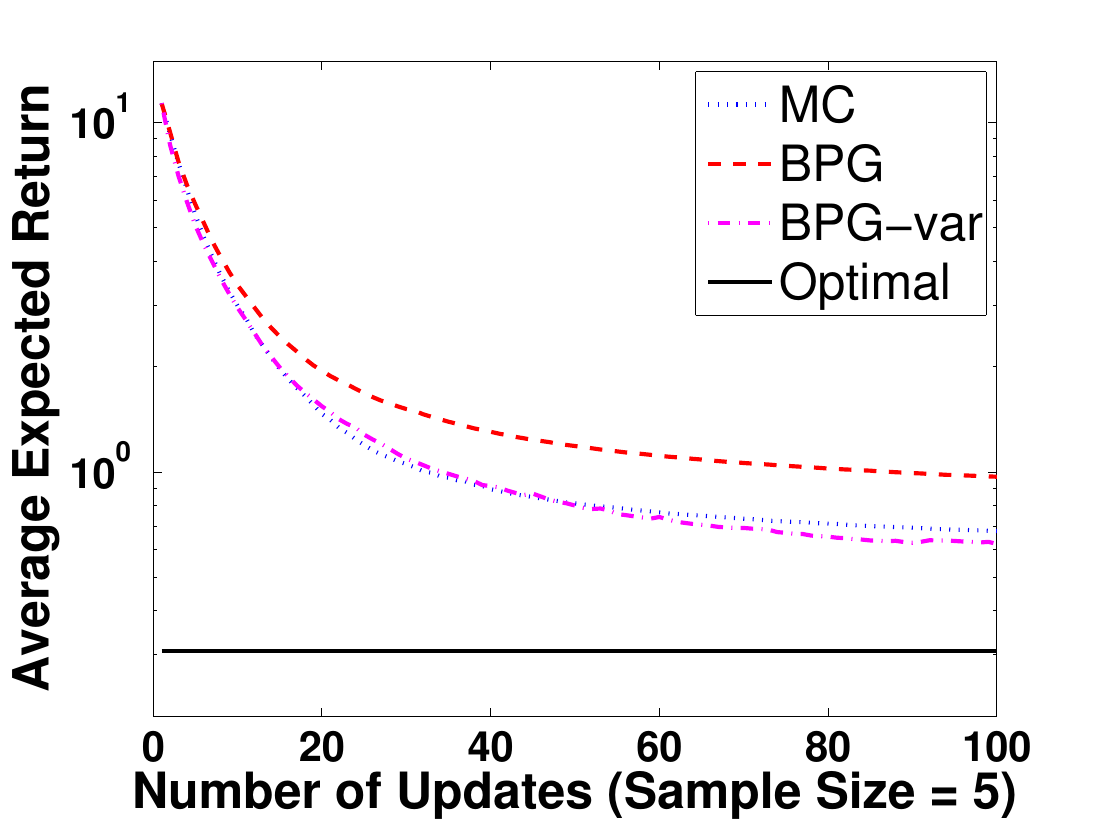}
\hfill
\includegraphics[height=5.3cm,width=0.49\textwidth]{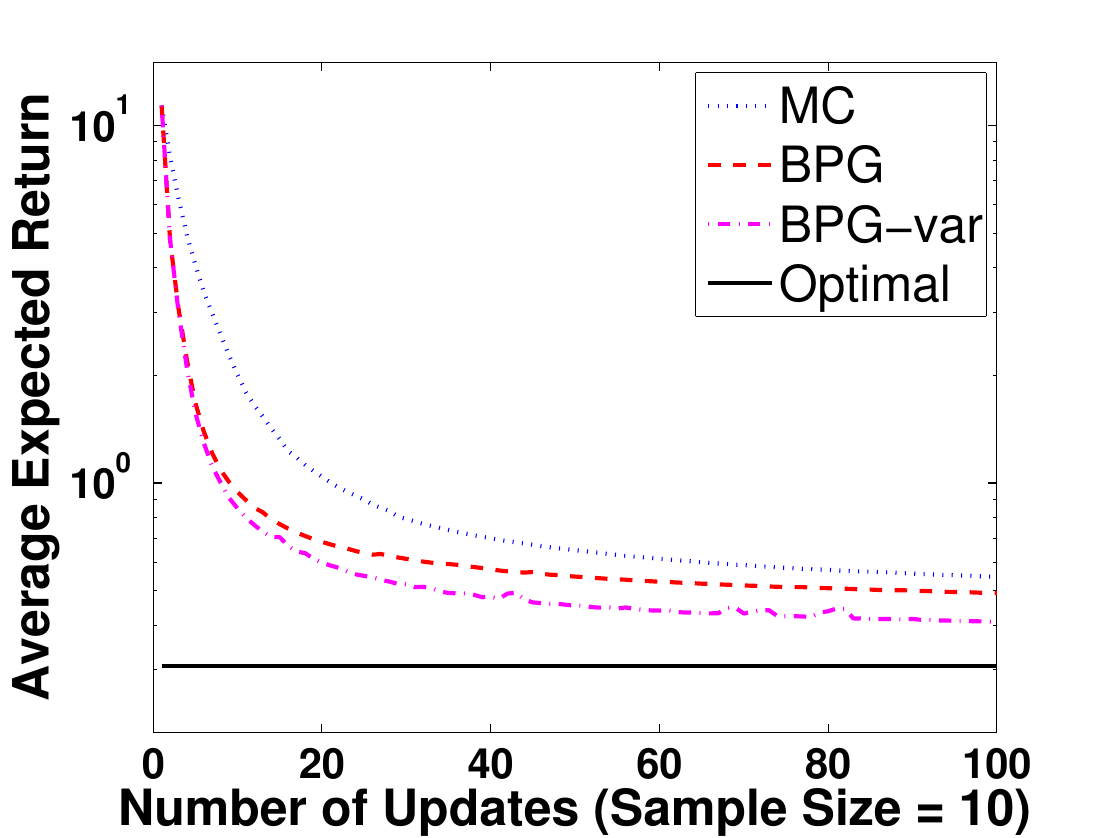} \\
\includegraphics[height=5.3cm,width=0.49\textwidth]{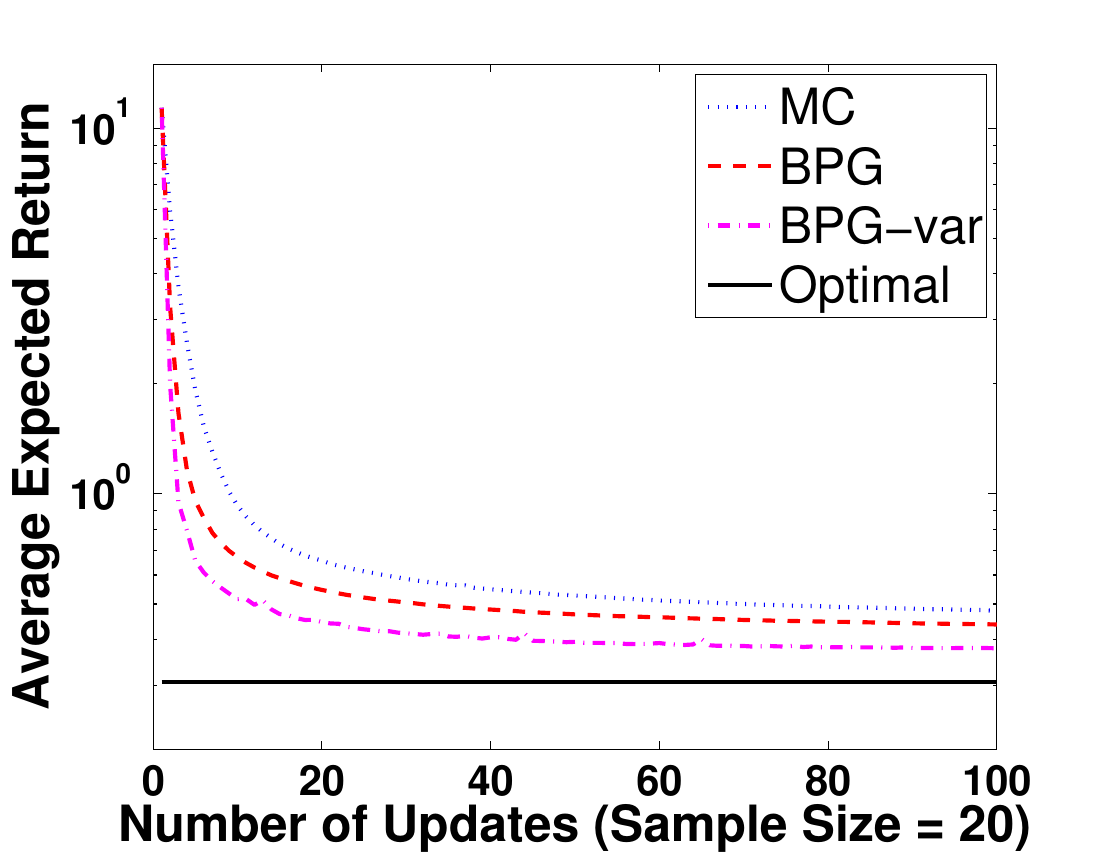}
\hfill
\includegraphics[height=5.3cm,width=0.49\textwidth]{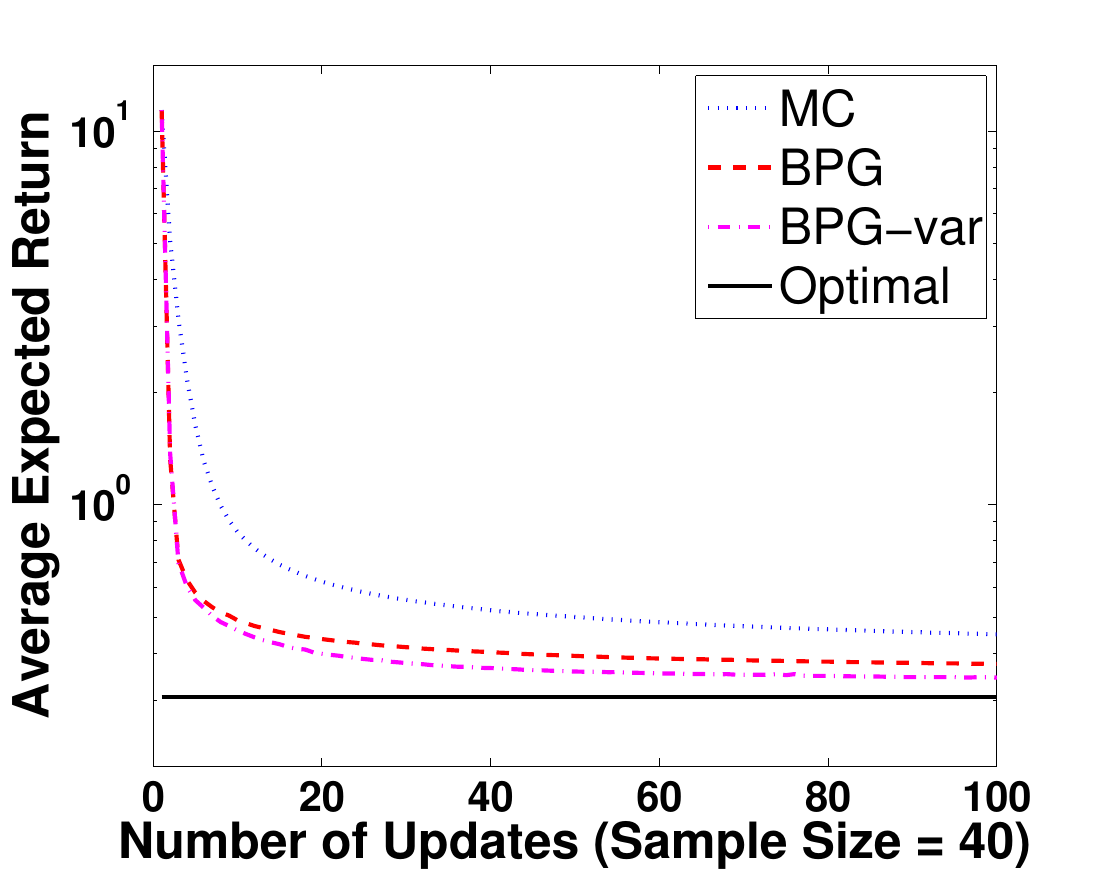}
\caption{A comparison of the average expected returns of the BPG algorithm that uses the posterior covariance in its updates (BPG-var), with the average expected return of the BPG and a MC-based policy gradient algorithm (MCPG) for sample sizes $M=5$, $10$, $20$, and $40$. All results are averaged over $10^4$ runs.}        
\label{Fig5-LQR}
\end{center}
\end{figure}

In an approach similar to the one used in the experiments of Figure~\ref{Fig5-LQR},~\citet{Vien11HM} used BQ to estimate the Hessian matrix distribution, and then used its mean as learning rate schedule to improve the performance of BPG. They empirically showed that their method performs better than BPG and BPNG in terms of convergence speed.


\section{Bayesian Actor-Critic}
\label{sec:BAC}

The models and algorithms of Section~\ref{sec:BPG} consider complete trajectories as the basic observable unit, and thus, do not require the dynamics within each trajectory to be of any special form. In particular, it is not necessary for the dynamics to have the Markov property, allowing the resulting algorithms to handle partially observable MDPs, Markov games, and other non-Markovian systems. On the down side, these algorithms cannot take advantage of the Markov property when operating in Markovian systems. Moreover, since the unit of observation of these algorithms is the entire trajectory, their gradient estimates have larger variance than the algorithms that will be discussed in this section, whose unit of observation is {\em (current state, action, next state)}, since they take advantage of the Markov property, especially when the size of the trajectories is large. 

In this section, we apply the Bayesian quadrature idea to the policy gradient expression given by Equation~\ref{eq:policy_grad}, i.e., 
\begin{equation*}
\nabla\eta(\vectheta)=\int dxda\;\nu(x;\vectheta)\nabla\mu(a|x;\vectheta)Q(x,a;\vectheta),
\end{equation*}
and derive a family of Bayesian actor-critic (BAC) algorithms. In this approach, we place a Gaussian process (GP) prior over action-value functions using a prior covariance kernel defined on state-action pairs: $k(\vecz,\vecz')=\Cov\big[Q(\vecz),Q(\vecz')\big]$. We then compute the GP posterior conditioned on the sequence of individual observed transitions. In the same vein as Section~\ref{sec:BPG}, by an appropriate choice of a prior on action-value functions, we are able to derive closed-form expressions for the posterior moments of $\nabla\eta(\vectheta)$. The main questions here are: {\bf 1)} how to compute the GP posterior of the action-value function given a sequence of observed transitions? and {\bf 2)} how to choose a prior for the action-value function that allows us to derive closed-form expressions for the posterior moments of $\nabla\eta(\vectheta)$? Fortunately, well developed machinery for computing the posterior moments of $Q(\vecz)$ is provided in a series of papers by~\citet{Engel03BM,Engel05RL} (for a thorough treatment see~\citealp{Engel05AR}). In the next two sections, we will first briefly review some of the main results pertaining to the Gaussian process temporal difference (GPTD) model proposed in~\citet{Engel05RL}, and then will show how they may be combined with the Bayesian quadrature idea in developing a family of Bayesian actor-critic algorithms. 


\subsection{Gaussian Process Temporal Difference Learning}
\label{subsec:GPTD}

The Gaussian process temporal difference (GPTD) learning~\citep{Engel03BM,Engel05RL} model is based on a statistical generative model relating the observed reward signal $r$ to the unobserved action-value function $Q$ 
\begin{equation}\label{eq:gen_model}
r(\vecz_i) = Q(\vecz_i) - \gamma Q(\vecz_{i+1}) + N(\vecz_i,\vecz_{i+1}),
\end{equation}
where $N(\vecz_i,\vecz_{i+1})$ is a zero-mean noise signal that accounts for the discrepancy between $r(\vecz_i)$ and $Q(\vecz_i)-\gamma Q(\vecz_{i+1})$. Let us define the finite-dimensional processes $\vecr_t$, $Q_t$, $N_t$, and the $t \times (t+1)$ matrix $\matH_t$ as follows:
\begin{align} \label{eq:RQN}
\vecr_t &= \big(r(\vecz_0),\ldots,r(\vecz_t)\big)^\top,\hspace{1in} 
Q_t = \big(Q(\vecz_0),\ldots,Q(\vecz_t)\big)^\top, \nonumber\\
N_t &= \big(N(\vecz_0,\vecz_1),\ldots,N(\vecz_{t-1},\vecz_t)\big)^\top,
\end{align}
\begin{equation} \label{eq:Ht}
\matH_t =
\left[
\begin{array}{ccccc}
1 & -\gamma & 0 & \ldots & 0 \\
0 & 1 & -\gamma & \ldots & 0 \\
\vdots & & & & \vdots \\
0 & 0 & \ldots & 1 & -\gamma
\end{array}
\right] .
\end{equation}
The set of Equations~\ref{eq:gen_model} for $i=0,\ldots,t-1$ may be written as $\vecr_{t-1} = \matH_{t} Q_{t} + N_t$. Under certain assumptions on the distribution of the discounted return random process~\citep{Engel05RL}, the covariance of the noise vector $N_t$ is given by  
\begin{equation} \label{eq:Sigma}
\matSigma_t = \sigma^2 \matH_t \matH_t^\top= \sigma^2
\begin{bmatrix}
1+ \gamma^2 & -\gamma & 0 & \ldots & 0 \\
-\gamma & 1 + \gamma^2 & -\gamma & \ldots & 0 \\
\vdots & \vdots & & & \vdots \\
0 & 0 & \ldots & -\gamma & 1 + \gamma^2
\end{bmatrix}.
\end{equation}
In episodic tasks, if $\vecz_{t-1}$ is the last state-action pair in the episode (i.e., $\vecx_t$ is a zero-reward absorbing terminal state), $\matH_t$ becomes a square $t \times t$ invertible matrix of the form shown in Equation~\ref{eq:Ht} with its last column removed. The effect on the noise covariance matrix $\matSigma_t$ is that the bottom-right element becomes $1$ instead of $1+\gamma^2$.

Placing a GP prior on $Q$ and assuming that $N_t$ is also normally distributed, we may use Bayes' rule to obtain the posterior moments of $Q$:
\begin{align} \label{eq:full_post}
\hat{Q}_t(\vecz) &= \exptE\left[ Q(\vecz)|\D_t \right] = \veck_t(\vecz)^\top \vecalpha_t ,\nonumber\\
\hat{S}_t(\vecz,\vecz') &= \Cov \left[Q(\vecz),Q(\vecz')|\D_t \right]= k(\vecz,\vecz')-\veck_t(\vecz)^\top \matC_t \veck_t(\vecz'),
\end{align}
where $\D_t$ denotes the observed data up to and including time step $t$. We used here the following definitions:
\begin{align} \label{eq:alpha_C}
&\veck_t(\vecz)=\big( k(\vecz_0,\vecz),\ldots, k(\vecz_t,\vecz)\big)^\top,\hspace{0.775in}
\matK_t =\big[\veck_t(\vecz_0), \veck_t(\vecz_1), \ldots, \veck_t(\vecz_t)\big],\nonumber\\
&\vecalpha_t=\matH_t^\top \left(\matH_t\matK_t\matH_t^\top+\matSigma_t \right)^{-1}\vecr_{t-1}\;,\hspace{0.55in}
\matC_t =\matH_t^\top \left(\matH_t \matK_t\matH_t^\top+\matSigma_t\right)^{-1}\matH_t\;.
\end{align}
Note that $\hat{Q}_t(\vecz)$ and $\hat{S}_t(\vecz,\vecz')$ are the posterior mean and covariance functions of the posterior GP, respectively. As more samples are observed, the posterior covariance decreases, reflecting a growing confidence in the Q-function estimate $\hat{Q}_t$. 


\subsection{A Family of Bayesian Actor-Critic Algorithms}
\label{subsec:BAC-Alg}

We are now in a position to describe the main idea behind our BAC approach. Making use of the linearity of Equation~\ref{eq:policy_grad} in $Q$ and denoting $\vecg(\vecz;\vectheta)=\pi^\mu(\vecz)\nabla\log\mu(\veca|\vecx;\vectheta)$, we obtain the following expressions for the posterior moments of the policy gradient~\citep{Ohagan91BQ}:
\begin{align}
\label{eq:grad_eta_post}
\exptE[\nabla\eta(\vectheta)|\D_t]&=\int_{\Z}d\vecz\vecg(\vecz;\vectheta)\hat{Q}_t(\vecz;\vectheta),\nonumber\\
\Cov\left[\nabla\eta(\vectheta)|\D_t \right]&=\int_{\Z^2}d\vecz d\vecz'\vecg(\vecz;\vectheta)\hat{S}_t(\vecz,\vecz')\vecg(\vecz';\vectheta)^\top. 
\end{align}
Substituting the expressions for the posterior moments of $Q$ from Equation~\ref{eq:full_post} into Equation~\ref{eq:grad_eta_post}, we obtain
\begin{align*}
\exptE[\nabla\eta(\vectheta)|\D_t]&=\int_{\Z} d\vecz\vecg(\vecz;\vectheta)\veck_t(\vecz)^\top \vecalpha_t ,\nonumber \\
\Cov \left[\nabla\eta(\vectheta)|\D_t \right]&=\int_{\Z^2}d\vecz d\vecz'\vecg(\vecz;\vectheta)\left(k(\vecz,\vecz')-\veck_t(\vecz)^\top\matC_t\veck_t(\vecz')\right)\vecg(\vecz';\vectheta)^\top.
\end{align*}
These equations provide us with the general form of the posterior policy gradient moments. We are now left with a computational issue, namely, how to compute the integrals appearing in these expressions? We need to be able to evaluate the following integrals:
\begin{equation} \label{eq:UV_def}
\matB_t=\int_{\Z} d\vecz\vecg(\vecz;\vectheta)\veck_t(\vecz)^\top,\hspace{0.5in}\matB_0=\int_{\Z^2} d\vecz d\vecz' \vecg(\vecz;\vectheta)k(\vecz,\vecz')\vecg(\vecz';\vectheta)^\top.
\end{equation}
Using these definitions, we may write the gradient posterior moments compactly as
\begin{equation} \label{eq:post_compact}
\exptE[\nabla\eta(\vectheta)|\D_t]=\matB_t\vecalpha_t\;,\hspace{0.5in}\Cov\left[\nabla\eta(\vectheta)|\D_t \right]=\matB_0-\matB_t\matC_t\matB_t^\top.
\end{equation}

In order to render these integrals analytically tractable, we choose our prior covariance kernel to be the sum of an arbitrary state-kernel $k_x$ and the (invariant) Fisher kernel $k_F$ between state-action pairs (see e.g.,~\citealp{ShaweTaylor04KM}, Chapter~12). The (policy dependent) Fisher information kernel and our overall state-action kernel are then given by
\begin{equation}
k_F(\vecz,\vecz')=\vecu(\vecz;\vectheta)^\top \matG(\vectheta)^{-1}\vecu(\vecz')\;,\hspace{0.5in}k(\vecz,\vecz')=k_x(\vecx,\vecx') + k_F(\vecz,\vecz')\;,
\label{eq:kernels}
\end{equation}
where $\vecu(\vecz;\vectheta)$ and $\matG(\vectheta)$ are the score function and Fisher information matrix defined as\footnote{Similar to $\vecu(\xi)$ and $\matG$ defined by Equations~\ref{eq:score} and~\ref{Fisher:Matrix}, to simplify the notation, we omit the dependence of $\vecu$ and $\matG$ to the policy parameters $\vectheta$, and replace $\vecu(\vecz;\vectheta)$ and $\matG(\vectheta)$ with $\vecu(\vecz)$ and $\matG$ in the sequel.}
\begin{align}
\label{eq:score-state-action}
\vecu(\vecz;\vectheta)&=\nabla\log\mu(a|x;\vectheta), \\
\label{eq:Fisher-state-action}
\matG(\vectheta)&=\exptE_{x\sim \nu^\mu,a\sim\mu}\left[\nabla\log\mu(a|x;\vectheta)\nabla\log\mu(a|x;\vectheta)^\top\right]=\exptE_{\vecz\sim\pi^\mu}\left[\vecu(\vecz;\vectheta)\vecu(\vecz;\vectheta)^\top\right].
\end{align}
Although here we have total flexibility in selecting the state kernel, we are restricted to the Fisher kernel for state-action pairs. This restriction may cause an error in approximating some action-value functions $Q$. This error depends on the problem at hand and is hard to quantify. This is exactly the same as selecting an inaccurate prior in any Bayesian algorithm or choosing a wrong representation (function space) in any machine learning algorithm (referred to as {\em approximation error} in the approximation theory). However, this restriction did not cause a significant error in our experiments (see Section~\ref{sec:BAC-experiments}), as in almost all of them, the gradients estimated by BAC were more accurate than those estimated by the MC-based method, given the same number of samples.

Note that in Sections~\ref{sec:BPG} to~\ref{sec:BPG-experiments} we used a formulation in which the observable unit is a system trajectory, and thus, the expected return and its gradient are defined by Equations~\ref{exp-ret} and~\ref{eq:grad}. In this formulation, the score function and Fisher information matrix are defined by Equations~\ref{eq:score} and~\ref{Fisher:Matrix}. However, in the formulation used in this section and in the rest of the paper, where the observable unit is an individual state-action-reward transition, the expected return and its gradient are defined by Equations~\ref{eq:exp-ret2} and~\ref{eq:policy_grad}. In this formulation, the score function and Fisher information matrix are defined by Equations~\ref{eq:score-state-action} and~\ref{eq:Fisher-state-action}, respectively.

A nice property of the Fisher kernel is that while $k_F(\vecz,\vecz')$ depends on the policy, it is invariant to policy reparameterization. In other words, it only depends on the actual probability mass assigned to each action and not on its explicit dependence on the policy parameters. As mentioned above, another attractive property of this particular choice of kernel is that it renders the integrals in Equation~\ref{eq:UV_def} analytically tractable, as made explicit in the following proposition 

\begin{proposition} \label{prop:UV}
Let $k(\vecz,\vecz')=k_x(x,x')+k_F(\vecz,\vecz')$ for all $(\vecz,\vecz')\in\Z^2$, where $k_x: \X^2\rightarrow\Re$ is an arbitrary positive definite state-kernel and $k_F:\Z^2\rightarrow\Re$ is the Fisher information kernel. Then $\matB_t$ and $\matB_0$ from Equation~\ref{eq:UV_def} satisfy
\begin{equation} \label{eq:UV}
\matB_t=\matU_t\;,\hspace{1.0in}\matB_0=\matG,
\end{equation}
where $\matU_t=\big[\vecu(\vecz_0),\vecu(\vecz_1),\ldots,\vecu(\vecz_t)\big]$.
\end{proposition}
\begin{proof}
See Appendix~\ref{sec:proofP6}.
\end{proof}

An immediate consequence of Proposition~\ref{prop:UV} is that, in order to compute the posterior moments of the policy gradient, we only need to be able to evaluate (or estimate) the score vectors $\vecu(\vecz_i),\;i=0,\ldots,t$ and the Fisher information matrix $\matG$ of our policy. Evaluating the Fisher information matrix $\matG$ is somewhat more challenging, since on top of taking the expectation with respect to the policy $\mu(a|x;\vectheta)$, computing $\matG$ involves an additional expectation over the state-occupancy density $\nu^\mu(x)$, which is not generally known. In most practical situations we therefore have to resort to estimating $\matG$ from data. When $\nu^\mu$ in the definition of the Fisher information matrix (Equation~\ref{eq:Fisher-state-action}) is the stationary distribution over states under policy $\mu$, one straightforward method to estimate $\matG$ from a trajectory $\vecz_0,\vecz_1,\ldots,\vecz_t$ is to use the (unbiased) estimator (see Proposition~\ref{prop:UV} for the definition of $\matU_t$):
\begin{equation} \label{eq:Ghat}
\hat{\matG}_t=\frac{1}{t+1}\sum_{i=0}^t\vecu(\vecz_i)\vecu(\vecz_i)^\top=\frac{1}{t+1}\matU_t \matU_t^\top.
\end{equation}
%
In case $\nu^\mu$ in Equation~\ref{eq:Fisher-state-action} is a discounted weighting of states encountered by following policy $\mu$ (as it is considered in this paper), a method for estimating $\matG$ from a number of trajectories is shown in Algorithm~\ref{alg:G_est}. Note that $(1-\gamma)\nu^\mu$ corresponds to the distribution of a Markov chain that starts from a state sampled according to $P_0$ and at each step either follows the policy $\mu$ with probability $\gamma$ or restarts from a new initial state drawn from $P_0$ with probability $1-\gamma$. It is easy to show that the average number of steps between two successive restarts of this distribution is $1/(1-\gamma)$.

\begin{algorithm}
\caption{\label{alg:G_est} Fisher Information Matrix Estimation Algorithm}
\begin{algor}[1]
\item [{*}]{\bf G-EST}$(\vectheta, M)$ \\
$\bullet$ $\vectheta$ policy parameters\\ 
$\bullet$ $M>0$ number of episodes used to estimate the Fisher information matrix\\
\item [{*}]$\hat{\matG}(\vectheta)=\vec0$
\item [for]$i=1$ to $M$
\item [{*}]done = false,$\hspace{0.25in}$term = false,$\hspace{0.25in}t=-1$
\item [{*}]Draw $x_0^i\sim P_0(\cdot)$
\item [while]not done
\item [{*}] $t=t+1$
\item [{*}] Draw $a_t^i\sim\mu(\cdot|x_t^i;\vectheta)\;$ and $\;x_{t+1}^i\sim P(\cdot|x_t^i,a_t^i)$
\item [{*}] {\bf if} $x_{t+1}^i=x_{\text{term}}$ {\bf then} done = true
\item [{*}] {\bf if} (done = false $\wedge$ term = false) {\bf then} 
\item [{*}]$\hspace{0.175in}\hat{\matG}(\vectheta):=\hat{\matG}(\vectheta)+\vecu(\vecz_t^i;\vectheta)\vecu(\vecz_t^i;\vectheta)^\top\hspace{0.125in}$ and $\hspace{0.125in}$ w.p. $1-\gamma\;\;$ term = true
\item [{*}] {\bf end if}
\item [endwhile]
\item [{*}] {\bf if} term = false {\bf then} $\hat{\matG}(\vectheta)=\hat{\matG}(\vectheta)+\big(\vecu(\vecz_t^i;\vectheta)\vecu(\vecz_t^i;\vectheta)^\top\big)/(1-\gamma)$
\item [endfor]
\item [{*}]{\bf return} $\hat{\matG}(\vectheta):=\hat{\matG}(\vectheta)/M$
\end{algor}
\end{algorithm}

Algorithm~\ref{alg:bac} is a pseudocode sketch of the Bayesian actor-critic algorithm, using either the conventional gradient or the natural gradient in the policy update, and with $\matG$ estimated using either $\hat{\matG}_t$ in Equation~\ref{eq:Ghat} or $\hat{\matG}(\vectheta)$ in Algorithm~\ref{alg:G_est}. 

\begin{algorithm}
\caption{\label{alg:bac} A Bayesian Actor-Critic Algorithm}
\begin{algor}[1]
\item [{*}]{\bf BAC}$(\vectheta, M, \epsilon)$ \\
$\bullet$ $\vectheta$ initial policy parameters\\ 
$\bullet$ $M>0$ episodes for gradient evaluation\\
$\bullet$ $\epsilon>0$ termination threshold\\
\item [{*}]done = false
\item [while]not done
\item [{*}]Run {\bf GPTD} for $M$ episodes. {\bf GPTD} returns $\vecalpha_t$ and $\matC_t$ (Equation~\ref{eq:alpha_C})
\item [{*}]Compute an estimate of the Fisher information matrix $\hat{\matG}_t$ (Equation~\ref{eq:Ghat}) or $\hat{\matG}(\vectheta)$ (Algorithm~\ref{alg:G_est})
\item [{*}]Compute $\matU_t$ (Proposition~\ref{prop:UV})
\item [{*}]$\Delta\vectheta=\matU_t\vecalpha_t\hspace{2.55in}$ {\bf (conventional gradient)} $\quad$ or \\
	$\Delta\vectheta=\hat{\matG}_t^{-1}\matU_t\vecalpha_t\;\;\;$ or
	$\;\;\;\Delta\vectheta=\hat{\matG}(\vectheta)^{-1}\matU_t\vecalpha_t\hspace{0.5in}$ {\bf (natural gradient)} 
\item [{*}]$\vectheta := \vectheta + \beta \Delta \vectheta$
\item [{*}] {\bf if} $\left| \Delta \vectheta \right| < \epsilon$ {\bf then }done = true
\item [endwhile]
\item [{*}]{\bf return} $\vectheta$
\end{algor}
\end{algorithm}


\subsection{BAC Online Sparsification}
\label{subsec:BAC-sparsification}

As was done for the BPG algorithms in Section~\ref{subsec:BPG-sparsification}, Algorithm~\ref{alg:bac} may be made more efficient, both in time and memory, by sparsifying the solution. 
\citet{Engel05RL} presented a sparse approximation of the GPTD algorithm by using an online sparsification method from~\citet{Engel02SO}. This sparsification method incrementally constructs a dictionary $\tilde{\D}$ of representative state-action pairs. Upon observing a new state-action pair $\vecz_i$, the distance between the feature-space image of $\vecz_i$ and the span of the images of current dictionary members is computed. If the squared distance 
$\delta_i=k(\vecz_i,\vecz_i)-\tilde{\veck}_{i-1}^\top(\vecz_i)\tilde{\matK}_{i-1}^{-1}\tilde{\veck}_{i-1}(\vecz_i)$ exceeds some positive threshold $\tau$, $\vecz_i$ is added to the dictionary, otherwise, it is left out. In calculating $\delta_i$, 
$\tilde{\veck}_{i-1}$ and $\tilde{\matK}_{i-1}$ are the dictionary kernel vector and kernel matrix before observing $\vecz_i$, respectively.~\citet{Engel05RL} showed that using this sparsification procedure, the posteriors moments of $Q$ may be compactly approximated as $\hat{Q}_t(\vecz)=\tilde{\veck}_t^\top(\vecz)\tilde{\vecalpha}_t$ and $\hat{S}_t(\vecz,\vecz')=k(\vecz,\vecz')-\tilde{\veck}_t^\top(\vecz)\tilde{\matC}_t\tilde{\veck}_t(\vecz')$, where
\begin{equation}
\tilde{\vecalpha}_t=\tilde{\matH}_t^\top \left(\tilde{\matH}_t\tilde{\matK}_t\tilde{\matH}_t^\top+\matSigma_t \right)^{-1}\vecr_{t-1}\;,\hspace{0.55in}
\tilde{\matC}_t =\tilde{\matH}_t^\top \left(\tilde{\matH}_t \tilde{\matK}_t\tilde{\matH}_t^\top+\matSigma_t\right)^{-1}\tilde{\matH}_t\;.
\label{eq:sparseQ}
\end{equation}
In Equation~\ref{eq:sparseQ}, $\tilde{\matH}_t=\matH_t\matA_t$, where $\matA_t$ is a $|\D_t|\times |\tilde{\D}_t|$ matrix whose $i$'th row is $[\matA]_{i,|\tilde{\D}_i|}=1$ and $[\matA]_{i,j}=0\;;\;\forall j\neq |\tilde{\D}_i|$, if we add the state-action pair $\vecz_i$ to the dictionary, and is 
$\tilde{\veck}^\top_{i-1}(\vecz_i)\tilde{\matK}_{i-1}^{-1}$ followed by zeros otherwise.

\begin{proposition}\label{prop:UV-sparse}
Using the sparsification method described above, the posterior moments of the gradient are approximated as
\begin{equation*}
\exptE\big[\nabla\eta(\vectheta)|\D_t\big]=\tilde{\matU}_t\tilde{\vecalpha}_t\;, \hspace{1in}
\Cov\big[\nabla\eta(\vectheta)|\D_t\big]=\matG-\tilde{\matU}_t\tilde{\matC}_t\tilde{\matU}_t^\top\;,
\end{equation*}
where $\tilde{\vecalpha}_t$ and $\tilde{\matC}_t$ are given by Equation~\ref{eq:sparseQ} and $\tilde{\matU}_t=\big[\vecu(\vecz_1),\ldots,\vecu(\vecz_{|\tilde{\D}_t|})\big]$ with $\vecz_i\in\tilde{\D}_t$.
\end{proposition}
\begin{proof}
The proof is straightforward by plugging the sparsified posterior mean and covariance of $Q$ with $\tilde{\vecalpha}_t$ and $\tilde{\matC}_t$ from Equation~\ref{eq:sparseQ} in Equation~\ref{eq:grad_eta_post} and following the steps until the end of Proposition~\ref{prop:UV}.
\end{proof}

\section{BAC Experimental Results}
\label{sec:BAC-experiments}

In this section, we empirically\footnote{The code for all the experiments of this section is available at \url{https://sequel.lille.inria.fr/Software/BAC}.} evaluate the performance of the Bayesian actor-critic method presented in this paper in a 10-state random walk problem as well as in the widely used continuous-state-space mountain car problem~\citep{Sutton98IR} and ship steering problem~\citep{Miller90NN}. In Section~\ref{sec:random-walk}, we first compare BAC, Bayesian quadrature (BQ), and Monte Carlo (MC) gradient estimates in the 10-state random walk problem. We then evaluate the performance of the BAC algorithm on the same problem, and compare it with a Bayesian policy gradient (BPG) algorithm and a MC-based policy gradient (MCPG) algorithm. In Section~\ref{sec:mountain-car}, we compare the performance of the BAC algorithm with a MCPG algorithm on the mountain car problem. The BPG, BAC, and MCPG algorithms used in our experiments are Algorithms~\ref{alg:BPG} and~\ref{alg:bac} presented in this paper, and Algorithm~1 in~\citet{Baxter01IP}, respectively. In Section~\ref{sec:ship-steering}, we compare the performance of the BAC algorithm with a MCPG algorithm on a problem in the ship steering domain. Similar to Section~\ref{sec:mountain-car}, the BAC, and MCPG algorithms used in our experiments are Algorithm~\ref{alg:bac} presented in this paper and Algorithm~1 in~\citet{Baxter01IP}, respectively.


\subsection{A Random Walk Problem}
\label{sec:random-walk}
   
In this section, we consider a 10-state random walk problem, $\X =\{1,2,\ldots,10\}$, with states arranged linearly from state 1 on the left to state 10 on the right. The agent has two actions to choose from: $\A=\{left,right\}$. The left wall is a retaining barrier, meaning that if the $left$ action is taken at state 1, in the next time-step the state will be 1 again. State 10 is a zero reward  absorbing state. The only stochasticity in the transitions is induced by the policy, which is defined as $\mu(right|x)=1/1+\exp(-\theta_x)$ and $\mu(left|x)=1-\mu(right|x)$, for all $x\in\X$. Note that each state $x$ has an independent parameter $\theta_x$. Each episode begins at state 1 and ends when the agent reaches state 10. The mean reward is 1 for states 1--9 and is 0 for state 10. The observed rewards for states 1--9 are obtained by corrupting the mean rewards with a 0.1 standard deviation i.i.d.~Gaussian noise. The discount factor is $\gamma=0.99$. In the BAC experiments, we use the Gaussian state kernel $k_x(x,x')=\exp(-||x-x'||^2/(2\sigma_k^2))$ with $\sigma_k=3$ and the state-action kernel $0.01k_F(\vecz,\vecz')$.

We first compare the MC, BQ, and BAC estimates of $\nabla\eta(\vectheta)$ for the policy induced by the parameters $\theta_x=\log(41/9)$ for all $x\in\X$, which is equivalent to 
$\mu(right|x)=0.82$. 
We use several different sample sizes: $M=5j,\;j=1,\ldots,20$. Here, by sample size we mean the number of episodes used to estimate the gradient. For each value of $M$, we compute the gradient estimates $10^3$ times. The true gradient is calculated analytically for reference. Figure~\ref{Fig1-RW} shows the mean squared error and the mean absolute angular error of MC, BQ, and BAC estimates of the gradient for different sample sizes $M$. 
The error bars in the right figure are the standard errors of the mean absolute angular errors. The results depicted in Figure~\ref{Fig1-RW} indicate that the BAC gradient estimates are more accurate and have lower variance than their MC and BQ counterparts. 

\begin{figure}
\begin{center}
\includegraphics[height=2in,width=1\textwidth]{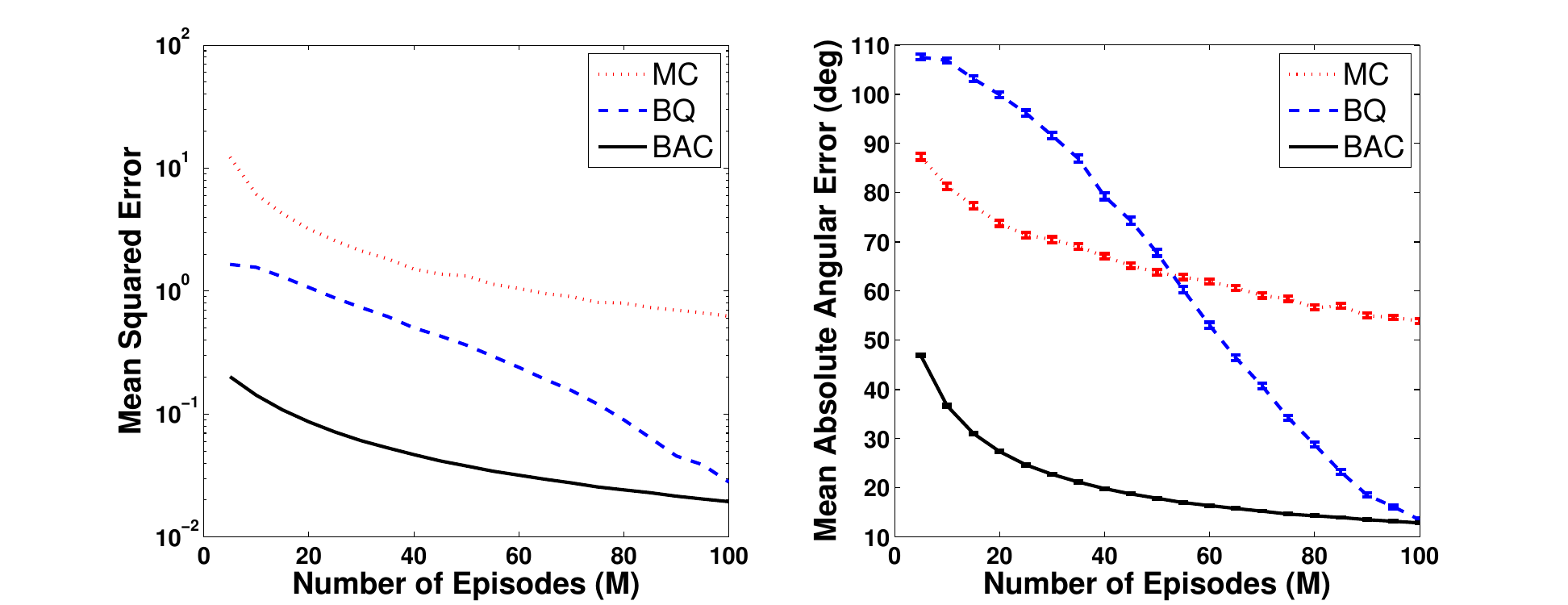}
\caption{The mean squared error and the mean absolute angular error of MC, BQ, and BAC gradient estimations as a function of the number of sample episodes $M$. All results are averaged over $10^3$ runs.}        
\label{Fig1-RW}
\end{center}
\end{figure}



Next, we use BAC to optimize the policy parameters and compare its performance with a BPG algorithm and a MCPG algorithm for $M=1,\;25,\;50$, and $75$. The BPG algorithm uses Model~1 of~Section~\ref{subsec:M1}. We use Algorithm~\ref{alg:bac} with the number of policy updates set to $500$ and the same kernels as in the previous experiment. The Fisher information matrix is estimated using Algorithm~\ref{alg:G_est}. The returns obtained by these methods are averaged over $10^3$ runs. For a fixed sample size $M$, we tried many values of the learning rate, $\beta$, for MCPG, BPG, and BAC, and those in Table~\ref{tab1-RW} yielded the best performance. Note that the learning rate used for each algorithm in each experiment is fixed and does not converge to zero. BAC showed a very robust performance when we changed the learning rate. By robust we mean that it never generated a policy for which an episode does not end after $10^6$ steps. This seems to be due to the fact that BAC gradient estimates are more accurate and have less variance than their MC and BPG counterparts. The performance of BPG improves as we increase the sample size $M$. It performs worse than MCPG for $M=1$ and $25$, but achieves a performance similar to BAC for $M=100$.  

\begin{table}
\begin{center}
\begin{tabular}{|c|c|c|c|c|}
\hline 
$\beta$ & $M=1$ & $M=25$ & $M=50$ & $M=75$ \\
\hline
MCPG & $0.005$ & $0.075$ & $0.1$ & $0.75$ \\
BPG & $0.0035$ & $0.015$ & $0.09$ & $0.5$ \\
BAC & $5$ & $5$ & $5$ & $5$ \\
\hline
\end{tabular}
\end{center}
\caption{Learning rates used by the algorithms in the experiments of Figure~\ref{Fig2-RW}.}
\label{tab1-RW}
\end{table}  

Figure~\ref{Fig2-RW} depicts the results of these experiments. From left to right and top to bottom the sub-figures correspond to the experiment in which all the algorithms used $M=1,\;25,\;50,$ and $75$ trajectories per policy update, respectively. Each curve depicts the difference between the exact average discounted return for the $500$ policies that follow each policy update and $\eta^*$ -- the optimal average discounted return. All curves are averaged over $10^3$ repetitions of the experiment. The BAC algorithm clearly learns significantly faster than the other algorithms (note that the vertical scale is logarithmic).

\begin{figure*}[!ht]
\begin{center}
\includegraphics[height=2in,width=0.475\textwidth]{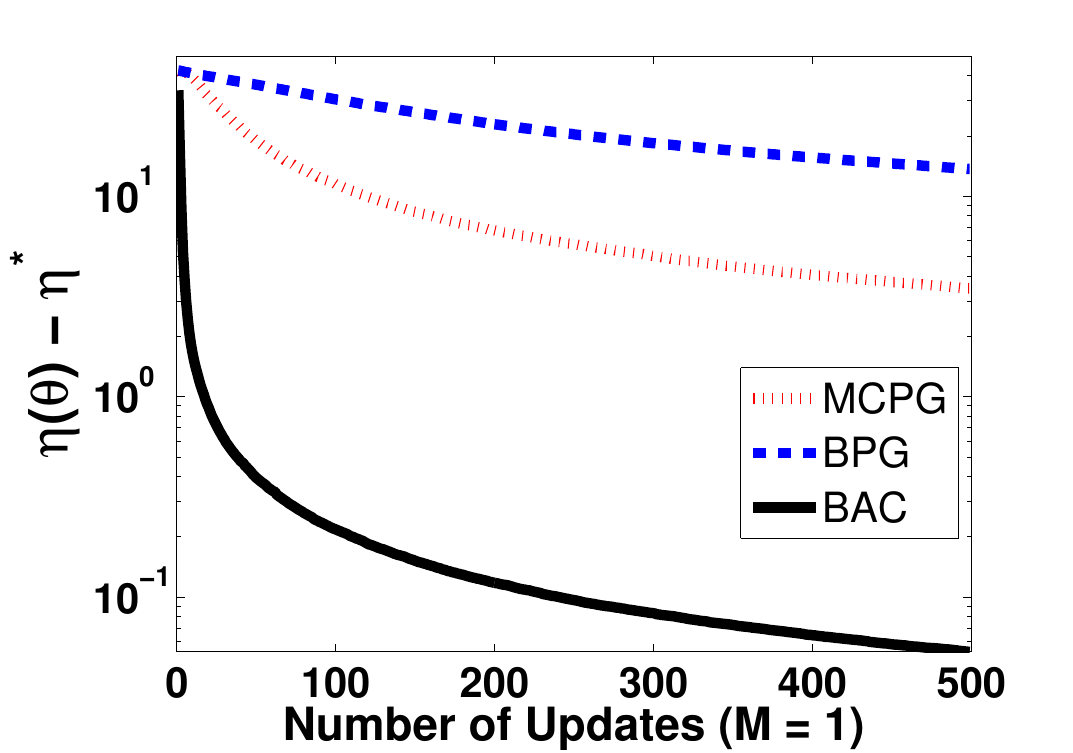}
\includegraphics[height=2in,width=0.475\textwidth]{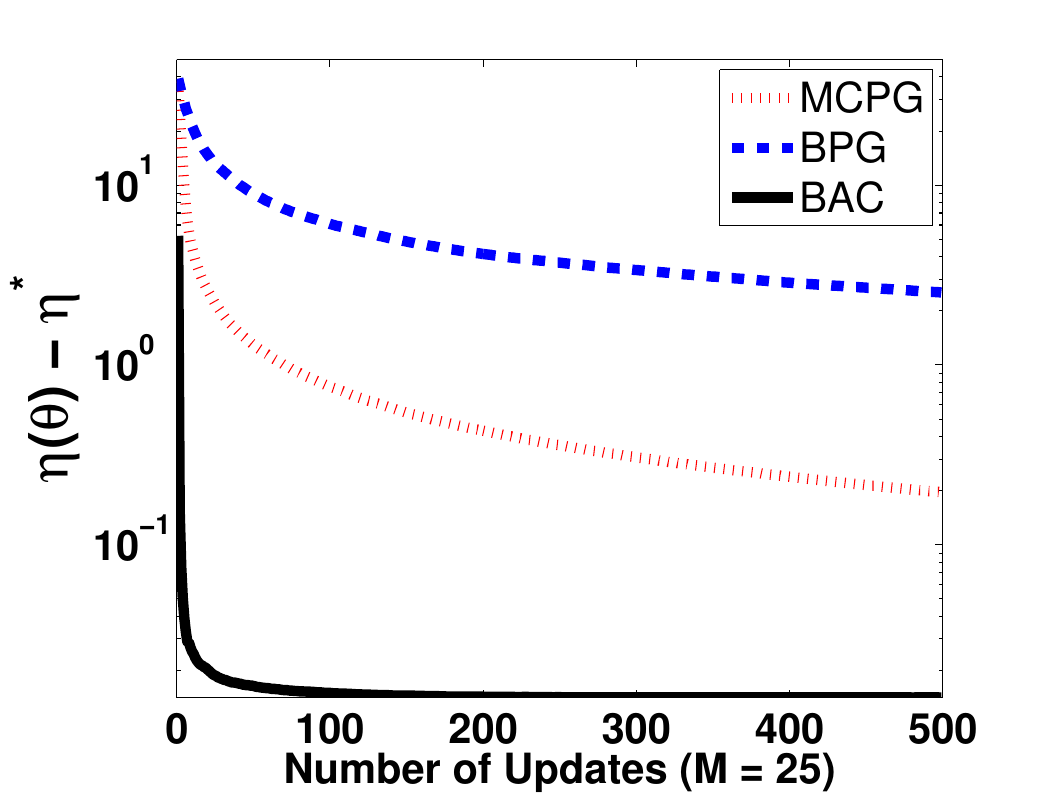} \\
\includegraphics[height=2in,width=0.475\textwidth]{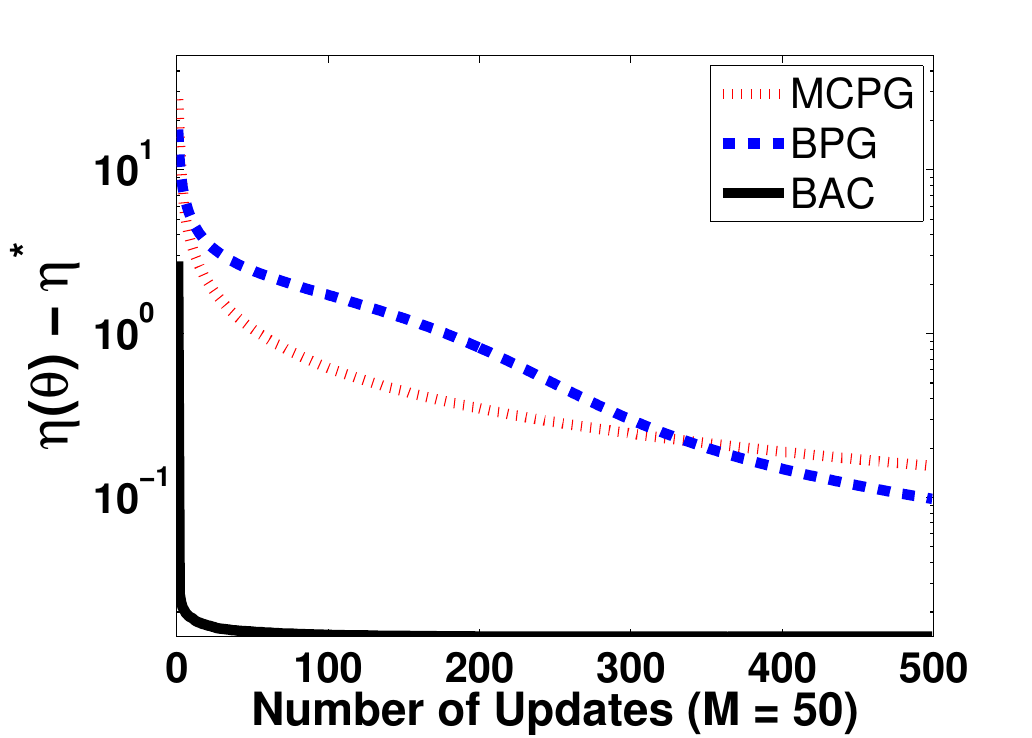}
\includegraphics[height=2in,width=0.475\textwidth]{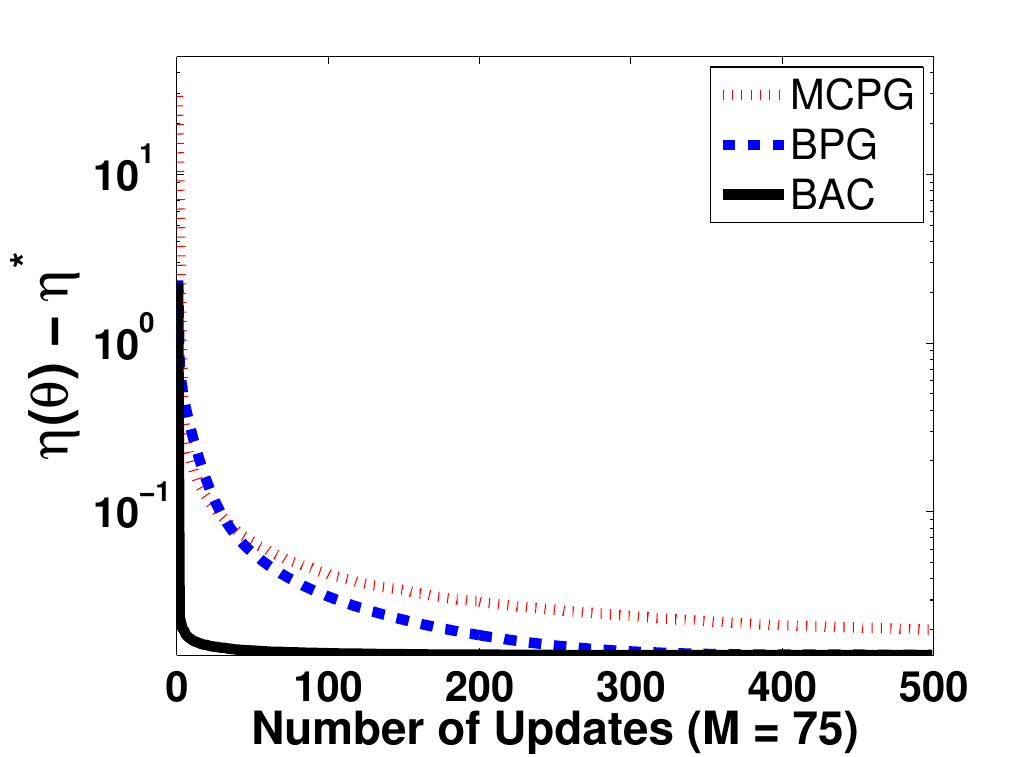}
\caption{Results for the policy learning experiment. The graphs depict the performance of the policies learned by each algorithm during $500$ policy updates. From left to right and top to bottom the number of episodes used to estimate the gradient is $M = 1,\;25,\;50$ and $75$. All results are averaged over $10^3$ independent runs.} 
\label{Fig2-RW}
\end{center}
\end{figure*}

\noindent
{\bf Remark:} Since BQ (and as a result BPG) is based on defining a kernel over system trajectories (quadratic Fisher kernel in Model~1~and Fisher kernel in Model~2), its performance degrades when the system generates trajectories of different size. This effect can be observed by most kernels that have been used in the literature for the trajectories generated by dynamical systems. This can be also observed in our experiments: BQ performs much better than MC in the ``Linear Quadratic Regulator" problem (Section~\ref{subsec:LQR}), in which all the system trajectories are of size 20, while its superiority over MC is less apparent in the ``Random Walk" problem (Section~\ref{sec:random-walk}). This is why we are not going to use BQ and BPG in the ``Mountain Car" (Section~\ref{sec:mountain-car}) and ``Ship Steering" (Section~\ref{sec:ship-steering}) problems, in which the system trajectories have different lengths.


\subsection{Mountain Car}
\label{sec:mountain-car}

In this section, we consider the mountain car problem as formulated in~\citet{Sutton98IR}, and report the results of applying the BAC and MCPG algorithms to optimize the policy parameters in this task.
The state $\vecx$ consists of the position $x$ and the velocity $\dot{x}$ of the car: $\vecx=(x,\dot{x})$. The reward is $-1$ on all time steps until the car reaches its goal at the top of the hill, which ends the episode. There are three possible actions: {\em forward}, {\em reverse}, and {\em zero}. The car moves according to the following simplified dynamics: 

\begin{table}
\centering
\begin{tabular}{lll}
$-1.2\leq x_{t+1} \leq 0.5\hspace{0.25in}$ & , & $\hspace{0.25in}-0.07\leq\dot{x}_{t+1} \leq 0.07\;,$ \\
$x_{t+1}=\text{bound}[x_t+\dot{x}_{t+1}]\hspace{0.25in}$ & , & $\hspace{0.25in}\dot{x}_{t+1}=\text{bound}\big[\dot{x}_t+0.001a_t-0.0025\cos(3x_t)\big]\;.$
\end{tabular}
\end{table}

\noindent
When $x_{t+1}$ reaches the left boundary, $\dot{x}_{t+1}$ is set to zero and when it reaches the right boundary, the goal is reached and the episode ends. Each episode starts from a random position and velocity uniformly sampled from their domains. We use the discount factor $\gamma=0.99$.

In order to define the policy, we first map the states $\vecx=(x,\dot{x})$ to the unit square $[0,1]\times[0,1]$. The policy used in our experiments has the following form:
\begin{equation*}
\mu(a_i|\vecx)=\frac{\exp\big(\phi(\vecx,a_i)^\top\vectheta\big)}{\sum_{j=1}^3\exp\big(\phi(\vecx,a_j)^\top\vectheta\big)}\;,\quad\quad i=1,2,3.
\end{equation*}
The policy feature vector is defined as $\phi(\vecx,a_i)=\big(\phi(\vecx)^\top\delta_{a_1a_i},\phi(\vecx)^\top\delta_{a_2a_i},\phi(\vecx)^\top\delta_{a_3a_i}\big)^\top$, where $\delta_{a_ja_i}$ is $1$ if $a_j=a_i$, and is $0$ otherwise. The state feature vector $\phi(\vecx)$ is composed of $16$ Gaussian functions arranged in a $4\times 4$ grid over the unit square as follows:
\begin{equation*}
\phi(\vecx)=\Big(\exp\big(-||\vecx-\bar{\vecx}_1||^2/(2\kappa^2)\big),\ldots,\exp\big(-||\vecx-\bar{\vecx}_{16}||^2/(2\kappa^2)\big)\Big)^\top,
\end{equation*}
where the $\bar{\vecx}_i$'s are the $16$ points of the grid $\{0,0.25,0.5,1\}\times\{0,0.25,0.5,1\}$ and $\kappa=1.3\times 0.25$. 

In Figure~\ref{Fig1-MC}, we compare the performance of BAC with a MCPG algorithm for $M=5,\;10,\;20,$ and $40$ episodes used to estimate each gradient. For BAC, we use Algorithm~\ref{alg:bac} with the number of policy updates set to $500$, a Gaussian state kernel $k_x(\vecx,\vecx')=\exp\big(-||\vecx-\vecx'||^2/(2\sigma_k^2)\big)$, with $\sigma_k=1.3\times 0.25$, and the state-action kernel $k_F(\vecz,\vecz')$. The Fisher information matrix is estimated using Algorithm~\ref{alg:G_est}. After every $50$ policy updates the learned policy is evaluated for $10^3$ episodes to estimate accurately the average number of steps to goal. Each evaluation episode starts from a random position and velocity uniformly chosen from their ranges, and continues until the car either reaches the goal or a limit of $200$ time-steps is exceeded. The experiment is repeated $100$ times for the entire horizontal axis to obtain average results and confidence intervals. The error bars in this figure are the standard errors of the performance of the algorithms.

\begin{figure*}[!h]
\begin{center}
\includegraphics[height=2in,width=0.475\textwidth]{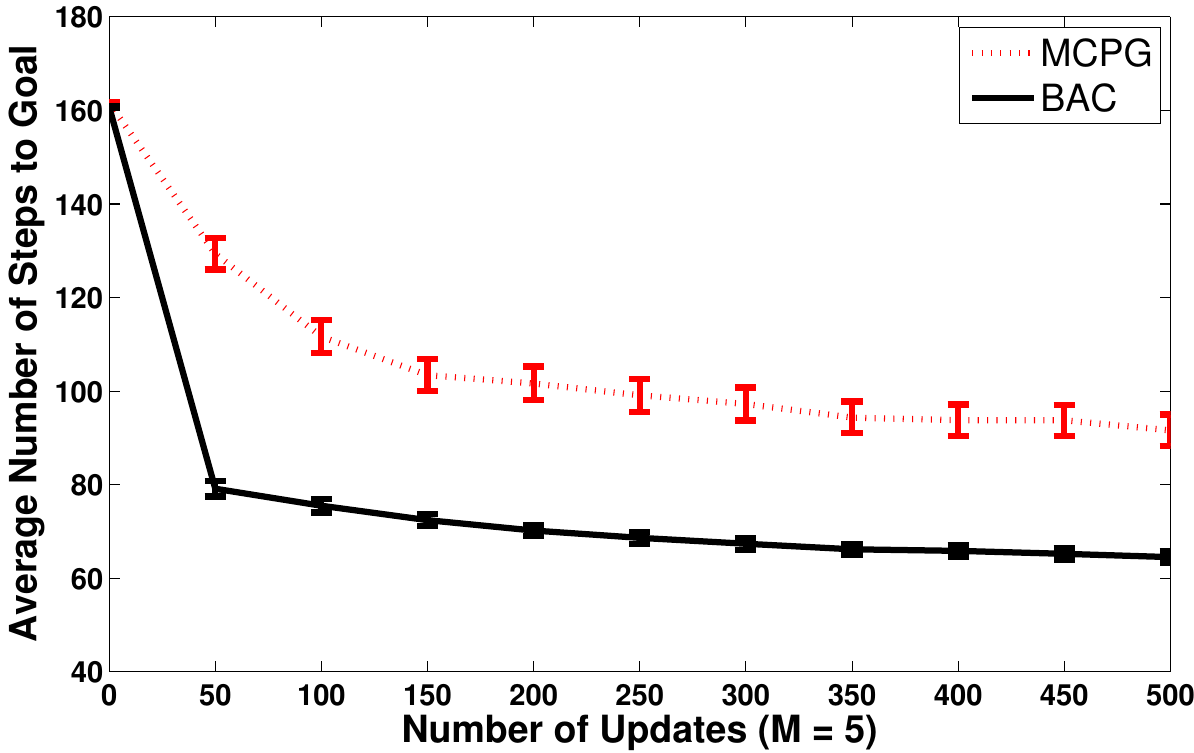}
\includegraphics[height=2in,width=0.475\textwidth]{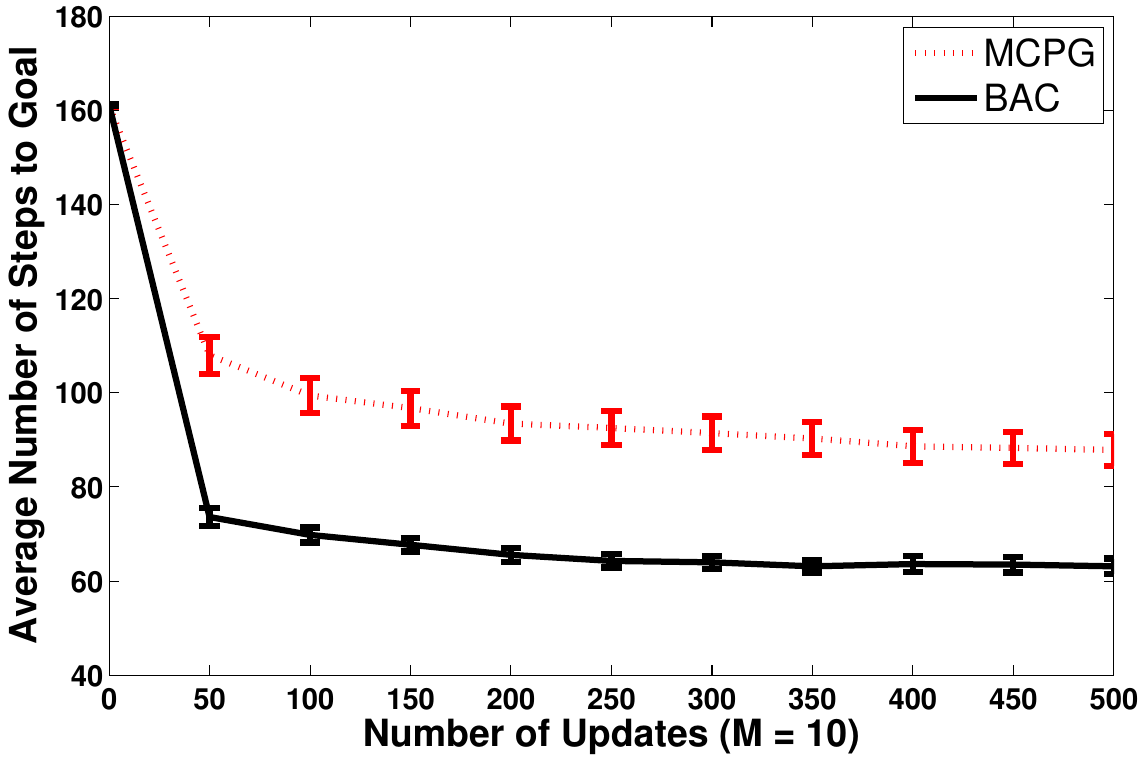} \\
\includegraphics[height=2in,width=0.475\textwidth]{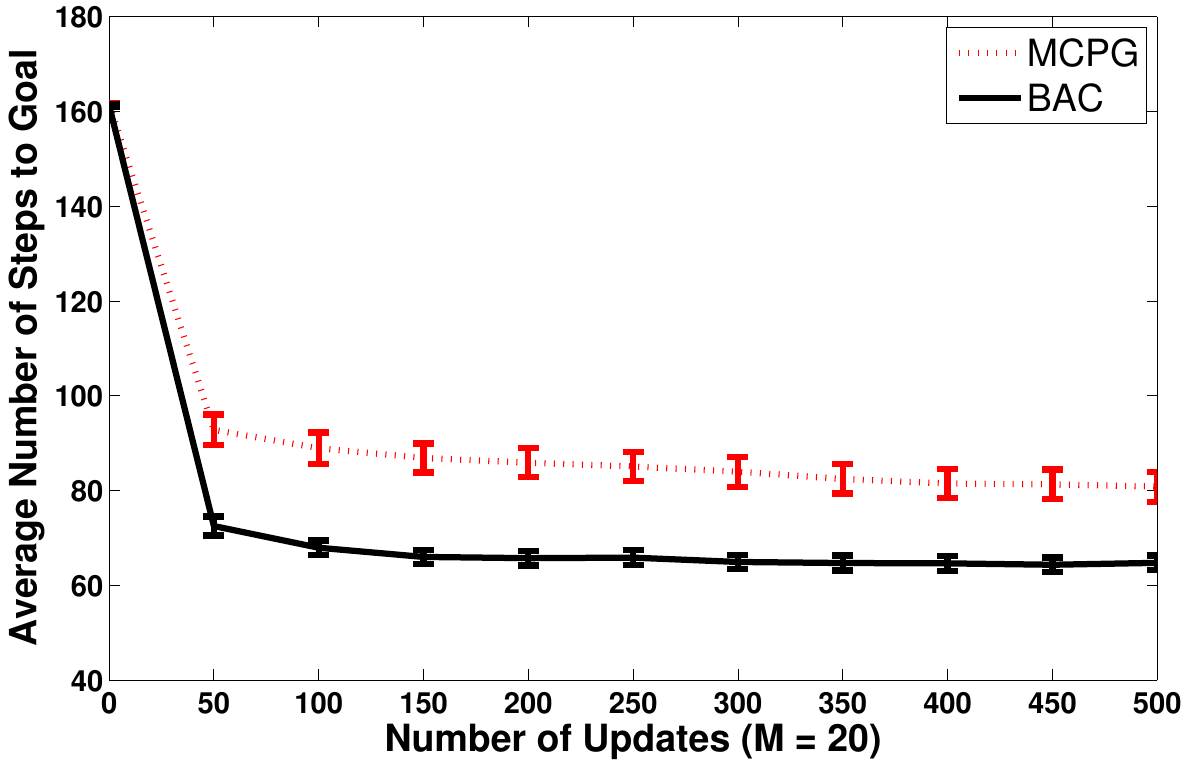}
\includegraphics[height=2in,width=0.475\textwidth]{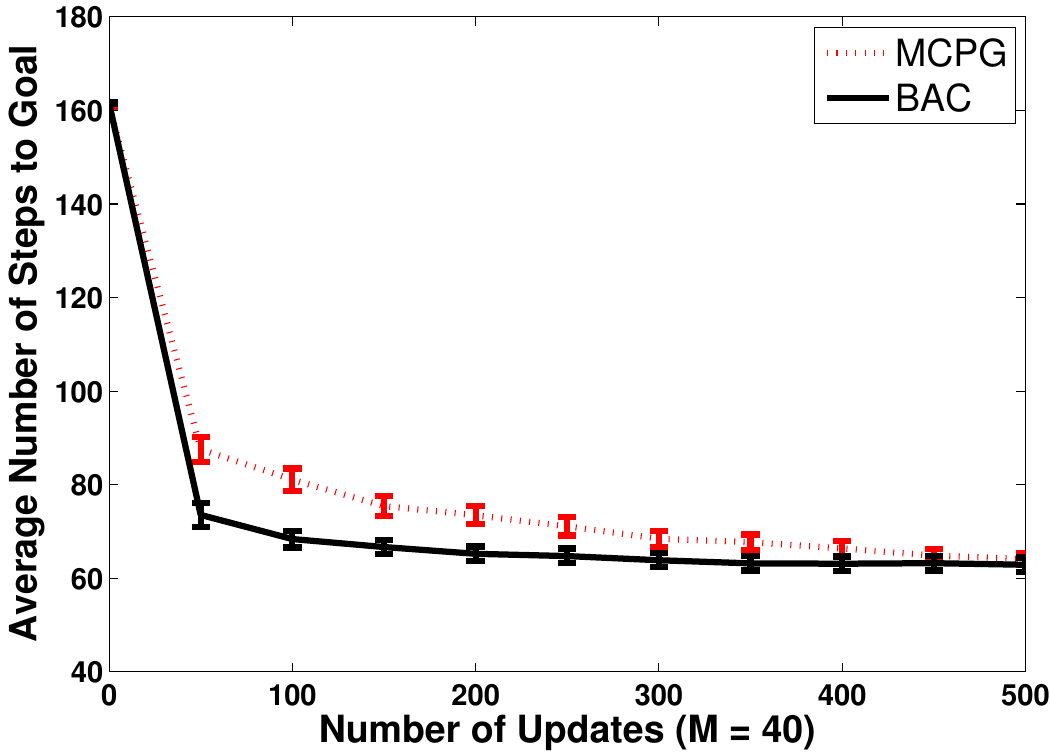}
\caption{The graphs depict the performance of the policies learned by BAC and a MCPG algorithm during $500$ policy updates in the mountain car problem. From left to right and top to bottom the number of episodes used to estimate the gradient is $M = 5,\;10,\;20,$ and $40$. All results are averaged over $100$ independent runs.}        
\label{Fig1-MC}
\end{center}
\end{figure*}

For a fixed sample size $M$, each method starts with an initial learning rate and decreases it according to the schedule $\beta_t=\beta_0\beta_c/(\beta_c+t)$. We tried many values of the learning rate parameters $(\beta_0,\beta_c)$ for MCPG and BAC, and those in Table~\ref{tab1-MC} yielded the best performance. Note that $\beta_c=\infty$ means that we used a fixed learning rate $\beta_0$ for that experiment. The graphs indicate that BAC performs better and has lower variance than MCPG. It is able to find a good policy with only $M=5$ sample size and its performance does not change much as the sample size is increased. On the other hand, the performance of MCPG improves and its variance is reduced as we increase the sample size. Note that for $M=40$, MCPG finally achieves a similar performance (still with slower rate) as BAC.  

\begin{table}
\begin{center}
\begin{tabular}{|c|c|c|c|c|}
\hline 
$\beta$ & $M=5$ & $M=10$ & $M=20$ & $M=40$ \\
\hline
MCPG & $0.025\;,\;\infty$ & $0.1\;,\;100$ & $0.2\;,\;100$ & $0.25\;,\;\infty$ \\
BAC & $0.025\;,\;\infty$ & $0.05\;,\;\infty$ & $0.1\;,\;\infty$ & $0.1\;,\;250$ \\
\hline
\end{tabular}
\end{center}
\caption{Learning rates used by the algorithms in the experiments of Figure~\ref{Fig1-MC}.}
\label{tab1-MC}
\end{table}  


\subsection{Ship Steering}
\label{sec:ship-steering}

In this section, we perform comparative experiments between BAC and MCPG on a more challenging problem in the continuous state continuous action \textit{ship steering} domain~\citep{Miller90NN}.

\paragraph{Domain Description}
In this domain, a ship is located in a $150 \times 150$ meter square water surface. At any point in time $t$, the state of the ship is described by four continuous variables that are defined below along with their range of values
\begin{equation*}
{\bf x}_t =(x_t,y_t,\theta_t,\dot\theta_t) \in [0\mathrm{m},150\mathrm{m}] \times
[0\mathrm{m},150\mathrm{m}] \times [-180^{\circ},180^{\circ}] \times 
[-15^{\circ}/\mathrm{s},15^{\circ}/\mathrm{s}],
\end{equation*}
where $x_t$ and $y_t$ represent the position of the ship, $\theta_t$ the angle between the vertical axis and the ship orientation, and $\dot\theta_t$ the actual turning rate (see the upper-left panel in Figure~\ref{fig:ship-m5}). At the beginning of each episode, the ship starts at $(x_1,y_1) = (40\mathrm{m},40\mathrm{m})$, with $\theta_1$ and $\dot\theta_1$ sampled uniformly at random from their ranges. The only available action variable is $a_t \in [-15^{\circ},15^{\circ}]$, which is the \textit{desired} turning rate. To model the ship inertia and water resistance, there is a $T=5$ time steps lag for the desired turning rate to become the actual turning rate. Moreover, the ship moves with the constant speed of $V = 3\mathrm{m/s}$ and $\Delta = 0.2\mathrm{s}$ is the sampling interval. The following set of equations summarizes the ship's dynamics:
\begin{align*}
x_{t+1} &= x_t + \Delta \; V  \sin \theta_t  \quad\quad\quad\quad\quad\quad\quad y_{t+1} = y_t + \Delta \; V \cos \theta_t \\
\theta_{t+1} &= \theta_t + \Delta \; \dot\theta_t \quad\quad\quad\quad\quad\quad\quad\quad\quad\;\; \dot\theta_{t+1} = \dot\theta_t + \frac{\Delta}{T}(a_t - \dot\theta_t)
\end{align*}
The goal of the ship is to navigate to $(x_*,y_*) = (100\mathrm{m},100\mathrm{m})$ within 500 times steps. If this does not happen or the ship moves out of the boundary, the episode terminates as a failure. The goal of the policy is to maximize the probability of the ship successfully reaching $(x_*,y_*)$. Thus, we set the discount factor to $\gamma=1$ in this problem.

\paragraph{Learning} 
For both MCPG and BAC, we used a CMAC function approximator with $9$ four dimensional tilings, each of them discretizing the state space into $5 \times 5 \times 36 \times 5 = 4500$ tiles. Therefore, each policy parameter $\vecw$ is of size $N = 9 \times 4500 = 40500$. Each state ${\bf x}$ is represented by a binary vector ${\mathbf \phi}({\bf x})$, where $\phi_i({\bf x}) = 1$ if and only if the state ${\bf x}$ falls in the $i$th tile, and thus, $\sum_{i=1}^N \phi_i({\bf x}) \leq 9$. To define a precise mapping from states to actions, $\vecw_t: {\bf x_t}=(x_t,y_t,\theta_t,\dot\theta_t) \rightarrow a_t$, we first sample $\tilde a_t$ from the Gaussian
\begin{equation*}
\tilde a_t \sim \mathcal{N}\left(\frac{\sum_{i=1}^N \vecw_t^{(i)}\phi_i({\bf x}_t)}{\sum_{i=1}^N \phi_i({\bf x}_t)},1\right),
\end{equation*}
and then map it to the allowed range $[-15^{\circ},15^{\circ}]$ using the sigmoid transformation
\begin{equation*}
a_t = 15^{\circ} \cdot \frac{2}{\pi} \cdot \arctan\left(\frac{\pi}{2} \cdot 
\tilde a_t 
\right).
\end{equation*}
For the BAC experiments, we used the Gaussian state kernel $k_x({\bf x},{\bf x}')=\exp(-||{\bf x}-{\bf x}'||^2/(2\sigma_k^2))$, with $\sigma_k=1$ and the state-action kernel $k_F(\vecz,\vecz')$, i.e.,~the Fisher kernel.

\paragraph{Setup}
In order to improve the computational efficiency, we use several numerical approximations. First, to calculate the score function for a trajectory (Equation~\ref{eq:score}) in both MCPG and BAC, we approximate the gradient of the action distribution in $a_t$ with the one in $\tilde a_t$, i.e.,
\begin{equation*}
\nabla\log\mu(a_t|{\bf x}_t;\vecw_t) \approx\nabla\log\mu(\tilde a_t|{\bf x}_t;\vecw_t).
\end{equation*}
Second, we calculate the gradient using the online sparsification procedure described in Section~\ref{subsec:BPG-sparsification}. Finaly, we never explicitly calculate the inverse of the Fisher information matrix $\hat{\matG}$ and instead calculate the product of $\hat{\matG}^{-1}$ with the score. For the numerical stability we also add $10^{-6}$ to the diagonal of $\hat{\matG}$.

Similar to the other experiments in the paper, we varied the number of trajectories used to estimate the gradient of a policy as $M=5$, $10$, and $20$. Table~\ref{tab:shiprates} shows the best values of the learning rate $\beta$ for both MCPG and BAC for different values of $M$. To evaluate each method, we ran $100$ independent learning trials. At each trial, we evaluate the performance of the policy every $100$ iterations by executing it $100$ (independent) times with $\theta_1$ and $\dot\theta_1$ randomly sampled. For each of these execution, we observe if the ship reached $(x_*,y_*)$ within $500$ steps and estimate the policy success ratio. We set the total number of gradient updates to $T=3000$ for $M=5$ and $10$ and to $T=1000$ for $M=20$.

\begin{table}
\begin{center}
\begin{tabular}{|c|c|c|c|c|}
\hline 
$\beta$ & $M=5$ & $M=10$ & $M=20$  \\
\hline
MCPG & $0.01$ & $0.01$ & $0.01$  \\
BAC & $0.5$ & $0.4$ & $0.5$ \\
\hline
\end{tabular}
\end{center}
\caption{Learning rates used by the algorithms in the experiments of Figure~\ref{fig:ship-m5}.}
\label{tab:shiprates}
\end{table}  

\paragraph{Results}
The results for all the experiments are presented in Figure~\ref{fig:ship-m5} along with their standard deviations. Naturally, using more trajectories for the gradient update improves both methods. However, this improvement is bigger for the BAC method. In the case of $M=5$, MCPG produces slightly better policies at the beginning of learning, but is soon outperformed by BAC. For $M=10$ and $20$, BAC produces better policies from the beginning, especially for $M=20$. This is consistent with the results of the other experimental domains reported in the paper. For all values of $M$, BAC converges to a policy with a better success ratio than MCPG. Finally, as expected, BAC has usually less variance in its performance than MCPG.

\begin{figure*}[!ht]
\begin{center}
\includegraphics[width=0.45\textwidth]{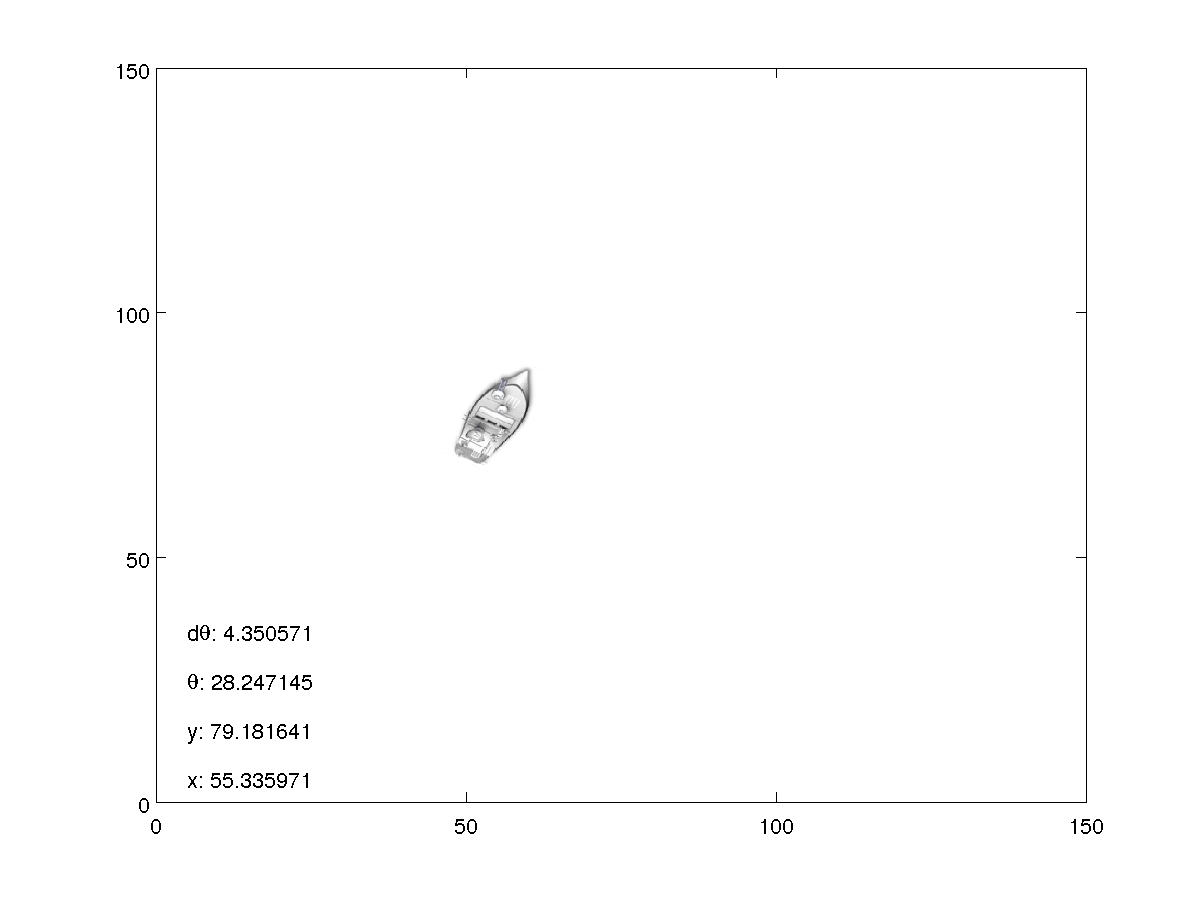}
\includegraphics[width=0.45\textwidth]{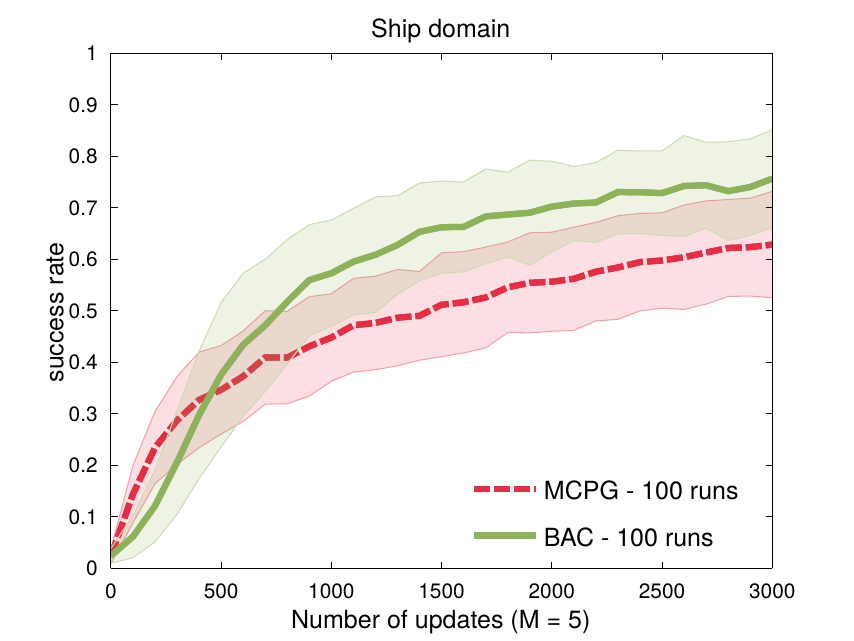}
\includegraphics[width=0.45\textwidth]{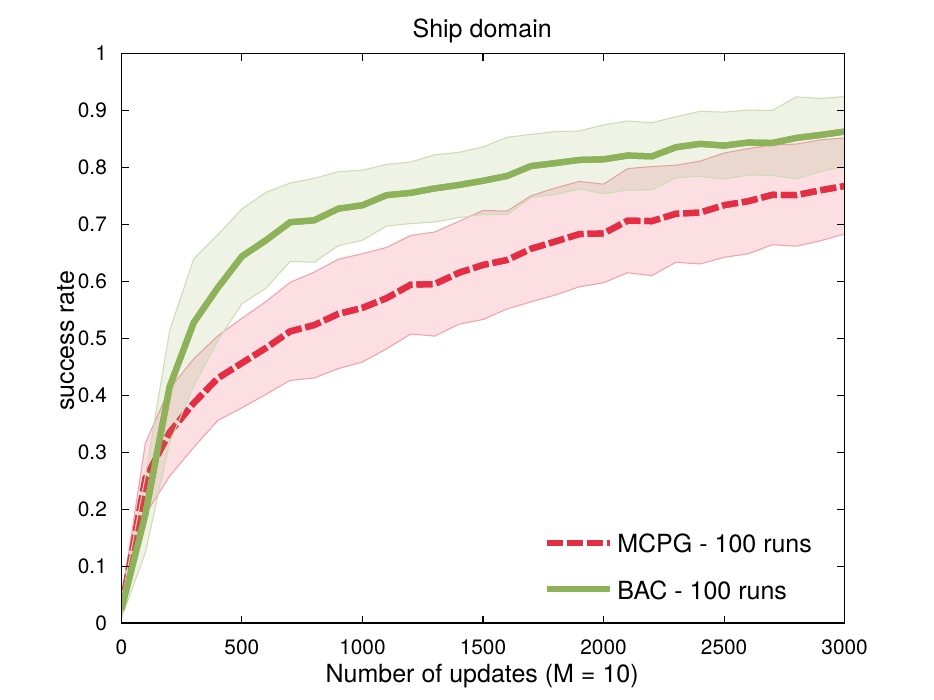}
\includegraphics[width=0.45\textwidth]{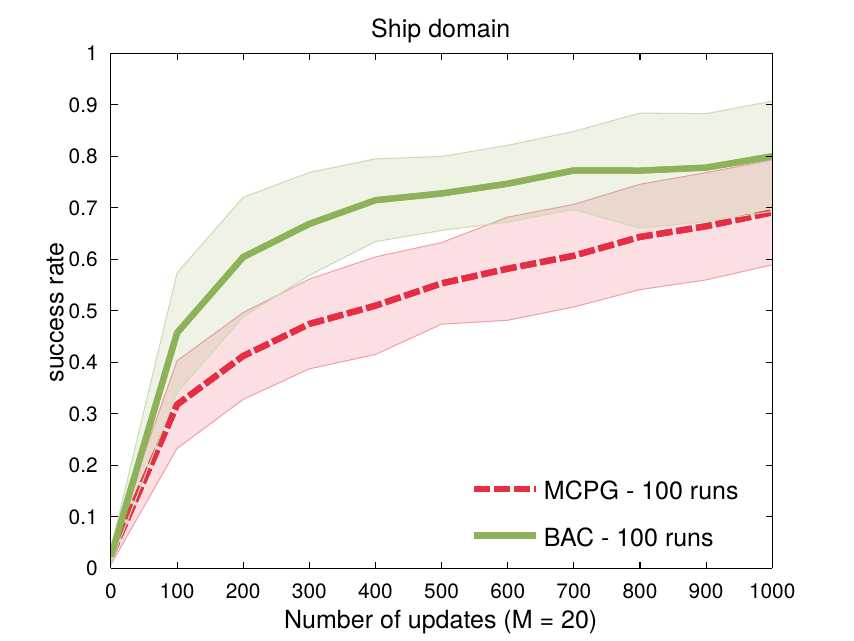}
\caption{Success rate of the policies learned by BAC and MCPG in the ship steering problem.}        
\label{fig:ship-m5}
\end{center}
\end{figure*}


\section{Other Advancements in Bayesian Reinforcement Learning}
\label{sec:related-work}

The algorithms presented in this paper belong to the class of {\em Bayesian model-free RL}, as they do not assume that the system's dynamic is known and do not explicitly construct a model of the system. In recent years, Bayesian methodology has been used to develop algorithms in several other areas of RL. In this section, we provide a brief overview of these results (for more details, see the survey by~\citealt{Ghavamzadeh15BR}).

Another widely-used class of RL algorithms are those that build an explicit model of the system and use it to find a good (or optimal) policy, thus, are known as {\em model-based RL} algorithms. Recent years have witnessed many applications of the Bayesian methodology to this class of RL algorithms. The main idea of model-based Bayesian RL is to explicitly maintain a posterior over the model parameters and to use it to select actions in order to appropriately balance exploration and exploitation. The class of model-based Bayesian RL algorithms include those that work with MDPs and those that work with POMDPs (e.g.,~\citealt{Ross08BAPOMDP,doshi08}). The MDP-based algorithms can be further divided to those that are offline (e.g.,~\citealt{duff01bamdpfsc,poupart06beetle}), those that are online (e.g.,~\citealt{dearden99model,strens00,wang05sparse,Ross08BAPOMDP}), and those that have probably approximately correct (PAC)-guarantees (e.g.,~\citealt{kolter09,asmuth09,sorg10}). 

The use of Bayesian methodology has also been explored to solve the {\em inverse RL} (IRL) problem, i.e.,~learning the underlying model of the decision-making agent (expert) from its observed behavior and the dynamics of the system~\citep{Russell98LA}.  The main idea of Bayesian IRL (BIRL) is to use a prior to encode the reward preference and to formulate the compatibility with the expert's policy as a likelihood in order to derive a probability distribution over the space of reward functions, from which the expert's reward function is somehow extracted. The most notable works in the area of BIRL include those by~\citet{Ramachandran07BI,Choi11MA,Choi12NV,Michini12BN,Michini12IE}. 

Bayesian techniques have also been used to derive algorithms for the {\em collaborative multi-agent RL} problem. When dealing with multi-agent systems, the complexity of the decision problem is increased in the following way: while single-agent BRL requires maintaining a posterior over the MDP parameters (in the case of model-based methods) or over the value/policy (in the case of model-free methods), in multi-agent BRL, it is also necessary to keep a posterior over the policies of the other agents.~\citet{Chalkiadakis13CM} showed that this belief can be maintained in a tractable manner subject to certain structural assumptions on the domain, for example that the strategies of the agents are independent of each other.  

{\em Multi-task RL} (MTRL) is another area that has witnessed the application of Bayesian methodology. All approaches to MTRL assume that the tasks share similarity in some components of the problem such as dynamics, reward structure, or value function. The Bayesian MTRL methods assume that the shared components are drawn from a common generative model~\citep{Wilson07MR,Mehta08TV,Lazaric10BM}. In~\citet{Mehta08TV}, tasks share the same dynamics and reward features, and only differ in the weights of the reward function. The proposed method initializes the value function for a new task using the previously learned value functions as a prior.~\citet{Wilson07MR} and~\citet{Lazaric10BM} both assume that the distribution over some components of the tasks is drawn from a hierarchical Bayesian model. 

Bayesian learning methods have also been used for regret minimization in multi-armed bandits. This area that goes back to the seminal work of~\citet{Gittins79BP}, has become very active with the Bayesian version of the upper confidence bound (UCB) algorithm~\citep{Kaufmann12BU} and the recent advancements in the analysis of Thompson Sampling~\citep{Agrawal11AT,Kaufmann12TS,agrawal2013further,agrawal2013thompson,Russo14IT,Gopalan14TS,guha2014stochastic,Liu2015prior} and its state-of-the-art empirical performance~\citep{Scott10MB,Chapelle11EE}, which has also led to its use in several industrial applications~\citep{Graepel10WB,Tang13AA}.


\section{Discussion}
\label{sec:discussion}

In this paper, we first proposed an alternative approach to the conventional frequentist (Monte-Carlo based) policy gradient estimation procedure. Our approach is based on {\em Bayesian quadrature}~\citep{Ohagan91BQ}, a Bayesian method for integral evaluation. The idea is to model the gradient of the expected return with respect to the policy parameters, which is of the form of an integral, as Gaussian processes (GPs). This is done by dividing the integrand into two parts, treating one as a random function (or random field), whose random nature reflects our subjective uncertainty concerning its true identity. This allows us to incorporate our prior knowledge of this term (part) into its prior distribution. Observing (possibly noisy) samples of this term allows us to employ Bayes' rule to compute a posterior distribution of it conditioned on these samples. This in turn induces a posterior distribution over the value of the integral, which is the gradient of the expected return. By properly partitioning the integrand and by appropriately selecting a prior distribution, a closed-form expression for the posterior moments of the gradient of the expected return is obtained. We proposed two different ways of partitioning the integrand resulting in two distinct Bayesian models. For each model, we showed how the posterior moments of the gradient conditioned on the observed data are calculated. In line with previous work on Bayesian quadrature, our Bayesian approach tends to significantly reduce the number of samples needed to obtain accurate gradient estimates. Moreover, estimates of the natural gradient and the gradient covariance are provided at little extra cost. We performed detailed experimental comparisons of the Bayesian policy gradient (BPG) algorithms presented in the paper with classic Monte-Carlo based algorithms on a bandit problem as well as on a linear quadratic regulator problem. The experimental results are encouraging, but we conjecture that even better gains may be attained using this approach. This calls for additional theoretical and empirical work. It is important to note that the gradient estimated by Algorithm~\ref{alg:BPG-Eval} may be employed in conjunction with conjugate-gradients and line-search methods for making better use of the gradient information. We also showed that the models and algorithms presented in this paper can be extended to partially observable problems without any change along the same lines as~\citet{Baxter01IP}. This is due to the fact that our BPG framework considers complete system trajectories as its basic observable unit, and thus, does not require the dynamic within each trajectory to be of any special form. This generality has the downside that our proposed framework cannot take advantage of the Markov property when the system is Markovian. 

To address this issue, we then extended our BPG framework to actor-critic algorithms and presented a new Bayesian take on the actor-critic architecture. By using GPs and choosing their prior distributions to make them compatible with a parametric family of policies, we were able to derive closed-form expressions for the posterior distribution of the policy gradient updates. The posterior mean is used to update the policy and the posterior covariance to gauge the reliability of this update. Our Bayesian actor-critic (BAC) framework uses individual state-action-reward transitions as its basic observable unit, and thus, is able to take advantage of the Markov property of the system trajectories (when the system is indeed Markovian). This improvement seems to be borne out in our experiments, where BAC provides more accurate estimates of the policy gradient than either of the two BPG models for the same amount of data. Similar to BPG, another feature of BAC is that its natural-gradient variant is obtained at little extra cost. For both BPG and BAC, we derived the sparse form of the algorithms, which would make them significantly more time and memory efficient. Finally, we performed an experimental evaluation of the BAC algorithm, comparing it with classic Monte-Carlo based policy gradient algorithms, as well as our BPG algorithms, on a random walk problem, the widely used mountain car problem~\citep{Sutton98IR}, and the continuous state and continuous action ship steering domain~\citep{Miller90NN}.

Additional experimental work is required to investigate the behavior of BPG and BAC algorithms in larger and more realistic domains, involving continuous and high-dimensional state and action spaces. The BPG and BAC algorithms proposed in the paper use only the posterior mean of the gradient in their updates. We conjecture that the second-order statistics obtained from BPG and BAC (both in the actor and critic) may be used to devise more efficient algorithms. In one of the experiments in Section~\ref{sec:BPG-experiments}, we employed the covariance information provided by Algorithm~\ref{alg:BPG-Eval} for risk-aware selection of the step size in the gradient updates, which showed promising performance. Other interesting directions for future work include {\bf 1)} investigating other possible partitions of the integrand in the expression for $\nabla\eta_B(\vectheta)$ into a GP term and a deterministic term, {\bf 2)} using other types of kernel functions such as sequence kernels, {\bf 3)} combining our approach with MDP model estimation to allow transfer of learning between different policies (model-based Bayesian policy gradient), and {\bf 4)} investigating more efficient methods for estimating the Fisher information matrix. Another direction is to derive a fully non-parametric actor-critic algorithm. In BAC, the critic is based on Gaussian process temporal difference learning, which is a non-parametric method, while the actor uses a family of parameterized policies. The idea here would be to replace the actor in the BAC algorithm with a non-parametric actor that performs gradient search in a function space (e.g., a reproducing kernel Hilbert space) of policies. 


\acks{Part of the computational experiments was conducted using the Grid'5000 experimental testbed (https://www.grid5000.fr). Yaakov Engel was supported by an Alberta Ingenuity fellowship.} 


\newpage

\appendix
\section{Proof of Proposition~\ref{prop:3}}
\label{sec:proofP3}

We start the proof with the $M\times 1$ vector $\vecb$, whose $i$th element can be written as
\begin{small}
\begin{align*}
(\vecb)_i
&=
\int k(\xi,\xi_i)\Pr(\xi;\vectheta)d\xi \\
&\stackrel{\text{(a)}}{=} \int\big(1+\vecu(\xi)^\top\matG^{-1}\vecu(\xi_i)\big)^2\Pr(\xi;\vectheta)d\xi \\
&\stackrel{\text{(b)}}{=}
\int\Pr(\xi;\vectheta)d\xi+2\left(\int\vecu(\xi)\Pr(\xi;\vectheta)d\xi\right)^\top\matG^{-1}\vecu(\xi_i)+\int\vecu(\xi_i)^\top\matG^{-1}\vecu(\xi)\vecu(\xi)^\top\matG^{-1}\vecu(\xi_i)\Pr(\xi;\vectheta)d\xi \\
&\stackrel{\text{(c)}}{=}
1+\big(\vecu(\xi_i)^\top\matG^{-1}\big)\left(\int\vecu(\xi)\vecu(\xi)^\top\Pr(\xi;\vectheta)d\xi\right)\big(\matG^{-1}\vecu(\xi_i)\big) \\
&\stackrel{\text{(d)}}{=}
1+\big(\vecu(\xi_i)^\top\matG^{-1}\big)\matG\big(\matG^{-1}\vecu(\xi_i)\big) \stackrel{\text{(e)}}{=} 1+\vecu(\xi_i)^\top\matG^{-1}\vecu(\xi_i)
\end{align*}
\end{small}

\vspace{-0.1in}
\noindent
{\bf (a)} substitutes $k(\xi,\xi_i)$ with the quadratic Fisher kernel from Equation~\ref{kernel1}, {\bf (b)} is algebra, {\bf (c)} follows from {\bf (i)} $\int\Pr(\xi;\vectheta)d\xi=1$, and {\bf (ii)} $\int\vecu(\xi)\Pr(\xi;\vectheta)d\xi=\int\nabla\log\Pr(\xi;\vectheta)\Pr(\xi;\vectheta)d\xi$ $=\int\nabla\Pr(\xi;\vectheta)d\xi=\nabla\int\Pr(\xi;\vectheta)d\xi=\nabla(1)=0$, {\bf (d)} is the result of replacing the integral with the Fisher information matrix $\matG$, {\bf (e)} is algebra, and thus, the claim follows. \\
\newline  
Now the proof for the scalar $b_0$
\begin{small}
\begin{align}
b_0
&=
\iint k(\xi,\xi')\Pr(\xi;\vectheta)\Pr(\xi';\vectheta)d\xi d\xi' \nonumber \\
&\stackrel{\text{(a)}}{=}
\iint\big(1+\vecu(\xi)^\top\matG^{-1}\vecu(\xi')\big)^2\Pr(\xi;\vectheta)\Pr(\xi';\vectheta)d\xi d\xi' \nonumber\\
&\stackrel{\text{(b)}}{=}
\iint\Pr(\xi;\vectheta)\Pr(\xi';\vectheta)d\xi d\xi'+2\iint\vecu(\xi)^\top\matG^{-1}\vecu(\xi')\Pr(\xi;\vectheta)\Pr(\xi';\vectheta)d\xi d\xi' \nonumber\\
&+ 
\iint\vecu(\xi)^\top\matG^{-1}\vecu(\xi')\vecu(\xi')^\top\matG^{-1}\vecu(\xi)\Pr(\xi;\vectheta)\Pr(\xi';\vectheta)d\xi d\xi' \nonumber\\
&\stackrel{\text{(c)}}{=}
1+2\left(\int\vecu(\xi)\Pr(\xi;\vectheta)d\xi\right)^\top\matG^{-1}\left(\int\vecu(\xi')\Pr(\xi';\vectheta)d\xi'\right) \nonumber\\
&+\int\vecu(\xi)^\top\matG^{-1}\left(\int\vecu(\xi')\vecu(\xi')^\top\Pr(\xi';\vectheta)d\xi'\right)\matG^{-1}\vecu(\xi)\Pr(\xi;\vectheta)d\xi \nonumber\\
&\stackrel{\text{(d)}}{=}
1+\int\vecu(\xi)^\top\matG^{-1}\vecu(\xi)\Pr(\xi;\vectheta)d\xi
\label{proof1}
\end{align}
\end{small}

\vspace{-0.1in}
\noindent
{\bf (a)} substitutes $k(\xi,\xi')$ with the quadratic Fisher kernel from Equation~\ref{kernel1}, {\bf (b)} is algebra, {\bf (c)} follows from {\bf (i)} $\iint\Pr(\xi;\vectheta)\Pr(\xi';\vectheta)d\xi d\xi'=1$, and {\bf (ii)} $\int\vecu(\xi)\Pr(\xi;\vectheta)d\xi=0$, and finally {\bf (d)} is the result of replacing the integral within the parentheses with the Fisher information matrix $\matG$.

The Fisher information matrix $\matG$ is positive definite and symmetric. Thus, it can be written as $\matG=\matV\matLambda\matV^\top$, where $\matV=[\vecv_1,\ldots,\vecv_n]$ and $\matLambda=\diag[\lambda_1,\ldots,\lambda_n]$ are the matrix of orthonormal eigenvectors and the diagonal matrix of eigenvalues of matrix $\matG$, respectively. By replacing $\matG^{-1}$ with $\matV\matLambda^{-1}\matV^\top$ in Equation~\ref{proof1} we obtain 
\begin{small}
\begin{align*}
b_0
&=
1+\int\vecu^\top(\xi)\matV\matLambda^{-1}\matV^\top\vecu(\xi)\Pr(\xi;\vectheta)d\xi \\
&\stackrel{\text{(a)}}{=}
1+\int\left(\matV^\top\vecu(\xi)\right)^\top\matLambda^{-1}\left(\matV^\top\vecu(\xi)\right)\Pr(\xi;\vectheta)d\xi
\stackrel{\text{(b)}}{=} 
1+\int\left(\sum_{i=1}^n\lambda^{-1}_i\left(\matV^\top\vecu(\xi)\right)_i^2\right)\Pr(\xi;\vectheta)d\xi \\
&\stackrel{\text{(c)}}{=}
1+\sum_{i=1}^n\lambda^{-1}_i\left(\int\left(\vecv_i^\top\vecu(\xi)\right)^2\Pr(\xi;\vectheta)d\xi\right) 
\stackrel{\text{(d)}}{=}
1+\sum_{i=1}^n\lambda^{-1}_i\left(\int\vecv_i^\top\vecu(\xi)\vecv_i^\top\vecu(\xi)\Pr(\xi;\vectheta)d\xi\right) \\
&\stackrel{\text{(e)}}{=}
1+\sum_{i=1}^n\lambda^{-1}_i\left(\int\vecv_i^\top\vecu(\xi)\vecu(\xi)^\top\vecv_i\Pr(\xi;\vectheta)d\xi\right)
\stackrel{\text{(f)}}{=}
1+\sum_{i=1}^n\lambda^{-1}_i\vecv_i^\top\left(\int\vecu(\xi)\vecu(\xi)^\top\Pr(\xi;\vectheta)d\xi\right)\vecv_i \\
&\stackrel{\text{(g)}}{=} 
1+\sum_{i=1}^n\lambda^{-1}_i\vecv_i^\top\matG\vecv_i \stackrel{\text{(h)}}{=} 1+\sum_{i=1}^n\lambda^{-1}_i\vecv_i^\top\lambda_i\vecv_i=1+\sum_{i=1}^n\vecv_i^\top\vecv_i=1+\sum_{i=1}^n||\vecv_i||^2 
\stackrel{\text{(i)}}{=}
1+n
\end{align*}
\end{small}

\vspace{-0.1in}
\noindent
{\bf (a)} and {\bf (b)} are algebra, {\bf (c)} is the result of switching the sum and the integral, {\bf (d)} is algebra, {\bf (e)} follows from the fact that $\vecv_i^\top\vecu(\xi)$ is a scalar, and thus, can be replaced by its transpose, {\bf (f)} is algebra, {\bf (g)} substitutes the integral within the parentheses with the Fisher information matrix $\matG$, {\bf (h)} replaces $\matG\vecv_i$ with $\lambda_i\vecv_i$, {\bf (i)} follows from the orthonormality of $\vecv_i$'s, and thus, the claim follows.


\section{Proof of Proposition~\ref{prop:4}}
\label{sec:proofP4}

We start with the proof of $\matB$. This $n\times M$ matrix may be written as
\begin{small}
\begin{align*} 
\matB
&=
\int\nabla\Pr(\xi;\vectheta)\veck(\xi)^\top d\xi=\int\nabla\Pr(\xi;\vectheta)\left[k(\xi,\xi_1),\ldots,k(\xi,\xi_M)\right]d\xi \\
&\stackrel{\text{(a)}}{=} 
\int\nabla\Pr(\xi;\vectheta)\left[\vecu(\xi)^\top\matG^{-1}\vecu(\xi_1),\ldots,\vecu(\xi)^\top\matG^{-1}\vecu(\xi_M)\right]d\xi \\
&\stackrel{\text{(b)}}{=}  
\left(\int\nabla\Pr(\xi;\vectheta)\vecu(\xi)^\top d\xi\right)\matG^{-1}\left[\vecu(\xi_1),\ldots,\vecu(\xi_M)\right] \\
&\stackrel{\text{(c)}}{=}  
\left(\int\vecu(\xi)\vecu(\xi)^\top\Pr(\xi;\vectheta)d\xi\right)\matG^{-1}\left[\vecu(\xi_1),\ldots,\vecu(\xi_M)\right] \\
&\stackrel{\text{(d)}}{=}  
\matG\matG^{-1}\left[\vecu(\xi_1),\ldots,\vecu(\xi_M)\right] 
\stackrel{\text{(e)}}{=} 
\left[\vecu(\xi_1),\ldots,\vecu(\xi_M)\right]=\matU
\end{align*}
\end{small}

\vspace{-0.1in}
\noindent
{\bf (a)} substitutes $k(\xi,\xi_i)$ with the Fisher kernel from Equation~\ref{kernel2}, {\bf (b)} is algebra, {\bf (c)} follows from $\nabla\Pr(\xi;\vectheta)=\vecu(\xi)\Pr(\xi;\vectheta)$, {\bf (d)} substitutes the integral within the parentheses with the Fisher information matrix $\matG$, {\bf (e)} is algebra, and thus, the claim follows. \\
\newline
Now the proof for the $n\times n$ matrix $\matB_0$
\begin{small}
\begin{align*}
\matB_0
&=
\iint k(\xi,\xi')\nabla\Pr(\xi;\vectheta)\nabla\Pr(\xi';\vectheta)^\top d\xi d\xi' \\
&\stackrel{\text{(a)}}{=} 
\iint\nabla\Pr(\xi;\vectheta)k(\xi,\xi')\nabla\Pr(\xi';\vectheta)^\top d\xi d\xi' \\
&\stackrel{\text{(b)}}{=} 
\iint\big(\vecu(\xi)\Pr(\xi;\vectheta)\big)\vecu(\xi)^\top\matG^{-1}\vecu(\xi')\big(\vecu(\xi')\Pr(\xi';\vectheta)\big)^\top d\xi d\xi' \\
&\stackrel{\text{(c)}}{=}
\left(\int\vecu(\xi)\vecu(\xi)^\top\Pr(\xi;\vectheta)d\xi\right)\matG^{-1}\left(\int\vecu(\xi')\vecu(\xi')^\top\Pr(\xi';\vectheta)d\xi'\right)
\stackrel{\text{(d)}}{=}
\matG\matG^{-1}\matG=\matG
\end{align*}
\end{small}

\vspace{-0.1in}
\noindent
{\bf (a)} follows from the fact that $k(\xi,\xi')$ is scalar, {\bf (b)} substitutes $k(\xi,\xi')$ with the Fisher information kernel from Equation~\ref{kernel2} and $\nabla\Pr(\xi;\vectheta)$ with $\vecu(\xi)\Pr(\xi;\vectheta)$, {\bf (c)} is algebra, {\bf (d)} is the result of substituting the integrals within the parentheses with the Fisher information matrix $\matG$, and thus, the claim follows.


\section{Proof of Proposition~\ref{prop:5}}
\label{sec:proofP5}

Here we only show the proof for Model 1, the proof for Model 2 is straightforward following the same arguments. The sparse approximations of the kernel matrix $\matK$ and kernel vector $\veck(\cdot)$ may be written as $\matK\approx\matA\tilde{\matK}\matA^\top$ and 
$\veck(\cdot)\approx\matA\tilde{\veck}(\cdot)$, respectively (Equations.~2.2.8 and~2.2.9 in~\citealp{Engel05AR}). If we replace $\matK$ and $\veck(\cdot)$ in Equation~\ref{model1_posteriors} with their sparse approximations, we obtain
\begin{align}
\label{model1_sparse_posteriors}
\exptE\big[\nabla\eta_B(\vectheta)|\D_M\big]&=\matY\big(\matA\tilde{\matK}\matA^\top+\matSigma\big)^{-1}\vecb,\nonumber \\
\Cov\big[\nabla\eta_B(\vectheta)|\D_M\big]&= \Big[b_0-\vecb^\top \big(\matA\tilde{\matK}\matA^\top+\matSigma\big)^{-1}\vecb\Big]\matI.
\end{align} 
Sparsification does not change $b_0$ and it remains equal to $n+1$ (see Proposition~\ref{prop:3}), however it modifies $\vecb$ to
\begin{equation*}
\vecb=\int\veck(\xi)\Pr(\xi;\vectheta)d\xi=\int\matA\tilde{\veck}(\xi)\Pr(\xi;\vectheta)d\xi=\matA\int\tilde{\veck}(\xi)\Pr(\xi;\vectheta)d\xi=\matA\tilde{\vecb},
\end{equation*}
where $\tilde{\vecb}=\int\tilde{\veck}(\xi)\Pr(\xi;\vectheta)d\xi$ is exactly $\vecb$, only the kernel vector $\veck(\cdot)$ has been replaced by the sparse kernel vector $\tilde{\veck}(\cdot)$. Thus using Proposition~\ref{prop:3}, we have $(\tilde{\vecb})_i=1+\vecu(\xi_i)^\top\matG^{-1}\vecu(\xi_i)$, with 
$\xi_i\in\tilde{\D}$. By replacing $\vecb$ with $\matA\tilde{\vecb}$ in Equation~\ref{model1_sparse_posteriors}, we obtain

\begin{small}
\begin{align*}
\exptE\big[\nabla\eta_B(\vectheta)|\D_M\big]&=\matY\big(\matA\tilde{\matK}\matA^\top+\matSigma\big)^{-1}\matA\tilde{\vecb}=\matY\matSigma^{-1}\big(\matA\tilde{\matK}\matA^\top\matSigma^{-1}+\matI\big)^{-1}\matA\tilde{\vecb},\nonumber \\
\Cov\big[\nabla\eta_B(\vectheta)|\D_M\big]&=\Big(b_0-\tilde{\vecb}^\top\matA^\top\big(\matA\tilde{\matK}\matA^\top+\matSigma\big)^{-1}\matA\tilde{\vecb}\Big)\matI=\Big(b_0-\tilde{\vecb}^\top\matA^\top\matSigma^{-1}\big(\matA\tilde{\matK}\matA^\top\matSigma^{-1}+\matI\big)^{-1}\matA\tilde{\vecb}\Big)\matI.
\end{align*} 
\end{small}

\vspace{-0.05in}
\noindent
The claim follows using Lemma 1.3.2 in~\citet{Engel05AR}.


\section{Proof of Proposition~\ref{prop:UV}}
\label{sec:proofP6}

We start the proof with the $n\times(t+1)$ matrix $\matB_t$, whose $i$th column may be written as
\begin{small}
\begin{align*}
(\matB_t)_i&=\int_\Z d\vecz\vecg(\vecz;\vectheta)k(\vecz,\vecz_i)=\int_\Z d\vecz\pi^\mu(\vecz)\nabla\log\mu(a|x;\vectheta)\Big(k_x(x,x_i)+k_F(\vecz,\vecz_i)\Big) \\
&=\int_\Z d\vecz\pi^\mu(\vecz)\nabla\log\mu(a|x;\vectheta)k_x(x,x_i)+\int_\Z d\vecz\pi^\mu(\vecz)\nabla\log\mu(a|x;\vectheta)k_F(\vecz,\vecz_i) \\
&=\int_\X dx\nu^\mu(x)k_x(x,x_i)\int_\A da\mu(a|x;\vectheta)\nabla\log\mu(a|x;\vectheta)+\int_\Z d\vecz\pi^\mu(\vecz)\vecu(\vecz)\vecu(\vecz)^\top\matG^{-1}\vecu(\vecz_i) \\
&=\int_\X dx\nu^\mu(x)k_x(x,x_i)\int_\A da\nabla\mu(a|x;\vectheta)+\left(\int_\Z d\vecz\pi^\mu(\vecz)\vecu(\vecz)\vecu(\vecz)^\top\right)\matG^{-1}\vecu(\vecz_i) \\
&=\int_\X dx\nu^\mu(x)k_x(x,x_i)\nabla\left(\int_\A da\mu(a|x;\vectheta)\right)+\matG\matG^{-1}\vecu(\vecz_i) \\
&=\int_\X dx\nu^\mu(x)k_x(x,x_i)\nabla(1)+\vecu(\vecz_i)=\vecu(\vecz_i)
\end{align*}
\end{small}

\noindent
The 1st line follows from the definition of matrix $\matB_t$, function $\vecg$, and kernel $k$, the 2nd line is algebra, the 3rd line follows from the definition of $\pi^\mu$ and the Fisher kernel $k_F$, the 4th line is algebra, the 5th line is the result of replacing the integral in the parentheses with the Fisher information matrix $\matG$, finally the 6th line is algebra, and the claim follows. \\

\noindent
Now the proof for the $n\times n$ matrix $\matB_0$
\begin{small}
\begin{align*}
\matB_0&=\int_{\Z^2}d\vecz d\vecz'\vecg(\vecz;\vectheta)k(\vecz,\vecz')\vecg(\vecz';\vectheta)^\top \\
&\stackrel{\text{(a)}}{=}
\int_{\Z^2}d\vecz d\vecz'\pi^\mu(\vecz)\nabla\log\mu(a|x;\vectheta)\Big(k_x(x,x')+k_F(\vecz,\vecz')\Big)\nabla\log\mu(a'|x';\vectheta)^\top\pi^\mu(\vecz') \\
&\stackrel{\text{(b)}}{=}
\int_{\Z^2}d\vecz d\vecz'\pi^\mu(\vecz)\nabla\log\mu(a|x;\vectheta)k_x(x,x')\nabla\log\mu(a'|x';\vectheta)^\top\pi^\mu(\vecz') \\
&+
\int_{\Z^2}d\vecz d\vecz'\pi^\mu(\vecz)\nabla\log\mu(a|x;\vectheta)k_F(\vecz,\vecz')\nabla\log\mu(a'|x';\vectheta)^\top\pi^\mu(\vecz') \\
&\stackrel{\text{(c)}}{=}
\int_{\X^2}dxdx'\nu^\mu(x)\nu^\mu(x')k_x(x,x')\int_{\A^2}dada'\mu(a|x;\vectheta)\nabla\log\mu(a|x;\vectheta)\nabla\log\mu(a'|x';\vectheta)^\top\mu(a'|x';\vectheta) \\
&+
\int_{\Z^2}d\vecz d\vecz'\pi^\mu(\vecz)\vecu(\vecz)\vecu(\vecz)^\top\matG^{-1}\vecu(\vecz')\vecu(\vecz')^\top\pi^\mu(\vecz') \\
&\stackrel{\text{(d)}}{=}
\int_{\X^2}dxdx'\nu^\mu(x)\nu^\mu(x')k_x(x,x')\int_\A da\nabla\mu(a|x;\vectheta)\int_\A da'\nabla\mu(a'|x';\vectheta)^\top \\
&+
\left(\int_\Z d\vecz\pi^\mu(\vecz)\vecu(\vecz)\vecu(\vecz)^\top\right)\matG^{-1}\left(\int_\Z d\vecz'\pi^\mu(\vecz')\vecu(\vecz')\vecu(\vecz')^\top\right)=\matG\matG^{-1}\matG=\matG
\end{align*}
\end{small}

\noindent
{\bf (a)} follows from the definition of function $\vecg$ and kernel $k$, {\bf (b)} is algebra, {\bf (c)} follows from the definition of $\pi^\mu$ and the Fisher kernel $k_F$, {\bf (c)} is algebra, finally {\bf (d)} follows from $\int_\A da\nabla\mu(a|x;\vectheta)=0$ and $\matG=\int_\Z dz\pi^\mu(\vecz)\vecu(\vecz)\vecu(\vecz)^\top$, and the claim follows.

\newpage
\bibliography{10-245}


\end{document}